\documentclass[onecolumn,draftclsnofoot]{IEEEtran}

\usepackage{amssymb,amsthm,latexsym,amscd,amsfonts,mathrsfs}
\usepackage[cmex10]{amsmath} 
\usepackage{amsmath}
\allowdisplaybreaks[4]
\usepackage{bm}
\usepackage{dsfont}
\usepackage{graphicx}
\usepackage[utf8]{inputenc} 
\usepackage[T1]{fontenc}
\usepackage{hyperref}
\usepackage{url}
\usepackage{ifthen}
\usepackage{cite}    
\usepackage{cite}    
\usepackage{color}
\usepackage{multicol, blindtext}
\usepackage{mathtools, cuted}
\usepackage{subfigure}
\usepackage{epsfig}
\usepackage{float}
\usepackage{lipsum}
\usepackage{stfloats}
\usepackage{array}
\usepackage{algorithmic}
\usepackage{algorithm2e}
\usepackage{multirow}
\usepackage{cuted}

\usepackage{authblk}
\usepackage{hyperref}

\newtheorem{theorem}{Theorem}
\newtheorem{lemma}{Lemma}

\newtheorem{definition}{Definition}

\newtheorem{remark}{Remark}
\newtheorem{corollary}{Corollary}
\newtheorem{assumption}{Assumption}

\DeclareMathOperator{\gen}{gen}

\normalsize

\DeclarePairedDelimiterX{\Rdivx}[2]{(}{)}{%
  #1\;\delimsize\|\;#2%
}

\usepackage{float}
\usepackage{relsize}
\usepackage{subfigure}
\usepackage{array}
\usepackage{tabularx}
\usepackage{multirow}
\usepackage{tikz}
\newcommand*\circled[1]{\tikz[baseline=(char.base)]{
            \node[shape=circle,draw,inner sep=0.35pt] (char) {#1};}}
\begin{document}

\title{Information-Theoretic Analysis of Next-Token Prediction under Markovian Data}

\author{
Masoud~Kavian,
Abdellatif~Zaidi,
and Milad~Sefidgaran%
\thanks{This paper was presented at the 2026 IEEE International Symposium on Information Theory (ISIT) \cite{11653692}.}%
\thanks{Masoud~Kavian and Abdellatif~Zaidi are with the Mathematical and Algorithmic Science Laboratory, Huawei Paris Research Center, 92100 Boulogne-Billancourt, France, and also with the Laboratoire d'Informatique Gaspard Monge, Universit\'e Gustave Eiffel, 77420 Champs-sur-Marne, France (e-mail: masoud.kavian@huawei.com; abdellatif.zaidi@univ-eiffel.fr).

Milad~Sefidgaran is with the Mathematical and Algorithmic Science Laboratory, Huawei Paris Research Center, 92100 Boulogne-Billancourt, France (e-mail: milad.sefidgaran2@huawei.com).}
}

\maketitle

\begin{abstract}
We develop an information-theoretic framework for generalization in next-token prediction under temporally dependent data. We consider independent trajectories generated by finite-memory Markov processes and distinguish algorithmic dependence, quantified by mutual information, from temporal dependence, characterized by mixing. For cross-entropy loss, we derive an expected generalization bound using the Donsker--Varadhan variational representation and a McDiarmid-type concentration inequality for Markov chains. A refinement captures the joint effect of context length and temporal mixing through the mixing properties of the history-state process. We then extend the bound through a rate--distortion formulation, replacing mutual information with the minimum information rate required to represent the learned model within a prescribed distortion in the generalization gap, yielding informative guarantees for deterministic algorithms over continuous hypothesis spaces. For margin-based prediction, we derive explicit bounds for linear and self-attention next-token predictors via noisy low-dimensional compression, revealing the roles of context length, model complexity, sample size, margin, and temporal mixing. Experiments on TinyStories and ETTh2 show that longer contexts can reduce both training and test losses, but typically reduce training loss more, enlarging the generalization gap. A complementary ETTh2 analysis identifies an effective predictive-memory scale near 24 hours, with no statistically supported improvement beyond this scale, offering a plausible explanation for test-performance saturation at larger contexts.
\end{abstract}

\begin{IEEEkeywords}
Next-token prediction, Generalization bounds, Markov processes, Mutual information, Rate--distortion theory, Self-attention.
\end{IEEEkeywords}

\IEEEpeerreviewmaketitle

\section{Introduction}
\label{sec:intro}

\IEEEPARstart{N}{ext}-token prediction (NTP) is the fundamental learning principle underlying modern autoregressive sequence models. Given a finite context, a learner assigns a probability distribution to the next observation, typically by minimizing the cross-entropy loss. From early neural language models \cite{bengio2000} to self-attention architectures \cite{vaswani2017}, increasingly long contexts have enabled models to exploit richer temporal structure without explicitly representing every possible history. From a statistical perspective, however, NTP differs fundamentally from classical learning with independent samples: consecutive tokens belong to the same trajectory, neighboring training examples have overlapping contexts, and their dependence is governed by the memory of the underlying stochastic process. Generalization in NTP must therefore account simultaneously for the information retained by the learned model and the temporal dependence present in the data.

This issue becomes particularly important as the context length $\rho$ grows. Longer contexts may contain useful information about the next observation and thereby improve prediction, but they also allow each loss term to depend on a larger portion of an already correlated trajectory. Consequently, improved prediction and a larger train--test gap may occur simultaneously, while contexts extending beyond the informative memory of the process may yield progressively smaller predictive gains. Related theoretical work has investigated the expressive power and inductive biases of self-attention \cite{yun2020transformers,likhosherstov2021expressive,edelman2022inductive}, its behavior under next-token training \cite{li2024mechanics,madden2024capacity,sander2024universality}, and the role of sparse or structured dependencies in length generalization \cite{golowich2025sparsity}. Other results have exposed limitations of the NTP objective itself on particular tasks \cite{bachmann2024pitfalls}. These observations motivate the central question of this work: \emph{how do context length and temporal dependence jointly affect the generalization of next-token predictors?}

Information-theoretic generalization provides a natural framework for addressing this question. In the i.i.d.\ setting, the work of Russo and Zou \cite{russo2016controlling} and Xu and Raginsky \cite{xu2017information} relates expected generalization to the dependence between the training sample $\mathbf{S}$ and the learned model $W$, as measured by the mutual information $I(\mathbf{S};W)$. This algorithm-dependent perspective has subsequently been refined through individual-sample and conditional information measures, information density, and other data-dependent quantities \cite{raginsky2017,bu2020tightening,steinke2020reasoning,negrea2019information,hellstrom2020generalization,harutyunyan2021blackbox}. A complementary perspective is provided by compression and rate--distortion theory: if the learned predictor can be represented approximately using little information while preserving its relevant behavior, then its effective complexity may be considerably smaller than its ambient dimension \cite{arora2018stronger,barsbey2021heavy,berger1971rate,Sefidgaran2022,sefidgaran2022rate,sefidgaran2024variable,sefidgaran2024minimum}. This lossy viewpoint is particularly useful for deterministic algorithms over continuous hypothesis spaces, for which the ordinary mutual information may be infinite.

Temporal dependence introduces an additional difficulty because the concentration arguments underlying i.i.d.\ information-theoretic bounds no longer apply directly. Generalization with dependent data has been studied using blocking methods, Rademacher complexity, PAC-Bayesian techniques, stability, and online learning \cite{yu1994rates,mohri2008rademacher,ralaivola2009chromatic,agarwal2013generalization,kuznetsov2017generalization}. Concentration results based on mixing, coupling, martingale, and spectral techniques further quantify how dependence reduces the effective amount of statistical information contained in a trajectory \cite{marton1996bounding,samson2000concentration,kontorovich2008concentration,glynn2002hoeffding,lezaud1998chernoff,paulin2012concentration,JMLR:v22:19-479}. For Markov chains, this penalty is naturally characterized through mixing times and spectral gaps \cite{levin2017markov,jerison2013general}, which also govern the difficulty of estimating Markov processes \cite{hao2018learning,wolfer2019minimax}. Consequently, the $T$ observations of a trajectory cannot generally be interpreted as $T$ independent samples.

NTP introduces an additional layer of dependence. A prediction at time $t$ uses the preceding $\rho$ observations, so modifying one token may affect not only its own prediction loss but also several subsequent losses whose contexts contain that token. Context length and the mixing behavior of the data therefore represent two distinct, yet interacting, forms of memory. Recent theoretical work on NTP has begun to address different aspects of this problem. Malach \cite{malach2023auto} studies the learning capabilities of autoregressive predictors, Lotfi et al.\ \cite{lotfi2024unlocking} derive token-level generalization guarantees, and Li et al.\ \cite{li2025generalizationntp} provide a Rademacher-complexity analysis of decoder-only Transformer pretraining that accounts for token dependence. Li et al.\ \cite{li2024mechanics} characterize the behavior learned by self-attention under NTP, while Y\"uksel and Flammarion \cite{yuksel2025sample} analyze empirical risk minimization for finite-order Markov data and relate prediction guarantees to the mixing properties of the chain.

Our focus is complementary to these works. Rather than relying primarily on hypothesis-class complexity, we develop \emph{algorithm-dependent information-theoretic} generalization guarantees for NTP under Markovian data. Importantly, we do not assume that the data-generating process belongs to the predictor class. The finite-memory Markov process generating the trajectories may differ entirely from the linear or self-attention models considered later, and hence the minimum achievable cross-entropy loss need not vanish. We consider $N$ independent trajectories of length $T$ and distinguish between two sources of dependence: the dependence of $W$ on $\mathbf{S}$, quantified through information-theoretic quantities, and the temporal dependence within each trajectory, characterized through mixing. This distinction reveals how the effective statistical value of the data depends not only on $NT$, but also on how rapidly the process forgets its past and on how much history each prediction uses.

Our contributions are summarized as follows.

\emph{First}, we derive expected information-theoretic generalization bounds for cross-entropy NTP under finite-memory Markovian data. By combining the Donsker--Varadhan variational representation with concentration inequalities for Markov chains, the resulting bounds explicitly separate algorithmic dependence, through mutual information, from temporal dependence, through the mixing behavior of the process. We further derive a refined bound based on the mixing properties of the history-state process, which absorbs the dependence induced by overlapping prediction contexts and removes the explicit context-length factor from the leading constant. The analysis therefore makes explicit how the statistical cost of memory arises from both overlapping prediction contexts and dependence within the underlying trajectory.
```

\emph{Second}, we extend this generalization analysis through a rate--distortion formulation \cite{Sefidgaran2022,sefidgaran2022rate}. Rather than representing the learned model exactly, we allow a prescribed distortion in its generalization behavior and characterize the minimum information rate required to achieve it. This extension remains meaningful for deterministic algorithms over continuous parameter spaces and provides a bridge from the general information-theoretic result to explicit model-dependent guarantees.

\emph{Third}, we specialize the framework to margin-based NTP and derive bounds for linear next-token predictors (LNTPs) and self-attention next-token predictors (SANTPs). Using noisy low-dimensional representations based on random projections, selective parameter retention, and controlled perturbations, we obtain scaling laws that reveal the roles of context length, temporal mixing, model norm, prediction margin, dictionary size, and training-set size. These results show that memory influences generalization both through the dependence structure of the data and through the complexity of representing a predictor operating on a longer context.

\emph{Fourth}, we complement the theory with experiments on language and real-world time-series data. On TinyStories \cite{eldan2023tinystories}, increasing context length improves next-token prediction over a substantial range while simultaneously enlarging the empirical train--test generalization gap. A qualitatively similar behavior is observed on ETTh2 \cite{zhou2021informer}: additional history initially improves out-of-sample prediction, but its benefit eventually saturates while the generalization gap remains pronounced. This phenomenon is consistent with the broader time-series literature, in which the efficient exploitation of long lookback windows has motivated architectures such as Autoformer \cite{wu2021autoformer}, FEDformer \cite{zhou2022fedformer}, PatchTST \cite{nie2023patchtst}, TimesNet \cite{wu2023timesnet}, Crossformer \cite{zhang2023crossformer}, and iTransformer \cite{liu2024itransformer}. An additional predictive-memory analysis of ETTh2 identifies an effective scale of approximately $24$ hours, beyond which longer histories provide no statistically supported improvement in out-of-sample prediction under our evaluation protocol.

Taken together, our results highlight a fundamental tradeoff in NTP: longer memory can improve prediction when additional context is informative, but it may simultaneously increase the statistical cost of generalization. The predictive benefit depends on how much new information is contained in the extended history, whereas the generalization cost depends on temporal mixing, context length, and the complexity and compressibility of the learned predictor. Our framework captures these effects jointly without requiring the predictor to coincide with the data-generating mechanism.

The remainder of the paper is organized as follows. Section~\ref{sec:Problem-setup} introduces the data model, NTP formulation, and generalization gap. Section~\ref{headings} develops the information-theoretic bounds and their rate--distortion extension. Section~\ref{Next_token_prediction_section} specializes the results to margin-based LNTP and SANTP models. Section~\ref{sec:experiments} presents the TinyStories and ETTh2 experiments, while additional experimental details and complete proofs are provided in the Appendix.



\section{Problem setup}
\label{sec:Problem-setup}

\subsection{Data generation process}

\vspace{-0.1cm}

Let $\mathcal{D}\coloneqq\{\varnothing,t_1,\ldots,t_d\}$ be a dictionary of input tokens of finite size $d\in \mathbb{N}^*$. We assume that each token in $\mathcal{D}$ is mapped to the embedding space $\mathcal{R}^d$ using the one-hot encoder $\mathcal{E}$ as
\begin{equation}
    \mathcal{E}:
    \begin{cases}
        \mathcal{D} \to \mathbb{R}^{d},\\
        \varnothing \mapsto 0,\\
        t_i \mapsto e_i,
    \end{cases}, \quad\quad{\text{$\{e_i\}_{i=1}^d \colon$  Canonical basis of $\mathbb{R}^d$.}}\nonumber
\end{equation}

We consider supervised learning (SL) in which an agent collects and learns from independent data trajectories whose elements are sampled from $\mathcal{D}$ according to non-stationary stochastic Markov processes with finite memory. For $T \in \mathbb{N}^*$ and context length $\rho$, a length-$T$ trajectory whose data are generated by a Markov process with distribution $\mu$ is a vector \(\mathbf{Z} = (Z_1,\hdots,Z_T) \in \mathcal{D}^{T}\) whose elements are sampled recursively as
\begin{align}
Z_t \sim \mu\left(Z_{t}|Z_{t-1},\ldots,Z_{t-\rho}\right), \quad t=2,\hdots, T,
\label{markov_chain_generated_data_1}
\end{align}
where $[x]^+ \triangleq \max\{x,1\}$. As will become clearer in the sequel, we study a class of collaboratively learned logit-type models that take $\rho$ embeddings as input and generate the next token. Namely, for a model $w \in \mathcal{W}$, we consider the class of mappings
\begin{align}
\mathcal{G}_{\mathcal{W}} \coloneqq \left\{ g_w \colon \mathbb{R}^{\rho d} \to \mathbb{R}^{d} \,\middle|\, w \in \mathcal{W} \right\}.
\end{align}

Given a sequence of tokens $\vec{z} \in \mathcal{D}^{\rho}$, a model $w \in \mathcal{W}$ predicts the next token $z \in \mathcal{D}$ according to
\begin{align}
p_{w}(z|\vec{z})=\big<\sigma(g_w(\mathcal{E}^{\odot\rho}(\vec{z})))\,,\,\mathcal{E}(z)\big>,
\label{conditional_distribution_def}
\end{align}
where $\sigma(\nu) = (\sigma(\nu)_1,\ldots,\sigma(\nu)_d)$ denotes the softmax function, defined as
\begin{equation}
    \sigma(\nu)_i=\frac{\exp(\nu_i)}{\sum_{j\in[d]}\exp(\nu_j)},\qquad i\in[d].
    \label{soft_max_def}
\end{equation}

To simplify the analysis, we make the following assumptions on the models in $\mathcal{G}_{\mathcal{W}}$.

\begin{assumption}[Boundedness of the Model Output] \label{subsubsec:bounded_output}
There exists a constant $B>1$ such that, for all $w\in\mathcal{W}$ and $\nu\in\mathcal{D}^{\rho}$, $\left\|\left(g_w\circ\mathcal{E}^{\odot\rho}\right)(\nu)\right\|_{\infty}\leq B$.
\end{assumption}

\begin{assumption}[Boundedness of the Model Parameters]\label{Bound_size_model}
The random parameter $W$ is assumed to be almost surely bounded. That is, there exists a constant $C>1$ such that $\mathbb{P}\left(\|W\|\leq C\right)=1$. Throughout this paper, $\left\|\cdot\right\|$ denotes the Frobenius norm for matrices.
\end{assumption}

\vspace{-0.3cm}

\subsection{Supervised-learning}\label{Streaming_FL_Sec}

\vspace{-0.1cm}

In a standard supervised-learning (SL) problem, the agent learns from pre-collected data, and thus all samples of the dataset are available.

We consider SL with an agent that collects and uses a finite set of $N$ trajectories, each consisting of $T$ data points, with samples generated by an underlying non-stationary Markov process with finite memory. Let
\begin{align}
    \mathbf{S} = \bigcup\nolimits_{n\in[N]}\mathbf{Z}^{(n)}  , \qquad
    \mathbf{Z}^{(n)} = \Big\{{Z}^{(n)}_1,\cdots,{Z}^{(n)}_T\Big\},
    \nonumber
\end{align}
where $\mathbf{Z}^{(n)}$ denotes the $n^{th}$ trajectory collected and used by the agent.

More precisely, let $\mathcal{A}$ denote the learning algorithm used by the agent. The model generated by the agent is
$W = \mathcal{A}\big(\mathbf{S}\big)$.

\subsection{Generalization gap} \vspace{-0.1cm}

We study the generalization gap induced by the model ${w}$. For convenience, let
\begin{equation}
    \scalebox{0.8}{$\displaystyle
        \mathbf{S} = \bigcup\nolimits_{n\in[N]}\mathbf{Z}^{(n)}  , \qquad
        \mathbf{Z}^{(n)} = \Big\{{Z}^{(n)}_1,\cdots,{Z}^{(n)}_T\Big\}.
    $}
    \nonumber
\end{equation}

For an agent, the possibly stochastic algorithm $\mathcal{A}$ induces a conditional distribution $P_{W|\mathbf{S}}$. Collectively, the algorithm $\mathcal{A}$ then induces a distribution over all models, conditional on all used data, which factorizes as $P_{W\mid\mathbf{S}}$.

Similarly, the induced joint distribution over all used data and models is
\begin{equation}
    \scalebox{0.9}{$\displaystyle
        \mathbf{P}_{\mathbf{S},W} = \mathbf{P}_{W \mid \mathbf{S}}\otimes\prod\nolimits_{n\in[N]} \mathbf{P}_{\mathbf{Z}^{(n)}}.
    $}\nonumber
\end{equation}

The agent is trained to predict the next token under the standard cross-entropy loss. For ${w} \in \mathcal{W}$, let
\begin{align}
    \ell\left(\mathbf{Z},w\right)=-\frac{1}{T}\sum\nolimits_{t \in [T]}
    \log p_{w}\left(Z^{(n)}_{t}\big|Z^{(n)}_{t-1}, \cdots, Z^{(n)}_{t-\rho}\right)
    \label{loss_function_one_sample}
\end{align}
denote the negative log-likelihood averaged over the trajectory $\mathbf{z}^{(n)} = \{z^{(n)}_{1},\hdots,z^{(n)}_{T}\}$, where the conditional probability $p_{w}
    \left(z^{(n)}_{t}\big|
    z^{(n)}_{t-1}, \ldots, z^{(n)}_{t-\rho}\right)$ is given by~\eqref{conditional_distribution_def}. Accordingly, the induced empirical risk is
\begin{align}
    \hat{\mathcal{L}}(\mathbf{S},{w}) =\frac{1}{N}
     \sum_{n=1}^{N} \ell(\mathbf{z}^{(n)},{w}),\nonumber
\end{align}
the corresponding population, or test, loss is
\begin{equation}
\mathcal{L}({w}) = \mathbb{E}_{\mathbf{Z}\sim \mu}\!\left[\ell(\mathbf{Z},{w})\right],
\label{population_risk_acroos_client}
\end{equation}
and the generalization gap is the difference
\begin{equation}
\gen(\mathbf{S},{w}) 
    = \mathcal{L}(w) - \hat{\mathcal{L}}(\mathbf{S},{w}).
    \label{generalization_croos_entropy_def}
\end{equation}









\section{Bounds on the Generalization Gap}
\label{headings}

As already noted in the ``Related Work'' section, a central difficulty in extending proof arguments established in the i.i.d.\ data setting, e.g., in~\cite{sefidgaran2024lessons}, to SL under Markovian data is that, in the former, the model trajectory is the only source of correlation, whereas, in the latter, the data trajectory is itself a stochastic process coupled with the model's evolution. This double coupling makes it impossible to use standard ``ghost sample'' or ``stability'' arguments that assume that a training point can be replaced by an independent test point without affecting the surrounding sequence. To handle these additional statistical dependencies, the analysis must employ concentration inequalities for dependent sequences that account for the fact that the ``information gain'' per sample is suppressed by the spectral gap or mixing time of the underlying Markov process.

\begin{definition}[Mixing time]\label{Mixing_time}
Let $(X_1,\ldots,X_T)$ be a Markov chain taking values in a Polish state space
$\mathcal{X}_1 \times \cdots \times \mathcal{X}_T$, where $X_i \in \mathcal{X}_i$ for all $i$. For any $0 < \epsilon < 1$, the \emph{mixing time} $\tau(\epsilon)$ is defined as the smallest integer $t$ such that, uniformly over all admissible time indices, the conditional distributions of $X_{i+t}$ initialized from any two states are within $\epsilon$ in total variation distance. Formally,
\begin{align}
    \overline{d}(t)
    &\coloneqq
    \max\sup d_{\mathrm{TV}}\!\left(
    \mathbb{P}(X_{i+t}\mid X_i=x'),
    \mathbb{P}(X_{i+t}\mid X_i=x'')
    \right),
    \label{mixing_time_mcdiarmid_max_sup}\\
    \tau(\epsilon)
    &\coloneqq
    \min\left\{t\in\mathbb{N}:\overline{d}(t)\le\epsilon\right\},\nonumber\\
    \tau_{\min}
    &\coloneqq
    \inf_{0\le\delta\le1}
    \tau(\delta)
    \left(\frac{2-\delta}{1-\delta}\right)^2.
    \label{tau_min_time_mcdiarmid_max_sup}
\end{align}
where the maximum and supremum in \eqref{mixing_time_mcdiarmid_max_sup} are taken over
$1\le i\le T-t$ and $(x',x'')\in\mathcal{X}\times\mathcal{X}$, respectively.
\end{definition}

\vspace{-8pt}

\subsection{Loss-less generalization bound}

\begin{theorem}\label{multi_round_markovian_general_bound}
Under Assumptions \ref{subsubsec:bounded_output}-\ref{Bound_size_model}, the generalization gap of the model learned in SL under the cross-entropy loss is bounded as
\begin{equation}
    \mathbb{E}\left[\gen(\mathbf{S},{W})\right]\leq\left(2B(1+2\rho)+\log d\right)\sqrt{\frac{\tau_{\min} I\left(\mathbf{S};W\right)}{2N T}}.\nonumber
\end{equation}
\end{theorem}

A complete and detailed proof is provided in Appendix~\ref{Proof_multi_round}. Here, we present a proof sketch highlighting the main ideas.

\emph{Proof sketch.}
Fix $\lambda>0$ and let $\mathbf{S}'$ be an independent copy of the training set. We first rewrite the expected generalization gap in terms of the difference between the loss evaluated on $\mathbf{S}$ and its expectation over $\mathbf{S}'$. Applying the Donsker--Varadhan variational inequality with $P_{\mathbf{S}|W}$ and $P_{\mathbf{S}}$ yields
$$
\mathbb{E}[\gen(\mathbf{S},W)]
\leq
\frac{I(\mathbf{S};W)}{\lambda}
+
\frac{1}{\lambda}
\sum_{n=1}^{N}
\mathbb{E}_{W}
\log
\mathbb{E}
\exp\!\left[
\frac{\lambda}{N}
\left(
\mathbb{E}_{\mathbf{Z}'^{(n)}}\ell(\mathbf{Z}'^{(n)},W)
-
\ell(\mathbf{Z}^{(n)},W)
\right)
\right],
$$
where the decomposition across $n$ follows from the independence of the $N$ training trajectories.

It therefore remains to control the moment-generating function of the loss along a single Markov trajectory. For a fixed $w$, we compare the losses associated with two trajectories by successively replacing their tokens. Since the prediction at time $t$ depends only on the preceding $\rho$ tokens, changing one token affects its own prediction term and may propagate through at most $\rho$ subsequent context windows. The resulting telescoping argument gives
$$
\left|
\ell(\mathbf{z},w)-\ell(\mathbf{z}',w)
\right|
\leq
\frac{C_1+\rho C_2}{T}
\sum_{t=1}^{T}
\mathbf{1}\{z_t\neq z'_t\},
$$
where $C_1$ controls the change in the predicted token for a fixed context, while $C_2$ controls the effect of changing one token in the context. The bounded-output assumption yields
$$
C_1\leq 2B+\log d,
\qquad
C_2\leq 4B,
$$
and hence the bounded-difference coefficients can be chosen as
$$
c_t
\leq
\frac{2B(1+2\rho)+\log d}{T},
\qquad t\in[T].
$$

Applying the McDiarmid-type concentration inequality for Markov chains then bounds each logarithmic moment-generating function in terms of $\tau_{\min}|c|^2$. Substituting this bound into the information-theoretic inequality above gives
$$
\mathbb{E}[\gen(\mathbf{S},W)]
\leq
\frac{I(\mathbf{S};W)}{\lambda}
+
\frac{\lambda\,\tau_{\min}
\left(2B(1+2\rho)+\log d\right)^2}
{8NT}.
$$

Optimizing over $\lambda>0$ yields the desired bound.

\begin{lemma}\label{Hoeffdinf_Lemma}
Under Assumptions \ref{subsubsec:bounded_output}-\ref{Bound_size_model},
let $S_t:=(Z_t,Z_{t-1},\ldots,Z_{t-\rho+1})$, and let
$\tilde{\tau}_{\min}$ be the mixing-time quantity associated with
the history-state process $\{S_t\}_{t\in[T]}$. Then, the expected
generalization gap of the model learned in SL under the
cross-entropy loss satisfies
\begin{align}
    \mathbb{E}\left[\gen(\mathbf{S},W)\right]
    \leq \left(2B+\log d\right)
    \sqrt{
        \frac{\tilde{\tau}_{\min} I(\mathbf{S};W)}
        {2NT}
    },
    \nonumber
\end{align}
where
\begin{equation}
\tilde{\tau}_{\min}
:=
\inf_{0<\delta<1}
\tilde{\tau}(\delta)
\left(\frac{2-\delta}{1-\delta}\right)^2.\label{new_tau-min}
\end{equation}
In particular, the resulting bound is tighter than that of
Theorem~\ref{multi_round_markovian_general_bound} whenever
\begin{align}
    \tilde{\tau}_{\min}
    <
    \left(
        \frac{2B(1+2\rho)+\log d}
        {2B+\log d}
    \right)^2
    \tau_{\min},
    \nonumber
\end{align}
where $\tau_{\min}$ is defined in
\eqref{tau_min_time_mcdiarmid_max_sup}. The proof is provided in Appendix~\ref{Proof_Hoeffding_employing}.
\end{lemma}

\vspace{-0.2cm}

\subsection{Lossy generalization bound}

In this section, we adopt an approach that is essentially similar to that in~\cite{Sefidgaran2022, sefidgaran2022rate} to improve upon the bound of Theorem~\ref{multi_round_markovian_general_bound}. In particular, this improvement prevents the new bound from taking large values for deterministic algorithms with continuous hypothesis spaces. This permits a more flexible analysis in which the gap is bounded by the minimum information rate required to represent the model updates within a specified error tolerance.

Fix $\epsilon\in\mathbb{R}^+$. Define the rate-distortion function
\begin{align}
    R_{\mathcal{D}}(\epsilon)\coloneqq\inf I(\mathbf{S};\hat{W}),
    \label{general_bound_lossy_mutual_1}
\end{align}
where the mutual information is evaluated with respect to $P_{U}P_{\mathbf{S}}P_{\hat{W}|\mathbf{S},U}$, and the infimum is taken over all conditional Markov kernels $P_{\hat{W}|\mathbf{S},U}$ that satisfy
\begin{align}
\mathbb{E}\Big[\gen(\mathbf{S},{W}) - \gen(\mathbf{S},\hat{W})\Big] \leq \epsilon,
\label{def:Rdist}
\end{align}
for some $\epsilon>0$ in~\eqref{def:Rdist}, with the expectation taken with respect to $P_{U}P_{\mathbf{S},W}P_{\hat{W}|\mathbf{S},{W},U}$.

\begin{theorem}\label{lossy_inexpectation_bound_federated}
Under Assumptions \ref{subsubsec:bounded_output}-\ref{Bound_size_model}, the generalization gap of the model learned in SL under the cross-entropy loss is bounded as
\begin{align}
\mathbb{E}\left[\gen(\mathbf{S},W)\right]\leq \left(2B(1+2\rho)+\log d\right)\sqrt{\frac{\tau_{\min}\: R_{\mathcal{D}}(\epsilon)}{2NT}}+\epsilon,
\nonumber
\end{align}
where the rate-distortion function $R_{\mathcal{D}}(\epsilon)$ is defined as in~\eqref{general_bound_lossy_mutual_1}.
\end{theorem}

The advantage of the new bound is that, for some non-negative error tolerances $\epsilon$, the linear increase in the bound, compared with Theorem~\ref{multi_round_markovian_general_bound}, is favorably compensated by a reduction in the rate inside the square root. As will become clearer from the results for the Linear (LNTP) and Self-Attention Next Token Prediction (SANTP) architectures in the next section, this lossy refinement is what enables the derivation of explicit scaling laws with respect to the system parameters, which are not enabled by the result of Theorem~\ref{multi_round_markovian_general_bound}.



\section{Generalization bound on Next token prediction problem }
\label{Next_token_prediction_section}

\subsection{Margin-Based Generalization Gap}

Next-token prediction is based on the cross-entropy loss, which, by selecting the token that is most likely given the past $\rho$ tokens, essentially seeks to maximize the likelihood gap between the most likely and second most likely tokens. As such, it bears some similarity to classification problems: under supervised fine-tuning, each context is paired with a single target token, and greedy decoding selects the arg-max over the vocabulary. Minimizing the loss to within a margin is known to enable better rates, especially in classification problems~\cite{NIPS2017_b22b257a,gronlund2020}, where it also leads to more efficient algorithms. The corresponding distinction is that, under the cross-entropy loss, the algorithm continues to push the likelihoods further apart across the iterates, whereas, under the margin loss, it stops doing so once the margin is met. It is precisely for this reason that a lower sample complexity is required for the margin loss to achieve a desired accuracy. In this section, we consider a margin-based loss obtained by incorporating a margin parameter $\theta>0$.

For model ${w}$ and data sequence $\mathbf{z}^{(n)}$, let
\begin{align}
\eta({w},t,\mathbf{z}^{(n)})&\coloneqq\log\frac{C_{1}({w},t,\mathbf{z}^{(n)})}{C_{2}({w},t,\mathbf{z}^{(n)})},\nonumber\\
\ell_{\theta}(\mathbf{z}^{(n)},{w})&\coloneqq \frac{1}{T}\sum_{t \in [T]}
    \mathbf{1}\left\{\eta({w},t,\mathbf{z}^{(n)}) < \theta \right\},
 \label{Margin_loss_function_one_sample}
\end{align}
where
\begin{align}
C_{1}({w},t,\mathbf{z}^{(n)}) &= 
p_{w}\!\left(z^{(n)}_{t}|z^{(n)}_{t-1}, \dots, z^{(n)}_{t-\rho}\right),\nonumber\\
C_{2}({w},t,\mathbf{z}^{(n)}) &= \max_{z'^{(n)}_{t} \neq z^{(n)}_{t}}
p_{w}\!\left({z}'^{(n)}_{t}|z^{(n)}_{t-1}, \dots, z^{(n)}_{t-\rho}\right).\nonumber
\end{align}

The population and empirical risks are evaluated, respectively, as
\begin{align}
    \mathcal{L}_{\theta}({w})
    \coloneqq \frac{1}{T}\sum\nolimits_{t\in [T]}
    \mathbb{E}\left[
        \mathbf{1}\left\{\eta({w},t,\mathbf{z}^{(n)})  < {\theta} \right\}
    \right],
    \nonumber
\end{align}
and
\begin{align}
    \hat{\mathcal{L}}_{\theta}(\mathbf{s},{w})
    \coloneqq \frac{1}{NT}  \sum\nolimits_{n \in [N]} \left[\sum\nolimits_{t \in [T]}\mathbf{1}\left\{\eta({w},t,\mathbf{z}^{(n)})<\theta \right\}\right].
\nonumber
\end{align}

The margin generalization gap induced by model ${w}$ is
\begin{align}
\operatorname{gen}_{\theta} (\mathbf{s}, {w})
    =  \mathcal{L}_{-\theta/2}({w})
      - \hat{\mathcal{L}}_{\theta/2}(\mathbf{s},{w}).
\nonumber
\end{align}

\begin{corollary}\label{Extension_Lemma_1_Margin_generalization}
For any $\theta\in\mathbb{R}$, the margin loss defined in
\eqref{Margin_loss_function_one_sample} satisfies the
bounded-difference condition in \eqref{Mdiarmid_type_lemma_1}.
In particular, for any two sequences
$\mathbf{z}=\{z_1,\ldots,z_T\}$ and
$\mathbf{z}'=\{z'_1,\ldots,z'_T\}$,
\begin{align}
    \ell_{\theta}(\mathbf{z},w)
    -
    \ell_{\theta}(\mathbf{z}',w)
    &\leq
    \frac{\rho+1}{T}
    \sum_{t=1}^{T}
    \mathrm{1}\left\{z_t\neq z'_t\right\},
    \nonumber\\
    \ell_{\theta}(\mathbf{z},w)
    -
    \ell_{\theta}(\mathbf{z}',w)
    &\leq
    \frac{1}{T}
    \sum_{t=1}^{T}
    \mathrm{1}\left\{
        \eta(w,t,\mathbf{z})
        \neq
        \eta(w,t,\mathbf{z}')
    \right\}.
    \nonumber
\end{align}

Consequently, Theorem~\ref{multi_round_markovian_general_bound}
yields the following expected generalization bound under the
margin loss:
\begin{align}
    \mathbb{E}\left[\gen_{\theta}(\mathbf{S},W)\right]
    &\leq
    (\rho+1)
    \sqrt{
        \frac{
            \tau_{\min}I(\mathbf{S};W)
        }{
            2NT
        }
    },
    \label{Margin_generalization_bound_tau_out}
\end{align}
In addition, Lemma~\ref{Hoeffdinf_Lemma} yields
the refined bound
\begin{align}
    \mathbb{E}\left[\gen_{\theta}(\mathbf{S},W)\right]
    &\leq
    \sqrt{
        \frac{
            \tilde{\tau}_{\min}I(\mathbf{S};W)
        }{
            2NT
        }
    },
    \label{Margin_generalization_bound_tau_in}
\end{align}
where $\tilde{\tau}_{\min}$ is defined in \eqref{new_tau-min}.

The refined bound removes the explicit factor $\rho+1$ and is
strictly tighter than \eqref{Margin_generalization_bound_tau_out}
whenever
\begin{align}
    \tilde{\tau}_{\min}<(\rho+1)^2\tau_{\min}.
    \nonumber
\end{align}
\end{corollary}

\subsection{Linear Next Token Prediction}

We consider the linear NTP model studied in~\cite{yuksel2025sample}, which is similar to the linear auto-regressive models considered in~\cite{malach2023auto}.   
For $C \in \mathbb{R}^{+}$, let \( \mathcal{M}_{C}^{d_1 \times d_2} = \left\{ M \in \mathbb{R}^{d_1 \times d_2} \: :\: \|M\| \leq C\right\} \) denote the set of matrices in \( \mathbb{R}^{d_1 \times d_2} \) with $C$-bounded Frobenius norm. The parameter space is
\begin{equation}
    \mathcal{W} \coloneqq \left\{ \mathbf{M} = (M_{1}, \ldots, M_{\rho}) \in \mathcal{M}_{C}^{d \times \rho d}\right\}.
   \nonumber
\end{equation}
Accordingly, the class of logit mappings associated with the model is
$\mathcal{G}_{\mathcal{W}} \coloneqq \left\{ g_{\mathbf{M}} \;\middle|\; \mathbf{M} \in \mathcal{W} \right\}$, where, for an input \( \mathbf{x} = (x_{t-1}, \ldots, x_{t-\rho}) \in\mathbb{R}^{\rho d}\), we set
$g_{\mathbf{M}}(\mathbf{x}) \coloneqq \sum_{i=1}^{\rho} M_{i} x_{t-i}$.

\begin{theorem}\label{Theorem_Multi_round_LNTP}
The expected margin generalization gap of LNTP-SL is upper-bounded as follows:
 \begin{align}
        \mathbb{E}\left[\gen_{\theta}(\mathbf{S},W)\right]
        &\leq
        \mathcal{O}\left(\left(\frac{C\rho^{\frac{3}{2}}}{\theta}\right)^2
        \sqrt{
            \frac{
                 {\tau}_{\min}
                \log(dNT)
            }{
                 NT
            }
        }
        \right).
        \nonumber
    \end{align}
\end{theorem}

The complete proof is provided in Appendix~\ref{Proof_Theorem_Multi_round_LNTP}. 
To prove Theorem~\ref{Theorem_Multi_round_LNTP}, we leverage the extension of Theorem~\ref{lossy_inexpectation_bound_federated} derived for the margin loss in the general setting. We then specialize this result to the linear model by considering a carefully chosen explicit instance of the compression scheme, as outlined below. The remainder of the proof is essentially algebraic and consists of evaluating and bounding the rate and distortion of this scheme in the setting considered here.

Here, we present a proof sketch highlighting the main ideas.

\emph{Proof sketch.}
The proof proceeds by constructing an explicit admissible compression kernel for the lossy margin-based generalization bound and then controlling its distortion and rate. For the LNTP model, the log-margin can be written as the difference between two linear scores,
$$
\left\langle V(W,\ell),V(z_{t-1:t-\rho})\right\rangle
-
\max_{\ell'\neq\ell}
\left\langle V(W,\ell'),V(z_{t-1:t-\rho})\right\rangle,
$$
where $V(W,\ell)\in\mathbb{R}^{\rho d}$ collects the parameters associated with token $\ell$. Consequently, if $\hat W$ denotes a compressed model, the change in the log-margin satisfies
$$
D_t
\leq
2\max_{\ell\in[d]}
\left|
\left\langle
V(W-\hat W,\ell),
V(z_{t-1:t-\rho})
\right\rangle
\right|.
$$

We construct $\hat W$ using a Gaussian projection
$\mathsf A\in\mathbb{R}^{m\times\rho d}$ with
$
m=
\mathcal{O}\!\left(
\left(\frac{C\rho}{\theta}\right)^2\log(dNT)
\right).
$

Only parameter vectors $V(W,\ell)$ with sufficiently large norm are retained. Each retained vector is projected by $\mathsf A$; vectors whose projections exhibit atypically large norms are discarded, and independent bounded noise, uniform over an $m$-dimensional Euclidean ball, is added before reconstruction through $\mathsf A^\top$.

The first part of the analysis verifies that this compression satisfies the required distortion constraint. After applying a union bound over the retained vectors, the probability that the margin perturbation exceeds $\theta/2$ is decomposed into four events: distortion of an inner product under the Gaussian projection, the contribution of the added noise, an atypically large projection of the context vector, and an atypically large projection of the model vector. Standard Gaussian-projection concentration inequalities, together with the chosen projection dimension, make each contribution sufficiently small, yielding an overall distortion of order
$\epsilon=\mathcal{O}\!\left(\frac{1}{NT}\right).$

It remains to bound the rate of this compression. Let $\mathbf{T}_{[d]}$ denote the binary vector identifying the retained parameter vectors. By data processing,
$I(\mathbf{S};\hat W)\leq I(W;\hat W)$. Conditioned on $\mathbf{T}_{[d]}$, the entropy of each retained compressed vector is bounded by the volume of an $m$-dimensional ball containing its support, whereas its conditional entropy given $W$ is determined by the smaller ball supporting the additive noise. This yields a volume-ratio bound on the information carried by each retained vector. The additional cost of specifying the retained set is $H(\mathbf{T}_{[d]})$, which is controlled using a binary-entropy bound on the number of possible sparse support patterns. Combining these estimates gives, up to universal constants,
$I(\mathbf{S};\hat W)=
\mathcal{O}\!\left(
\left(\frac{C\rho}{\theta}\right)^4
\log(dNT)
\right).$

Substituting this rate, together with
$\epsilon=\mathcal{O}(1/(NT))$,
into the lossy margin-generalization bound based on
$\widetilde{\tau}_{\min}$ yields the claimed scaling.

\subsection{Self-Attention Next Token Prediction}

We now consider a Transformer-type architecture. For an agent, let $\mathcal{M}_{1},\mathcal{M}_{2},\mathcal{M}_{3}$ denote sets of $d\times d$ matrices with bounded Frobenius norms representing, e.g., query, key, and value matrices. Define the maps $g_{M_{[2]}}:\mathbb{R}^{\rho d}\to\mathbb{R}^{\rho}$ and $f_{M_{[3]}}:\mathbb{R}^{\rho d}\to\mathbb{R}^{d}$ as
\begin{align}
g_{M_{[2]}}(\mathbf{z})&\coloneqq\sigma\left(\frac{1}{\sqrt{\rho}}\big<M_{1} \mathcal{E}(z_{t-1}),M_{2}\mathcal{E}({z}_{t-j})\big>\right)_{j=1}^{\rho},\quad \nonumber\
f_{M_{[3]}}(\mathbf{z})&\coloneqq M_{3}\left(\sum_{j=1}^{\rho}\left(g_{M_{[2]}}(\mathbf{z})\right)_j\cdot\mathcal{E}({x}_{t-j})\right).\nonumber
\end{align}
That is, we consider any agent whose model belongs to the following constrained model class:

$$
\mathcal{W}
\coloneqq
\left\{
M_{[3]}=(M_1,M_2,M_3)
\;\middle|\;
\substack{
M_i\in\mathcal{M}_i^{d\times d},\ \|M_i\|\le\sqrt{C},\ i\in[2],\\
\|M_3\|\le C
}
\right\}.
$$

The corresponding hypothesis class is defined as

$$
\mathcal{F}
\coloneqq
\left\{
f_{M_{[3]}}
\;\middle|\;
M_{[3]}\in\mathcal{W}
\right\}.
$$

\begin{theorem}\label{Theorem_Multi_round_SANTP}
The expected generalization gap of SANTP-SL is bounded by
\begin{align}
\mathbb{E}\left[\gen_{\theta}(\mathbf{S},W)\right]
&\leq
\mathcal{O}\left(\left(\frac{ C^2{\rho}}{\theta}\right)^2
\sqrt{
\frac{
{\tau}_{\min}
\log(\rho dNT)
}{
NT
}
}
\right).
\nonumber
\end{align}
\end{theorem}

A detailed proof is provided in Appendix~\ref{Proof_Theorem_Multi_round_SANTP}; here, we briefly present a proof sketch.

\emph{Proof sketch.}
The proof follows the lossy-compression strategy used in the linear case, but the distortion analysis must additionally account for the nonlinear attention weights. Since the attention scores depend on $M_1$ and $M_2$ only through the product
$
M_4\coloneqq M_2^\top M_1,
$
we use the equivalent parameterization $W=(M_3,M_4)$ and compress $M_3$ and $M_4$ separately using independent Gaussian projections.

To control the distortion, we introduce the intermediate model
$\hat W_1=(M_3,\hat M_4)$ and decompose the perturbation of the log-margin into the change caused by replacing $M_4$ with $\hat M_4$ and the subsequent change caused by replacing $M_3$ with $\hat M_3$. For the first term, the $\ell_1$-Lipschitz property of the softmax implies that the change in the attention distribution is controlled by the perturbation of the attention scores. Using $|M_3|\leq C$, this reduces the corresponding margin perturbation to entrywise inner-product errors involving $M_4-\hat M_4$, with a threshold of order $\theta/(C\sqrt{\rho})$. We therefore choose
$
m_4=\mathcal{O}\left(\left(\frac{C^2\sqrt{\rho}}{\theta}\right)^2
\log(\rho dNT)\right)
$
for the projection of $M_4$. Gaussian inner-product preservation, concentration of projected norms, and bounded additive noise then imply that the contribution of this part to the distortion is $\mathcal{O}(1/(NT))$.

For the second term, the attention-weighted context is a convex combination of one-hot vectors and hence has $\ell_1$ norm at most one. The perturbation caused by replacing $M_3$ with $\hat M_3$ can therefore be controlled directly through inner-product preservation for the rows of $M_3$. Choosing
$
m_3=\mathcal{O}\left(
\left(\frac{C}{\theta}\right)^2
\log(dNT)
\right)
$
and applying the same projection and noise-concentration arguments yields another $\mathcal{O}(1/(NT))$ contribution. Thus, the complete compression kernel satisfies the required distortion constraint with
$
\epsilon=\mathcal{O}\left(\frac{1}{NT}\right).
$

Finally, the rate is bounded by conditioning on the binary masks that identify the retained rows of $M_3$ and $M_4$. By the data-processing inequality and the independence of the two compression mechanisms, the mutual information is bounded by the sum of the corresponding conditional rates. As in the linear case, each conditional rate is controlled by the logarithm of the ratio between the volume of a ball containing a noisy projected row and the volume of the noise ball, while the entropy of the two support masks is bounded using binary-entropy estimates. With the above choices of $m_3$ and $m_4$, the dominant contribution satisfies, up to universal constants and logarithmic factors,

$$
I(\mathbf{S};\hat W)
=
\mathcal{O}\!\left(
\left(\frac{C^4\rho}{\theta^2}\right)^2
\log(\rho dNT)
\right).
$$

Combining this rate with the $\mathcal{O}(1/(NT))$ distortion term in the lossy margin-generalization bound yields the stated SANTP generalization rate.

\begin{remark}
By associating $(s,\bar{s})=(4,2)$ with SANTP and $(s,\bar{s})=(2,3)$ with LNTP, the order-wise generalization bounds for LNTP and SANTP-SL are given by
\begin{align}
\mathcal{O}\left(\left(\frac{{C}^{s}\rho^{\bar{s}}}{\theta^2}\right)\sqrt{\frac{{\tau}_{\min}\log\left(\rho dNT\right):}{NT}}\right).\nonumber
\end{align}

\end{remark}

\begin{remark}
The bounds in Theorems~\ref{Theorem_Multi_round_LNTP} and
\ref{Theorem_Multi_round_SANTP} can be refined by using
$\tilde{\tau}_{\min}$, yielding
\begin{align}
\mathcal{O}\left(
\frac{C^{s}\rho^{\bar{s}}}{\theta^2}
\sqrt{
\frac{\tilde{\tau}_{\min}\log(\rho dNT)}{NT}
}
\right),
\nonumber
\end{align}
where $(s,\bar{s})=(4,1)$ for SANTP and $(s,\bar{s})=(2,2)$
for LNTP, and $\tilde{\tau}_{\min}$ is defined in
\eqref{new_tau-min}. These bounds are strictly tighter than
their scalar-mixing counterparts whenever
\begin{align}
\tilde{\tau}_{\min}<(\rho+1)^2\tau_{\min}.
\nonumber
\end{align}
\end{remark}

\subsection{Effect of memory (mixing time)}

An important question in this context is how the statistical nature of the data stream, i.e., whether the data are i.i.d. or Markovian, affects generalization performance. Our analysis reveals that Markovian dependencies introduce a fundamental ``memory penalty'' that is absent in the i.i.d. setting. While i.i.d. samples provide independent information at each time step, Markovian samples are statistically coupled, thereby reducing the effective information rate of the data stream. In our bounds, this effect is captured through the square root of the mixing time, $\sqrt{\tau_{\min}}$, which acts as an inflation factor for the generalization gap. Consequently, achieving generalization performance comparable to that of an i.i.d. system requires a larger data budget in Markovian SL to compensate for the redundancy introduced by temporal dependencies. This observation is consistent with existing results on hidden Markov models, where prediction performance converges more slowly under weak mixing conditions because of persistent long-range dependencies that obscure the underlying process structure.



\section{Experiments}
\label{sec:experiments}

In this section, we present numerical experiments to examine whether the qualitative dependence on context length highlighted by our theoretical analysis also arises in more realistic self-attention architectures. These experiments are not intended as a direct numerical verification of the finite-sample bounds or their exact parameter scaling, since the empirical Transformer architectures are more expressive than the stylized SANTP model analyzed theoretically. Rather, they are designed to assess the qualitative prediction that increasing the available context may improve next-token prediction while simultaneously enlarging the empirical train--test generalization gap. We first consider language-model pretraining on the TinyStories corpus, followed by experiments on the real-world ETTh2 time series.

\subsection{Language Model Pretraining on TinyStories}
\label{subsec:tinystories}

We consider next-token prediction on the TinyStories corpus. We use
$8\%$ of the original training split and train a byte-level BPE
tokenizer with vocabulary size $d=8000$ on the selected stories.
After tokenization, the resulting token stream is partitioned into
training and held-out evaluation subsets, with $10\%$ of the tokens
reserved for evaluation. Throughout the paper, we refer to the loss
measured on this held-out subset as the \emph{test loss}; the same
subset is denoted as the validation set in our implementation.

We train a GPT-style decoder-only Transformer based on causal
self-attention. The model consists of $6$ Transformer blocks with
embedding dimension $512$, $8$ attention heads, and an MLP expansion
ratio of $4$. Each block employs pre-normalization with RMSNorm,
causal self-attention, and a SwiGLU feed-forward network. The input
and output token embeddings are tied, and dropout is set to $0.1$.
The model is optimized using the standard cross-entropy next-token
prediction loss. Further implementation and optimization details are
provided in Appendix~\ref{app:tinystories_details}.

To investigate the effect of the context length that appears explicitly
in our theoretical bounds, we consider a centralized learning setup
and vary $\rho \in \{2,4,8,16,32,64,128\}$.
For each value of $\rho$, a new model is initialized and trained from
scratch, while the remaining optimization settings and the total
gradient-step budget are kept fixed. We estimate the empirical
generalization gap as
\begin{equation}
    \widehat{\operatorname{gen}}(\rho)
    =
    \widehat{\mathcal L}_{\mathrm{test}}(\rho)
    -
    \widehat{\mathcal L}_{\mathrm{train}}(\rho),
    \nonumber
\end{equation}
where $\widehat{\mathcal L}_{\mathrm{train}}$ denotes the
cross-entropy loss evaluated on the training stream and
$\widehat{\mathcal L}_{\mathrm{test}}$ denotes the corresponding
loss on the held-out evaluation stream. The reported curves are
averaged over independent training realizations, and the shaded
regions represent one standard deviation.

\begin{figure*}[t]
    \centering
    \begin{minipage}[t]{0.35\textwidth}
        \centering
        \includegraphics[width=\linewidth]{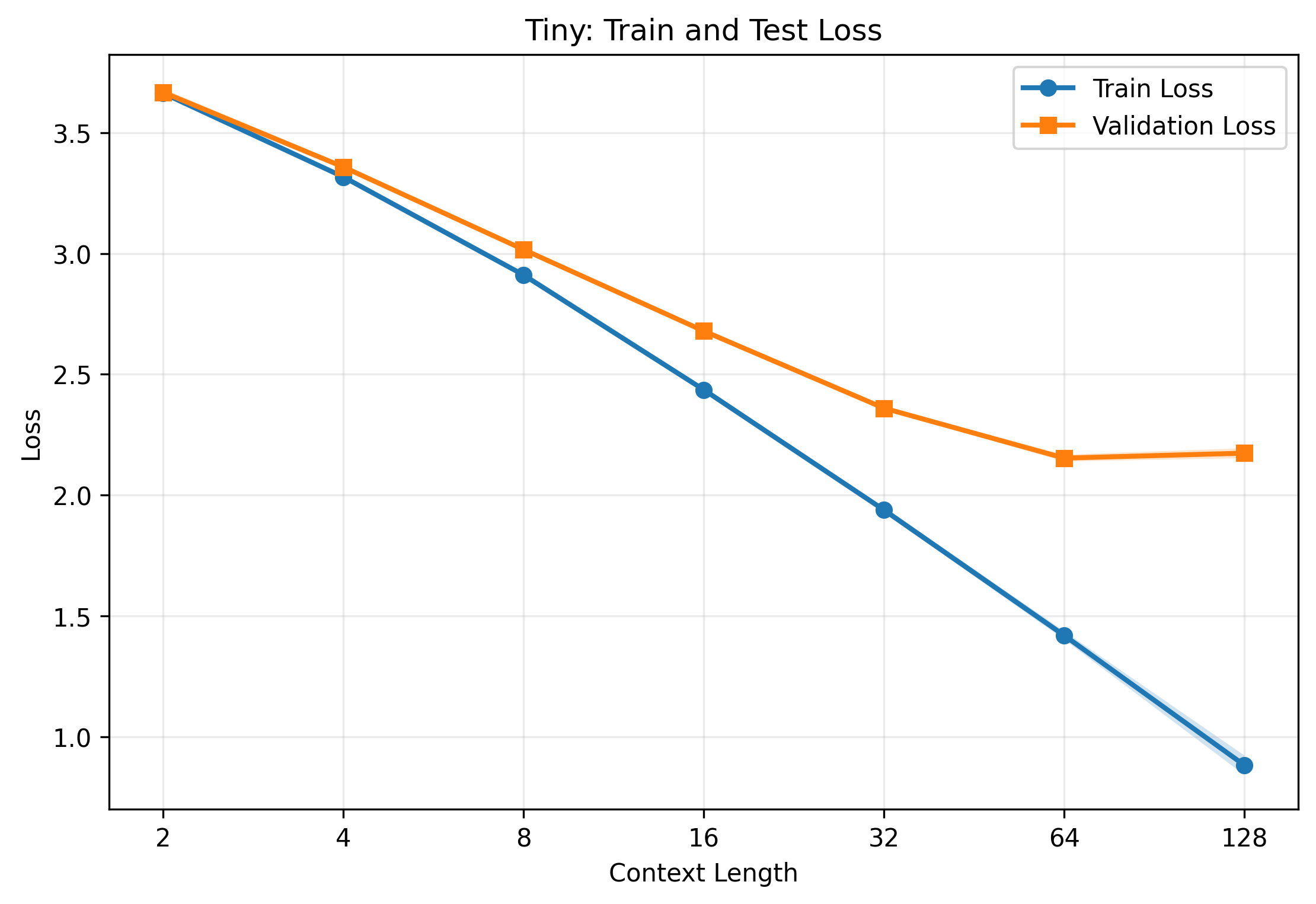}

        \vspace{1mm}
        \small (a) Training and test cross-entropy losses
    \end{minipage}
    \hspace{0.1\textwidth}
    \begin{minipage}[t]{0.35\textwidth}
        \centering
        \includegraphics[width=\linewidth]{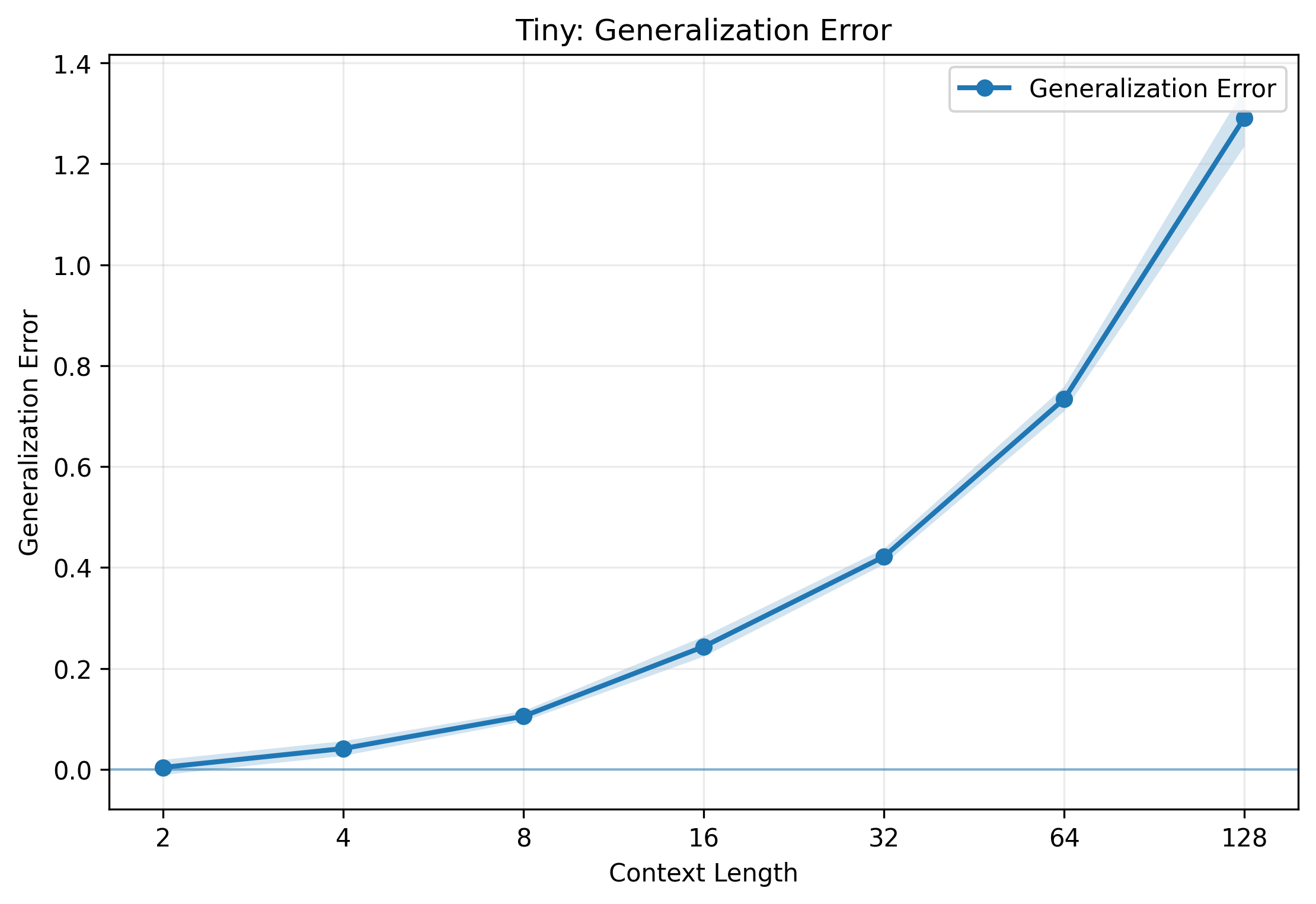}

        \vspace{1mm}
        \small (b) Generalization gap
    \end{minipage}

    \caption{\textbf{Effect of context length on next-token prediction
    over TinyStories.}
    A causal decoder-only Transformer is trained separately for
    $\rho\in\{2,4,8,16,32,64,128\}$ while keeping the optimization
    protocol and total gradient-step budget fixed.
    Panel~(a) reports the training and test cross-entropy losses,
    whereas panel~(b) reports their difference,
    $\widehat{\operatorname{gen}}
    =\widehat{\mathcal L}_{\mathrm{test}}
    -\widehat{\mathcal L}_{\mathrm{train}}$.
    Shaded regions indicate one standard deviation across independent
    training realizations. Increasing the context length substantially
    improves next-token prediction while simultaneously enlarging the
    discrepancy between the training and test losses.}
    \label{fig:tinystories_context}
\end{figure*}

Figure~\ref{fig:tinystories_context} shows a pronounced dependence of
both prediction performance and generalization on context length.
As $\rho$ increases, the training cross-entropy loss decreases
substantially, indicating that the self-attention model effectively
exploits the additional context to fit the next-token prediction task.
The test loss also decreases over most of the considered range,
showing that a longer context provides useful predictive information
beyond the training data. In particular, the test loss continues to
improve up to relatively large context lengths before the improvement
eventually saturates.

At the same time, Fig.~\ref{fig:tinystories_context}(b) shows that the
generalization gap becomes progressively larger as the context length
increases, with a particularly pronounced increase for the largest
values of $\rho$. Importantly, the widening of the generalization gap
does not necessarily imply poorer test performance: over a substantial
range of context lengths, both the training and test losses decrease
while their difference increases.

The observed trend is qualitatively aligned with the role of context
length in our theoretical analysis, where $\rho$ enters explicitly
into the generalization penalty. We emphasize, however, that the
experiment uses a deeper multi-layer Transformer and is therefore not
intended to verify the exact finite-sample scaling derived for the
stylized SANTP architecture. Rather, it demonstrates that the
qualitative performance--generalization tradeoff highlighted by the
theory can also arise in a substantially more expressive next-token
predictor: increasing the available context can improve predictive
performance while simultaneously widening the empirical train--test
gap.



\subsection{Next-Token Prediction on ETTh2}
\label{subsec:etth2_main}

We next examine whether the context-length dependence observed in
language modeling also arises for real-world temporally dependent data.
To this end, we consider the ETTh2 time series and formulate the
univariate \texttt{OT} sequence as a discrete next-token prediction
problem. The range of the series is uniformly quantized into
$d=20$ symbols, and a causal self-attention predictor is trained to
predict the next symbol from the preceding $\rho$ observations.

We consider $\rho\in\{2,4,8,16,32,64,128\}$. For each context length, a two-layer causal self-attention model with
embedding dimension $64$ and $4$ attention heads is trained from
scratch. The total optimization budget is fixed at $15{,}000$ gradient
updates for every value of $\rho$. The experiment is repeated over
$5$ independent random seeds. A validation subset of the training data
is used only for checkpoint selection; the selected checkpoint is then
evaluated on the full training set and on a separate test set.

As in the language-model experiment, we define the empirical
generalization gap as
\begin{equation}
    \widehat{\operatorname{gen}}(\rho)
    =
    \widehat{\mathcal L}_{\mathrm{test}}(\rho)
    -
    \widehat{\mathcal L}_{\mathrm{train}}(\rho).
    \nonumber
\end{equation}
Further details on the data construction, architecture, optimization,
and checkpoint-selection procedure are provided in
Appendix~\ref{app:etth2_details}.

\begin{figure*}[t]
    \centering
    \begin{minipage}[t]{0.35\textwidth}
        \centering
        \includegraphics[width=\linewidth]{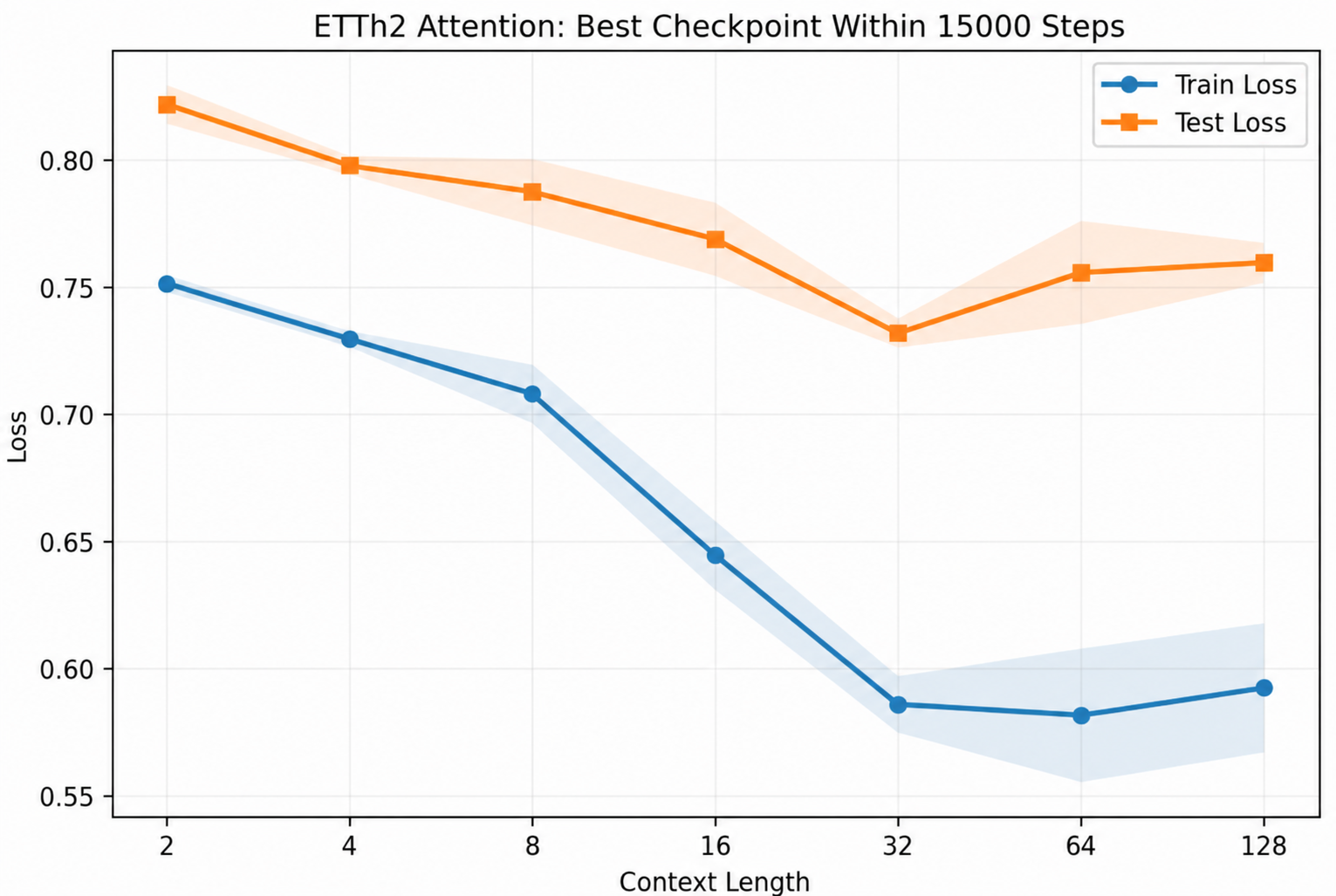}
        \small (a) Training and test cross-entropy losses.
    \end{minipage} 
    \hspace{0.1\textwidth}
    \begin{minipage}[t]{0.35\textwidth}
        \centering
        \includegraphics[width=\linewidth]{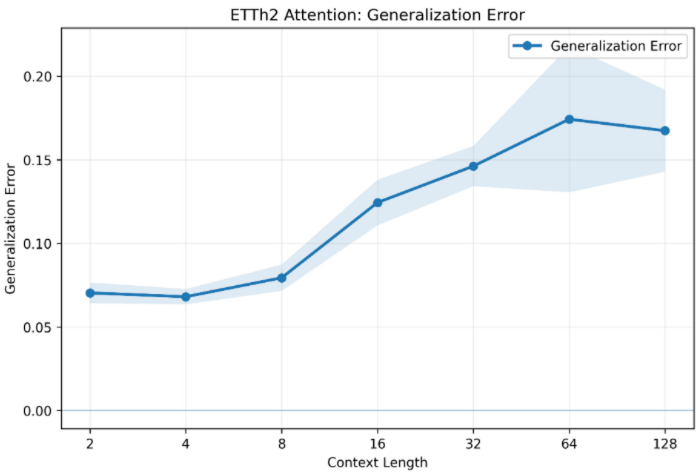}
        \small (b) Generalization gap.
    \end{minipage}
    \caption{\textbf{Effect of context length on next-token prediction
    over ETTh2.}
    The context length is varied over
    $\rho\in\{2,4,8,16,32,64,128\}$ while keeping the optimization
    protocol and the total budget of $15000$ gradient updates fixed.
    Panel~(a) reports the training and test cross-entropy losses, while
    panel~(b) reports
    $\widehat{\operatorname{gen}}
    =\widehat{\mathcal L}_{\mathrm{test}}
    -\widehat{\mathcal L}_{\mathrm{train}}$.
    Curves are averaged over $5$ independent random seeds, and shaded
    regions denote $95\%$ confidence intervals for the corresponding
    means.}
    \label{fig:etth2_context_main}
\end{figure*}

Figure~\ref{fig:etth2_context_main} shows that increasing the context
length substantially reduces the training loss. The test loss also
decreases initially, from approximately $0.82$ at $\rho=2$ to its
lowest value of approximately $0.73$ around $\rho=32$. Beyond this
range, however, the test loss no longer improves and instead increases
slightly, even though the training loss remains substantially smaller.

At the same time, the empirical generalization gap exhibits a clear
overall increase with context length. It is approximately $0.07$
for the shortest contexts and reaches values around $0.17$ for the
largest contexts. Although the dependence is not strictly monotone at
every individual value of $\rho$, the widening of the gap is
pronounced over the considered range.

These observations are qualitatively consistent with the
context-length dependence predicted by our generalization analysis for
self-attention next-token predictors. Increasing $\rho$ makes a longer
temporal history available to the model and initially improves its
out-of-sample predictive performance, but the benefit of this additional
history eventually saturates while the discrepancy between training and
test performance continues to widen. In
Appendix~\ref{app:etth2_memory}, we investigate this behavior further
through an independent empirical analysis of the predictive-memory scale
of ETTh2.



\section{Appendices}

\appendices



\subsection{On the Markovian random variables}\label{Markov_property_random_variable}

\begin{definition}[Marton coupling]\label{Marton_coupling}
Let $\mathbf{Z} \coloneqq (Z_1,\ldots,Z_T)$ be a random vector taking values in
$\mathcal{Z}_1 \times \cdots \times \mathcal{Z}_T$.
A \emph{Marton coupling} for $\mathbf{Z}$ is defined as a family of coupled random vectors
\begin{align}
    \left(
    \mathbf{Z}^{(z_1,\ldots,z_i,z'_i)},
    \mathbf{Z}'^{(z_1,\ldots,z_i,z'_i)}
    \right)
    \in
    \mathcal{Z} \times \mathcal{Z},
\end{align}
constructed for every $i \in [T]$ and every choice of
$z_1 \in \Omega_1, \ldots, z_i \in \Omega_i, z'_i \in \Omega_i$,
and satisfying the following conditions:
\begin{itemize}
    \item
    $
    \mathbf{Z}_1^{(z_1,\ldots,z_i,z'_i)} = z_1,\ \ldots,\ 
    \mathbf{Z}_i^{(z_1,\ldots,z_i,z'_i)} = z_i;
    $
    
    \item
    $
    {\mathbf{Z}'_1}^{(z_1,\ldots,z_i,z'_i)} = z_1,\ \ldots,\ 
    {\mathbf{Z}'_{i-1}}^{(z_1,\ldots,z_i,z'_i)} = z_{i-1},\ 
    {\mathbf{Z}'_i}^{(z_1,\ldots,z_i,z'_i)} = z'_i;
    $
    
    \item
    $
    \left(
    \mathbf{Z}^{(z_1,\ldots,z_i,z'_i)}_{i+1}, \ldots,
    \mathbf{Z}^{(z_1,\ldots,z_i,z'_i)}_T
    \right)
    \sim
    \mathbb{P}\!\left(
    Z_{i+1}, \ldots, Z_T
    \mid
    Z_1 = z_1, \ldots, Z_i = z_i
    \right);
    $
    
    \item
    $
    \left(
    {Z'_{i+1}}^{(z_1,\ldots,z_i,z'_i)}, \ldots,
    {Z'_T}^{(z_1,\ldots,z_i,z'_i)}
    \right)
    \sim
    \mathbb{P}\!\left(
    Z_{i+1}, \ldots, Z_T
    \mid
    Z_1 = z_1, \ldots, Z_{i-1} = z_{i-1}, Z_i = z'_i
    \right);
    $
    
    \item
    If $z_i = z'_i$, then
    $
    Z^{(z_1,\ldots,z_i,z'_i)} =
    Z'^{(z_1,\ldots,z_i,z'_i)}.
    $
\end{itemize}

Associated with a Marton coupling, the \emph{mixing matrix}
$\Gamma \coloneqq (\Gamma_{i,j})_{i,j \leq T}$
is defined as an upper-triangular matrix with diagonal entries
$\Gamma_{i,i} \coloneqq 1$ for all $i \leq T$, and off-diagonal entries given by
\begin{align*}
    \Gamma_{j,i} &\coloneqq 0, \qquad 1 \leq j < i \leq T, \\
    \Gamma_{i,j}
    &\coloneqq
    \sup_{z_1,\ldots,z_i,z'_i}
    \mathbb{P}\!\left[
    \mathbf{Z}^{(z_1,\ldots,z_i,z'_i)}_j
    \neq
    \mathbf{Z}'^{(z_1,\ldots,z_i,z'_i)}_j
    \right],
\end{align*}
for $1 \leq i < j \leq T$.
\end{definition}

\begin{definition}[Partition]
A \emph{partition} of a set $S$ is a collection of disjoint, non-empty subsets whose union is equal to $S$.
Analogously, let $\mathbf{Z} \coloneqq ({Z}_1,\ldots,{Z}_T)$ be a vector of random variables. We say that
$\hat{\mathbf{Z}} \coloneqq (\hat{\mathbf{Z}}_1,\ldots,\hat{\mathbf{Z}}_{\hat{T}})$ is a partition of $\mathbf{Z}$ if
$\{\hat{\mathbf{Z}}_i\}_{i=1}^{\hat{T}}$ forms a partition of the set $\{\mathbf{Z}_1,\cdots,\mathbf{Z}_T\}$.

For a given partition $\hat{\mathbf{Z}}$ of $\mathbf{Z}$, we denote by $s(\hat{\mathbf{Z}}_i)$ the number of elements in $\hat{\mathbf{Z}}_i$
(i.e., the size of $\hat{\mathbf{Z}}_i$), and define the \emph{size of the partition} as
\[
s(\hat{\mathbf{Z}}) \coloneqq \max_{1 \leq i \leq \hat{T}} s(\hat{\mathbf{Z}}_i).
\]

Moreover, we denote by $\mathcal{I}(\hat{\mathbf{Z}}_i)$ the set of indices corresponding to the elements of
$\hat{\mathbf{Z}}_i$, that is,
\[
{Z}_j \in \hat{\mathbf{Z}}_i \quad \text{if and only if} \quad j \in \mathcal{I}(\hat{\mathbf{Z}}_i).
\]
For any index set $S \subseteq [T]$, we write
\[
\mathbf{Z}_S \coloneqq \{Z_j : j \in S\}.
\]
In particular, $\hat{\mathbf{Z}}_i = Z_{\mathcal{I}(\hat{\mathbf{Z}}_i)}$.

Finally, if $Z$ takes values in the product space
$\mathcal{Z} \coloneqq \mathcal{Z}_1 \times \cdots \times \mathcal{Z}_T$,
then the partitioned random vector $\hat{\mathbf{Z}}$ takes values in
\[
\hat{\mathcal{Z}} \coloneqq \hat{\mathcal{Z}}_1 \times \cdots \times \hat{\mathcal{Z}}_{\hat{T}},
\]
where each $\hat{\mathbf{Z}}_i$ corresponds to the product space indexed by
$\mathcal{I}(\hat{\mathbf{Z}}_i)$.
\end{definition}

\begin{lemma}[McDiarmid-type inequality for dependent random variables]
\label{Mdiarmid_type_lemma}
Let $\mathbf{Z} \coloneqq (Z_1,\ldots,Z_T)$ be a sequence of random variables taking values in
$\mathcal{Z}$, with distribution $P$. Let
$\hat{\mathbf{Z}} \coloneqq (\hat{\mathbf{Z}}_1,\cdots,\hat{\mathbf{Z}}_{\hat{T}})$
be a partition of $Z$, taking values in $\hat{\mathcal{\mathbf{Z}}}$ with distribution $\hat{P}$.
Suppose that a Marton coupling exists for $\hat{\mathbf{Z}}$, with associated mixing matrix $\Gamma$.
Let $c \in \mathbb{R}_+^{T}$ and define $C(c) \in \mathbb{R}_+^{\hat{T}}$ by
\begin{align}
    C_i(c)
    \coloneqq
    \sum_{j \in \mathcal{I}(\hat{\mathbf{Z}}_i)} c_j,
    \qquad i \leq \hat{T}.\nonumber
\end{align}
Assume that the function $f:\mathcal{Z}\to\mathbb{R}$ satisfies
\begin{align}
    f(\mathbf{z})-f(\mathbf{z}')
    \leq
    \sum_{i=1}^{T} c_i \mathbf{1}\{z_i \neq z'_i\},
    \label{Mdiarmid_type_lemma_1}
\end{align}
for all $\mathbf{z},\mathbf{z}' \in \mathcal{Z}$. Then, for any $\lambda \in \mathbb{R}$,
\begin{align}
    \log \mathbb{E}\!\left[
    e^{\lambda (f(\mathbf{Z})-\mathbb{E}f(\mathbf{Z}))}
    \right]
    \leq
    \frac{\lambda^2 \|\Gamma C(c)\|^2}{8}.
    \nonumber
\end{align}
Consequently, for any $t \geq 0$,
\begin{align}
    \mathbb{P}\left(
    |f(\mathbf{Z})-\mathbb{E}f(\mathbf{Z})| > t
    \right)
    \leq
    2 \exp\!\left(
    -\frac{2t^2}{\|\Gamma C(c)\|^2}
    \right).\nonumber
\end{align}
\end{lemma}

\begin{corollary}[McDiarmid's inequality for Markov chains]
\label{mcdiarmid_inequality_markov_chain}
Let $\mathbf{Z} \coloneqq (Z_1,\ldots,Z_T)$ be a Markov chain taking values in
$\mathcal{Z}_1 \times \cdots \times \mathcal{Z}_T$, with mixing time $\tau(\delta)$
for $0 \leq \delta \leq 1$. Define
\begin{align}
    \tau_{\min}
    \coloneqq
    \inf_{0 \leq \delta \leq 1}
    \tau(\delta)
    \left(\frac{2-\delta}{1-\delta}\right)^2.
    \nonumber
\end{align}
Suppose that the function $f:\mathcal{Z}^T \to \mathbb{R}$ satisfies
\eqref{Mdiarmid_type_lemma_1} for some $c \in \mathbb{R}_+^{T}$.
Then, for any $t \geq 0$,
\begin{align}
    \mathbb{P}\!\left(
    |f(\mathbf{Z})-\mathbb{E}f(\mathbf{Z})| \geq t
    \right)
    \leq
    2 \exp\!\left(
    -\frac{2t^2}{\|c\|^2 \tau_{\min}}
    \right).\nonumber
\end{align}
Equivalently, for any $s \in \mathbb{R}$,
\begin{align}
    \log \mathbb{E}\!\left[
    e^{s (f(\mathbf{Z})-\mathbb{E}f(\mathbf{Z}))}
    \right]
    \leq
    \frac{s^2 \|c\|^2 \tau_{\min}}{8}.
   \nonumber
\end{align}
\end{corollary}



\begin{lemma}
\label{lem:window-loss-concentration}
Let $\{Z_t\}_{t=1}^{T}$ be a Markov process of order $\rho$, and let
$S_t:=(Z_t,Z_{t-1},\ldots,Z_{t-\rho+1})$. For a fixed
$w\in\mathcal W$, define
\[
Y_t:=\frac{1}{T}
\log\frac{1}{
p_w(Z_t\mid Z_{t-1},\ldots,Z_{t-\rho})}.
\]
Let $\tilde{\tau}(\delta)$ denote the mixing time of the
history-state process $\{S_t\}_{t\in[T]}$, defined by
\[
\tilde{\tau}(\delta)
:=
\min\left\{
m\geq1:
\max_i\sup_{s',s''}
d_{\mathrm{TV}}\!\left(
\mathbb{P}(S_{i+m}\in\cdot\mid S_i=s'),
\mathbb{P}(S_{i+m}\in\cdot\mid S_i=s'')
\right)
\leq\delta
\right\}.
\]
Then, for any function $f:\mathbb R^T\to\mathbb R$ satisfying
the bounded-difference condition in
\eqref{Mdiarmid_type_lemma_1} with coefficients
$c=(c_1,\ldots,c_T)\in\mathbb R_+^T$, and for every
$\lambda\in\mathbb R$,
\begin{equation}
\log\mathbb E
\left[e^{\lambda\big(f(Y)-\mathbb Ef(Y)\big)}\right]
\le
\frac{\lambda^2\tilde{\tau}_{\min}\|c\|_2^2}{8},
\label{Concentration_under_condition_A}
\end{equation}
where
\begin{equation}
\tilde{\tau}_{\min}
:=
\inf_{0<\delta<1}
\tilde{\tau}(\delta)
\left(\frac{2-\delta}{1-\delta}\right)^2.
\end{equation}
\end{lemma}









\subsection{Proof of Theorem \ref{multi_round_markovian_general_bound}}\label{Proof_multi_round}

By the definition of the generalization gap in \eqref{generalization_croos_entropy_def}, we have
\begin{align}
    &\mathbb{E}_{\mathbf{S},{W}}\left[\gen(\mathbf{S},{W})\right]\label{multi_round_generalization_bound_1}\\
    &\qquad=\mathbb{E}\left[\mathcal{L}({W})-\hat{\mathcal{L}}(\mathbf{S},{W})\right]\nonumber\\
    &\qquad=\frac{1}{\lambda}\mathbb{E}\Bigg[-\frac{\lambda}{T}\sum_{t=1}^{T}\mathbb{E}_{\mathbf{S}'}\left[\log p_{W}\left(Z'_{t}|Z'_{t-1},\cdots,Z'_{t-\rho}\right)\right]\nonumber\\
    &\qquad\qquad+\frac{\lambda}{N T}\sum_{n=1}^{N}\sum_{t=1}^{T}\mathbb{E}\left[\log p_{W}\left(Z^{(n)}_{t}|Z^{(n)}_{t-1},\cdots,Z_{t-\rho}\right)\right]\Bigg]\nonumber\\
    &\qquad=\frac{1}{\lambda}\mathbb{E}\Bigg[\mathbb{E}\Bigg[-\frac{\lambda}{N T}\sum_{t=1}^{T}\sum_{n=1}^{N}\mathbb{E}_{\mathbf{S}'}\left[\log p_{W}(Z'^{(n)}_{t}|Z'^{(n)}_{t-1},\cdots,Z'^{(n)}_{t-\rho})\right]\label{multi_round_generalization_bound_6}\\
    &\qquad\qquad+\frac{\lambda}{N T}\sum_{t=1}^{T}\sum_{n=1}^{N}\log p_{{W}}(Z^{(n)}_{t}|Z^{(n)}_{t-1},\cdots,Z^{(n)}_{t-\rho})\Bigg]\Bigg]\nonumber\\
    &\qquad\leq\frac{1}{\lambda}\mathbb{E}\left[D\left(P_{\mathbf{S}|{W}}\|P_{\mathbf{S}}\right)+\log\mathbb{E}_{\mathbf{S}}\left[e^{\phi(\mathbf{S},\mathbf{S}',{W})}\right]\right]\label{multi_round_generalization_bound_9}\\
    &\qquad=\frac{1}{\lambda}\left[I(\mathbf{S};{W})+\mathbb{E}_{{W}}\left[\log\mathbb{E}_{\mathbf{S}}\left[e^{\phi(\mathbf{S},\mathbf{S}',{W})}\right]\right]\right]\label{multi_round_generalization_bound_10}
\end{align}

where
\begin{itemize}
    \item $\mathbf{S}'$ denotes an independent copy of $\mathbf{S}$;
    
    \item the expectations in \eqref{multi_round_generalization_bound_1}--\eqref{multi_round_generalization_bound_9} are taken with respect to the random variables $(\mathbf{S},W)$;
    
    
    \item the inner and outer expectations in \eqref{multi_round_generalization_bound_6} correspond to the conditional expectation of $\mathbf{S}$ given $W$ and the expectation over ${W}$, respectively;
    
    
    \item the inequality in \eqref{multi_round_generalization_bound_9} follows from the Donsker--Varadhan variational representation by substituting
    \begin{align}
        &\phi\left(\mathbf{S},\mathbf{S}',{W}\right)=\frac{\lambda}{N}
         \sum_{n=1}^{N}
        \left(
            \mathbb{E}_{\mathbf{S}'}\!\left[
                \ell\left(\mathbf{Z}^{(n)},{W}\right)
            \right]
            -\ell\left(\mathbf{Z}^{(n)},{W}\right)
        \right),\nonumber
    \end{align}
    and choosing the probability distributions $P = P_{\mathbf{S}\,|{W}}$ and $Q = P_{\mathbf{S}}$ in
    \begin{align}
        \mathbb{E}_{P}[\phi]
        \;\leq\;
        D(P\|Q) + \log \mathbb{E}_{Q}\!\left[e^{\phi}\right].\nonumber
    \end{align}

    \item the equivalence between \eqref{multi_round_generalization_bound_9} and \eqref{multi_round_generalization_bound_10} follows from the definition of mutual information, which states that
\begin{align}
    I\!\left(\mathbf{S}; {W}\right)
    = \mathbb{E}\left[
        D\!\left(
            P_{\mathbf{S}\,|W}
            \,\big\|\, 
            P_{\mathbf{S}}
        \right)
    \right],\nonumber
\end{align}
where the expectation is taken with respect to the probability distribution of 
${W}$.
\end{itemize}

Using \eqref{multi_round_generalization_bound_1}, we obtain
\begin{align}
&\mathbb{E}_{\mathbf{S},{W}}\left[\gen(\mathbf{S}, {W})\right]\nonumber\\
& \leq \frac{1}{\lambda}\left[I(\mathbf{S};W)+\mathbb{E}_{W}\left[\log\mathbb{E}_{\mathbf{S}}\left[e^{\phi(\mathbf{S},\mathbf{S}',{W})}\right]\right]\right]\nonumber\\
&=\frac{I(\mathbf{S};W)}{\lambda}+\frac{1}{\lambda}\mathbb{E}_{W}\Bigg[\log\mathbb{E}_{\mathbf{S}}\Bigg[ e^{\sum_{n=1}^{N}\frac{\lambda}{N}\left(\mathbb{E}_{\mathbf{S}'}\left[\ell(\mathbf{Z}'^{(n)},{W})\right]-\ell(\mathbf{Z}^{(n)},{W})\right)}\Bigg]\Bigg]\nonumber\\
&=\frac{I(\mathbf{S};W)}{\lambda}+\frac{1}{\lambda}\mathbb{E}_{W}\Bigg[\log\prod_{n=1}^{N}\mathbb{E}_{\mathbf{S}}\Bigg[ e^{\frac{\lambda}{N}\left(\mathbb{E}_{\mathbf{S}'}\left[\ell(\mathbf{Z}'^{(n)},{W})\right]-\ell(\mathbf{Z}^{(n)},{W})\right)}\Bigg]\Bigg]\nonumber\\
&=\frac{I(\mathbf{S};W)}{\lambda}+\frac{1}{\lambda}\sum_{n=1}^{N}\mathbb{E}_{W}\Bigg[\log\mathbb{E}_{\mathbf{S}}\Bigg[ e^{\frac{\lambda}{N}\left(\mathbb{E}_{\mathbf{S}'}\left[\ell(\mathbf{Z}'^{(n)},{W})\right]-\ell(\mathbf{Z}^{(n)},{W})\right)}\Bigg]\Bigg]\label{bounding_moment_10}
\end{align}
where
\begin{itemize}
    \item \({\mathbf{Z}^{(n)}} \in \mathbb{R}^{{T}} \) denotes the collection \( \{Z^{(n)}_{1},\cdots,Z^{(n)}_{T}\} \). For this purpose, let \( n \in [N] \) and \( \mathbf{Z}^{(n)}, \mathbf{Z}'^{(n)}\in \mathbb{R}^T \);
\end{itemize}

It remains to derive an upper bound on the second term in \eqref{bounding_moment_10}. To bound this moment-generating function, we apply McDiarmid's inequality for Markov chains, as stated in Corollary~\ref{mcdiarmid_inequality_markov_chain}. To this end, we show that the property defined in \eqref{Mdiarmid_type_lemma_1} holds for the loss function \( \ell\left( \mathbf{z}^{(n)}, {w} \right) \). For any fixed values of \( {w} \) and $n\in[N]$, this function is defined as
\begin{align}
    \ell\left(\mathbf{z}^{(n)},w\right) = -\frac{1}{T} \sum_{t \in [T]} \log p_{w}(z^{(n)}_{t}|z^{(n)}_{t-1},\cdots,z^{(n)}_{t-\rho}).\nonumber
\end{align}

We consider the following telescoping sum:
\begin{align}
    &\ell\left( \mathbf{z}'^{(n)},{w}\right)-\ell\left(\mathbf{z}^{(n)},{w} \right)\label{croos_entropy_loss_func_7}\\
        &=\frac{1}{T}\sum_{t\in[T]}\Bigg[\log p_{w}\left({z}^{(n)}_{t}\big|{z}^{(n)}_{t-1},\dots, {z}^{(n)}_{t-\rho}\right) -\log p_{w}\!\left(
{z}_{t}^{\prime(n)}\,\middle|\,{z}_{t}^{\prime(n)},
\ldots,{z}_{t-\rho}^{\prime(n)} \right)\Bigg]\nonumber\\
        &=\frac{1}{T}\sum_{t\in[T]}\Bigg[\log p_{w}\left(\mathbf{z}^{(n)}_{t}|\mathbf{z}^{(n)}_{t-1}\cdots,\mathbf{z}^{(n)}_{t-\rho}\right)-\log p_{w}\left(z^{(n)}_{t}|z^{(n)}_{t-1},\cdots,z'^{(n)}_{t-\rho}\right)\Bigg]\nonumber\\
        &\qquad+\frac{1}{T}\sum_{t\in[T]}\Bigg[\log p_{w}\left(z^{(n)}_{t}|z^{(n)}_{t-1},\cdots,z'^{(n)}_{t-\rho}\right)-\log p_{w}\left(z^{(n)}_{t}|z^{(n)}_{t-1},\cdots,z'^{(n)}_{t-\rho+1},z'^{(n)}_{t-\rho}\right)\Bigg]\nonumber\\
        &\qquad\qquad\qquad\qquad\qquad\qquad\vdots\nonumber\\
        &\qquad+\frac{1}{T}\sum_{t\in[T]}\Bigg[\log p_{w}\left(z^{(n)}_{t}|z'^{(n)}_{t-1}\dots,z'^{(n)}_{t-\rho}\right)-\log p_{w}\left(z'^{(n)}_{t}| z'^{(n)}_{t-1}\dots, z'^{(n)}_{t-\rho}\right)\Bigg]\\
        &=\frac{1}{T}\sum_{t=1}^{T}\Bigg[\log p_{w}\left(z^{(n)}_{t}|z'^{(n)}_{t-1}\dots,z'^{(n)}_{t-\rho}\right)-\log p_{w}\left(z'^{(n)}_{t}| z'^{(n)}_{t-1}\cdots, z'^{(n)}_{t-\rho}\right)\nonumber\\
        &\qquad+\!\!\!\!\!\!\sum_{i\in[1:T-1]}\!\!\Bigg[\log p_{w}\left(z^{(n)}_{t}|z^{(n)}_{t-1},\cdots,z^{(n)}_{t-i},z'^{(n)}_{t-i-1},\cdots,z'^{(n)}_{t-\rho}\right) -\log p_{w}\!\left(\!
z_{t}^{(n)}\,\middle|\,z_{t-1}^{(n)},\cdots,z_{t-i}^{\prime(n)},\cdots,z_{t-\rho}^{\prime(n)}\right)\Bigg]\Bigg]\label{Difference_loss_telescope_sum_1}.
\end{align}

We now derive an upper bound on the terms in the last expression.

By the definition of \( p_{w}(z^{(n)}_{t}|z^{(n)}_{t-1},\cdots,z^{(n)}_{t-\rho})\), for \(i\geq \rho\), the difference between two terms in \eqref{Difference_loss_telescope_sum_1} is zero because the conditional probability distribution depends only on \( \{ z^{(n)}_{t-1}, \dots, z^{(n)}_{t-\rho} \} \). Therefore, without loss of generality, we can rewrite the last terms as follows:
\begin{align}
&\sum_{t\in[T]}\Bigg[\log p_{w}\left(z^{(n)}_{t}|z'^{(n)}_{t-1}\dots,z'^{(n)}_{t-\rho}\right)-\log p_{w}\!\left(z'^{(n)}_{t}| z'^{(n)}_{t-1}\dots, z'^{(n)}_{t-\rho}\right)\nonumber\\
        &\quad+\!\!\!\!\!\!\sum_{i\in[1:T-1]}\!\!\Bigg[\log p_{w}\left(z^{(n)}_{t}|z^{(n)}_{t-1},\cdots,z^{(n)}_{t-i},z'^{(n)}_{t-i-1},\cdots,z'^{(n)}_{t-\rho}\right)-\log p_{w}\!\left(\!
z_{t}^{(n)}\,\middle|\,z_{t-1}^{(n)},\cdots,z_{t-i}^{\prime(n)},\cdots,z_{t-\rho}^{\prime(n)}\right)\!\Bigg]\!\Bigg]\nonumber\\
    &=\sum_{t=1}^{T}\Bigg[\log p_{w}\left(z^{(n)}_{t}|z'^{(n)}_{t-1}\dots,z'^{(n)}_{t-\rho}\right)-\log p_{w}\!\!\left(z'^{(n)}_{t}| z'^{(n)}_{t-1}\dots, z'^{(n)}_{t-\rho}\right)\Bigg]\nonumber\\
    &
    \quad+\sum_{t=1}^{T}\sum_{i=1}^{T-1}\Bigg[\!\!\log p_{w}\left(z^{(n)}_{t}|z^{(n)}_{t-1},\cdots,z^{(n)}_{t-i},z'^{(n)}_{t-i-1},\cdots,z'^{(n)}_{t-\rho}\right)-\log p_{w}\!\left(\!
z_{t}^{(n)}\,\middle|\,z_{t-1}^{(n)},\cdots,z_{t-i}^{\prime(n)},\cdots,z_{t-\rho}^{\prime(n)}\right)\!\!\Bigg]\mathbf{1}_{\left\{1\leq i\leq\rho\right\}}\label{telescope_serie_conv_4}\\
    &=\sum_{t=1}^{T}\Bigg[\log p_{w}\left(z^{(n)}_{t}|z'^{(n)}_{t-1}\dots,z'^{(n)}_{t-\rho}\right)-\log p_{w}\left(z'^{(n)}_{t}| z'^{(n)}_{t-1}\dots, z'^{(n)}_{t-\rho}\right)\Bigg]\mathbf{1}_{\left\{z^{(n)}_{t}\neq z'^{(n)}_{t}\right\}}\nonumber\\
    &\quad+\sum_{t=1}^{T}\sum_{i=1}^{\rho}\Bigg[\log p_{w}\left(z^{(n)}_{t}|z^{(n)}_{t-1},\cdots,z^{(n)}_{t-i},z'^{(n)}_{t-i-1},\cdots,z'^{(n)}_{t-\rho}\right)-\log p_{w}\!\left(\!
z_{t}^{(n)}\,\middle|\,z_{t-1}^{(n)},\cdots,z_{t-i}^{\prime(n)},\cdots,z_{t-\rho}^{\prime(n)}\right)\!\!\!\Bigg]\mathbf{1}_{\left\{z^{(n)}_{t-i}\neq z'^{(n)}_{t-i}\right\}},\label{telescope_serie_conv_5}
\end{align}
where
\begin{itemize}
    \item Equations~\eqref{telescope_serie_conv_4}--\eqref{telescope_serie_conv_5} follow from the definition of \( p_{w}(z^{(n)}_{t}|z^{(n)}_{t-1},\cdots,z^{(n)}_{t-\rho}) \), which depends solely on \( \{z^{(n)}_{t-\rho}, \dots, z^{(n)}_{t-1}\} \). This implies that, for \( i \notin [1:\rho] \), the difference between any pair is zero.
\end{itemize}

Taking the maximum over all pairs $(z^{(n)}_{1}, \cdots, z^{(n)}_{T})$ and $(z'^{(n)}_{t},\cdots,z'^{(n)}_{T})$ gives
\begin{align}
        &\log p_{w}\!\left(\!z^{(n)}_{t}|z'^{(n)}_{t-1}\cdots,z'^{(n)}_{t-\rho}\right)-\log p_{w}\!\left(\!z'^{(n)}_{t}|z'^{(n)}_{t-1}\cdots, z'^{(n)}_{t-\rho}\right)\nonumber\\
        &\quad\leq\max\Bigg|\log p_{w}\left(z^{(n)}_{t}|z'^{(n)}_{t-1}\dots,z'^{(n)}_{t-\rho}\right)-\log p_{w}\!\left(\!z'^{(n)}_{t}| z'^{(n)}_{t-1}\dots, x'^{(n)}_{k,r,t-\rho}\right)\Bigg|\label{Defnition_max_difference_telescope_2}\\
        &\quad=\left|\log\frac{p_{w}\left(z^{(n)}_{t}|z'^{(n)}_{t-1}\dots,z'^{(n)}_{t-\rho}\right)}{p_{w}\left(z'^{(n)}_{t}| z'^{(n)}_{t-1}\dots, z'^{(n)}_{t-\rho}\right)}\right|\nonumber\\
        &\quad\coloneqq C_1\label{Defnition_max_difference_telescope_1}
\end{align}

In addition, for all $i\in[1:\rho]$ and $t\in[T]$, by considering the pairs $\left(z^{(n)}_{1},\cdots z^{(n)}_{T}\right)$ and
$\left(z'^{(n)}_{1},\cdots, z'^{(n)}_{T}\right)$, we have
\begin{align}
        &\log p_{w}\left(z^{(n)}_{t}|z^{(n)}_{t-1},\cdots,z^{(n)}_{t-i},z'^{(n)}_{t-i-1},\cdots,z'^{(n)}_{t-\rho}\right)\label{telescope_serie_conv_10}\\
    &\qquad\qquad-\log p_{w}\!\left(\!
z_{t}^{(n)}\,\middle|\,z_{t-1}^{(n)},\cdots,z_{t-i}^{\prime(n)},\cdots,z_{t-\rho}^{\prime(n)}\right)\nonumber\\
    &\leq \max\Bigg|\log p_{w}\left(z^{(n)}_{t}|z^{(n)}_{t-1},\cdots,z^{(n)}_{t-i},z'^{(n)}_{t-i-1},\cdots,z'^{(n)}_{t-\rho}\right)\label{telescope_serie_conv_11}\\
    &\qquad\qquad-\log p_{w}\!\left(\!
z_{t}^{(n)}\,\middle|\,z_{t-1}^{(n)},\cdots,z_{t-i}^{\prime(n)},\cdots,z_{t-\rho}^{\prime(n)}\right)\Bigg|\label{telescope_serie_conv_12}\\
    &\coloneqq {C_2}\nonumber
\end{align}
where the maximum is taken over pairs $\left(z^{(n)}_{1},\cdots z^{(n)}_{T}\right)$ and $\left(z'^{(n)}_{1},\cdots, z'^{(n)}_{T}\right)$, with $t\in[T]$ and $i\in[1:\rho]$.

Therefore, we can upper-bound \eqref{telescope_serie_conv_5} as follows:
\begin{align}
    &\eqref{telescope_serie_conv_5}\leq\Bigg[\max\Bigg|\log p_{w}\left(z^{(n)}_{t}|z'^{(n)}_{t-1}\dots,z'^{(n)}_{t-\rho}\right)-\log p_{w}\left(z'^{(n)}_{t}| z'^{(n)}_{t-1}\dots, z'^{(n)}_{t-\rho}\right)\Bigg|\Bigg]\sum_{t=1}^{T}\mathbf{1}_{\left\{z^{(n)}_{t}\neq z'^{(n)}_{t}\right\}}\label{telescope_serie_conv_16}\\
    &\qquad+\Bigg[\max\Bigg|\log p_{w}\left(z^{(n)}_{t}|z^{(n)}_{t-1},\cdots,z^{(n)}_{t-i},z'^{(n)}_{t-i-1},\cdots z'^{(n)}_{t-\rho}\right)\label{telescope_serie_conv_6}\\
    &\qquad\qquad-\log p_{w}\!\!\left(z^{(n)}_{t}|z^{(n)}_{t-1},\cdots,z'^{(n)}_{t-i},\cdots z'^{(n)}_{t-\rho}\right)\!\!\Bigg|\Bigg]\!\!\times S_{T,\rho}\label{telescope_serie_conv_7}\\
    &\qquad\quad\leq \left(C_1+\rho C_2\right)\times\sum_{t=1}^{T}\mathbf{1}_{\{z^{(n)}_{t}\neq z'^{(n)}_{t}\}},\label{telescope_serie_conv_9}
\end{align}
where
\begin{itemize}

\item the maxima in \eqref{telescope_serie_conv_16} and \eqref{telescope_serie_conv_6} are taken over the same variables as before;

\item $S_{T,\rho}$ denotes
\begin{align}
    \sum_{t=1}^{T}\sum_{i=1}^{\rho}\mathbf{1}_{\left\{z^{(n)}_{t-i}\neq z'^{(n)}_{t-i}\right\}},\nonumber
\end{align}

\item \eqref{telescope_serie_conv_9} holds because
\begin{align}
    \sum_{t=1}^{T}\sum_{i=1}^{\rho}\mathbf{1}\{z^{(n)}_{t-i}\neq z'^{(n)}_{t-i}\}\leq\rho\sum_{t=1}^{T}\mathbf{1}\{z^{(n)}_{t}\neq z'^{(n)}_{t}\},
\end{align}
which holds because any difference between \( z^{(n)}_{t} \) and \( z'^{(n)}_{t} \) propagates at most \( \rho \) times.

We can now rewrite \eqref{croos_entropy_loss_func_7} with the following upper bound:
\begin{align}
    &\ell(\mathbf{z}^{(n)},{w})-\ell(\mathbf{z},{w})\leq\frac{(C_1+\rho C_2)}{T}\times\sum_{t=1}^{T}\mathbf{1}\left\{z^{(n)}_{t}\neq z'^{(n)}_{t}\right\},\label{difference_worse_bound_2}
\end{align}
which has the same form as \eqref{Mdiarmid_type_lemma_1}, where $c_{t}$ is bounded by
\begin{align}
    &\frac{(C_1+\rho C_2)}{T}.\label{difference_moment_max}
\end{align}
\end{itemize}

\begin{lemma}\label{max_probability_conditional_markov}
    For any $w\in\mathcal{W}$, we have
    \begin{align}
        \log\frac{1}{p_{{{w}}}(z_{t}|z_{t-1},\cdots,z_{t-\rho})}\leq 2B+\log d.\nonumber
    \end{align}
\end{lemma}

The proof is provided in Appendix~\ref{Proof_max_probability_conditional_markov}.

To complete the proof, we need to establish an upper bound for \eqref{difference_moment_max}, uniformly over all realizations of $\{z^{(n)}_{1},\cdots,z^{(n)}_{T}\}$ and $\{z'^{(n)}_{1},\cdots,z'^{(n)}_{T}\}$ and all $t\in[T],i\in[1:\rho]$, in order to bound \( c_{t} \) as follows:
\begin{align}
    \eqref{Defnition_max_difference_telescope_1}&=\max\left|\log\frac{p_{w}\left(z^{(n)}_{t}|z'^{(n)}_{t-1}\dots,z'^{(n)}_{t-\rho}\right)}{p_{w}\left(z'^{(n)}_{t}| z'^{(n)}_{t-1}\dots, z'^{(n)}_{t-\rho}\right)}\right|\\
    &\leq\max\left[\log\frac{1}{p_{w}\left(z'^{(n)}_{t}|z'^{(n)}_{t-1}\dots, z'^{(n)}_{t-\rho}\right)}\right]\\
    &\leq 2B+\log d,\label{upperbound_negative_loss}
\end{align}
where
\begin{itemize}
    \item the maximum is taken over the same variables as in \eqref{Defnition_max_difference_telescope_2};
    \item the last inequality follows from Lemma~\ref{max_probability_conditional_markov}.
\end{itemize}

On the other hand, we have
\begin{align}
        \eqref{telescope_serie_conv_12}&=\max\Bigg|\log p_{w}\left(z^{(n)}_{t}|z^{(n)}_{t-1},\cdots,z^{(n)}_{t-i+1},z'^{(n)}_{t-i},\cdots z'^{(n)}_{t-\rho}\right)\nonumber\\
    &\qquad\qquad-\log p_{w}\left(z^{(n)}_{t}|z^{(n)}_{t-1},\cdots,z'^{(n)}_{t-i+1},z'^{(n)}_{t-i},\cdots z'^{(n)}_{t-\rho}\right)\Bigg|\\
        &=\max\Bigg|\log\left[\frac{p_{w}\left(z^{(n)}_{t}|z^{(n)}_{t-1},\cdots,z^{(n)}_{t-i+1},z'^{(n)}_{t-i},\cdots z'^{(n)}_{t-\rho}\right)}{p_{w}\left(z^{(n)}_{t}|z^{(n)}_{t-1},\cdots,z'^{(n)}_{t-i+1},z'^{(n)}_{t-i},\cdots z'^{(n)}_{t-\rho}\right)}\right]\Bigg|\label{croos_entropy_loss_func_4}\\
&=\max\Bigg|\log\left[\frac{\big<\sigma\left(g_{w}\left(\mathcal{E}(z^{(n)}_{t-1}),\cdots,\mathcal{E}(z'^{(n)}_{t-i}),\cdots,\mathcal{E}(z'^{(n)}_{t-\rho})\right)\right),\mathcal{E}(z^{(n)}_{t})\big>}{\big<\sigma\left(g_{w}\left(\mathcal{E}(z^{(n)}_{t-1}),\cdots,\mathcal{E}(z'^{(n)}_{t-i+1}),\cdots,\mathcal{E}(z'^{(n)}_{t-\rho})\right)\right),\mathcal{E}(z^{(n)}_{t})\big>}\right]\Bigg|\label{croos_entropy_loss_func_1}
\end{align}
where
\begin{itemize}
    \item in \eqref{croos_entropy_loss_func_1}, the conditional probability distribution is substituted using the definition in \eqref{conditional_distribution_def}, which is the inner product between
\begin{align}
\sigma\left(g_{w}\left(\mathcal{E}(z^{(n)}_{t-1}), \cdots, \mathcal{E}(z^{(n)}_{t-\rho})\right)\right),\nonumber
\end{align}
and $\mathcal{E}(z^{(n)}_{t})$.
\end{itemize}

Without loss of generality, assume that the maximum in \eqref{croos_entropy_loss_func_1}
is attained for the pair of sequences
\[
\bigl(z^{(n)}_{1},\ldots,z^{(n)}_{T}\bigr)
\quad \text{and} \quad
\bigl(z'^{(n)}_{1},\ldots,z'^{(n)}_{T}\bigr).
\]
Then, by rewriting \eqref{croos_entropy_loss_func_1} using the softmax
function defined in \eqref{soft_max_def}, we obtain the following:
\begin{align}
&g_{w,j} = g_w\left(\mathcal{E}(z^{(n)}_{t-1}), \cdots,\mathcal{E}(z'^{(n)}_{t-i}),\cdots,\mathcal{E}(z'^{(n)}_{t-\rho})\right)_j,\label{element_wise_g_funct_1} \\
&g'_{w,j} = g_w\left(\mathcal{E}(x^{(n)}_{t-1}), \cdots,\mathcal{E}(z'^{(n)}_{t-i+1}),\cdots,\mathcal{E}(z'^{(n)}_{t-\rho})\right)_j,\label{element_wise_g_funct_2}
\end{align}
and
\begin{align}
&g^{(j)}_{w} = \frac{\exp({g_{w,j}})}{\sum_{j \in [d]} \exp({g_{w,j}})}, \qquad g'^{(j)}_{w} = \frac{\exp({g'_{w,j}})}{\sum_{j \in [d]} \exp({g'_{w,j}})},\label{change_measure_g_func_1}
\end{align}
where \( g_w\left(\mathcal{E}(z^{(n)}_{t-1}), \cdots, \mathcal{E}(z^{(n)}_{t-\rho})\right)_j \) and $\mathcal{E}_j(z^{(n)}_{t})$ denote the \(j\)-th coordinates of the vectors
\[
g_w\left(\mathcal{E}(z^{(n)}_{t-1}), \cdots, \mathcal{E}(z^{(n)}_{t-\rho})\right)
\]
and $\mathcal{E}(z^{(n)}_{t})$ in $\mathbb{R}^d$, respectively.

With these definitions, we have
\begin{align}
        \eqref{croos_entropy_loss_func_1} &= \Bigg| \left[\log\left[\frac{\sum_{j\in[d]}g^{(j)}_{w}\cdot\mathcal{E}_j(z^{(n)}_{t})}{\sum_{j\in[d]}g'^{(j)}_{w}\cdot\mathcal{E}_j(z^{(n)}_{t})}\right]\right]\nonumber\\
        &=\Bigg|\log\left[\frac{\sum_{j\in[d]}\exp({g_{w,j}})\cdot\mathcal{E}_j(z^{(n)}_{t})}{\sum_{j\in[d]}\exp({g'_{w,j}})\cdot\mathcal{E}_j(z^{(n)}_{t})}\right] 
        +\log\left[\frac{\sum_{j\in[d]}\exp({g'_{w,j}})}{\sum_{j\in[d]}\exp({g_{w,j}})}\right]\Bigg|\nonumber\\
&\leq\Bigg|\log\left[\frac{\sum_{j\in[d]}\exp({g_{w,j}})\cdot\mathcal{E}_j(z^{(n)}_{t})}{\sum_{j\in[d]}\exp({g'_{w,j}})\cdot\mathcal{E}_j(z^{(n)}_{t})}\right]\Bigg|
          +\Bigg|\left[\log\left(\frac{\sum_{j\in[d]}\exp({g'_{w,j}})}{\sum_{j\in[d]}\exp({g_{w,j}}})\right)\right]\Bigg|\label{simplification_difference_cross_entropy_11}\\
        &\leq\max\Bigg\{\max_{j\in[d]}\log\left[\frac{\exp({g_{w,j}})\cdot\mathcal{E}_j(z^{(n)}_{t})}{\exp({g'_{w,j}})\cdot\mathcal{E}_j(z^{(n)}_{t})}\right],\max_{j\in[d]}\log\left[\frac{\exp({g'_{w,j}})\cdot\mathcal{E}_j(z^{(n)}_{t})}{\exp({g_{w,j}})\cdot\mathcal{E}_j(z^{(n)}_{t})}\right]\Bigg\}\label{simplification_difference_croos_entropy_12}\\
        &\quad +\max\Bigg\{\!\max_{j\in[d]}\log\!\left[\frac{\exp({g_{w,j}})}{\exp({g'_{w,j}})}\right]\!,\max_{j\in[d]}\log\!\left[\frac{\exp({g'_{w,j}})}{\exp({g_{w,j}})}\right]\!\!\Bigg\}\label{simplification_difference_croos_entropy_15}\\
        &=\max_{j\in[d]}\Big|g_{w,j}-g'_{w,j}\Big|+\max_{j\in[d]}\left|g_{w,j}-g'_{w,j}\right|\nonumber\\
        &=2\max_{j\in[d]}\left|g_{w,j}-g'_{w,j}\right|\label{simplification_difference_croos_entropy_17}\\
        &\leq 4B\label{simplification_difference_croos_entropy_19}
\end{align}
where
\begin{itemize}
    \item Equations~\eqref{simplification_difference_cross_entropy_11} and \eqref{simplification_difference_croos_entropy_17} follow from the triangle inequality, i.e.,
    \begin{align}
        |x+y|\leq |x|+|y|,\nonumber
    \end{align}

    \item Equations~\eqref{simplification_difference_croos_entropy_12}--\eqref{simplification_difference_croos_entropy_15} follow from the following inequality:
    \begin{align}
        &\Bigg|\log\left(\frac{\sum_{j=1}^{d} \exp({g_{{w},j}})}{\sum_{j=1}^{d}  \exp({g'_{{w},j}})}\right)\Bigg|\nonumber\\
        &=\max\Bigg\{\!\!\log\left(\!\frac{\sum_{j=1}^{d} \exp({g_{w,j}})}{\sum_{j=1}^{d} \exp({g'_{w,j}})}\right)\!,\log\left(\!\frac{\sum_{j=1}^{d} \exp({g'_{w,i}})}{\sum_{j=1}^{d} \exp({g_{w,j}})}\right)\!\!\!\Bigg\}\nonumber\\
        &\leq \max\Bigg\{\!\!\log\left(\max_{j\in[d]} \frac{\exp({g_{w,j}})}{\exp({g'_{w,j}})}\right),\log\left(\max_{j\in[d]} \frac{\exp({g'_{w,j}})}{\exp({g_{w,j}})}\right)\!\!\Bigg\}\label{max_bound_change_g_func_1}\\
        &= \max\Bigg\{\!\!\max_{j\in[d]}\log\left( \frac{\exp({g_{w,j}})}{\exp({g'_{w,j}})}\right),\max_{j\in[d]}\log\left( \frac{\exp({g'_{w,j}})}{\exp({g_{w,j}})}\right)\!\!\Bigg\},\label{max_bound_change_g_func_2}\\
        &=\max_{j\in [d]}\Big|g_{w,j}-g'_{w,j}\Big|,
    \end{align}
    where \eqref{max_bound_change_g_func_1} uses the fact that, for two tuples of positive numbers $(a_1,\cdots,a_d)$ and $(b_1,\cdots,b_d)$, we have
    \begin{align}
        \frac{a_1+\cdots+a_d}{b_1+\cdots+b_d}\leq\max_{j\in[d]}\frac{a_j}{b_j}
    \end{align}

    \item \eqref{simplification_difference_croos_entropy_19} follows from Assumption~\eqref{subsubsec:bounded_output}.
\end{itemize}

Thus, by substituting \eqref{upperbound_negative_loss} and \eqref{simplification_difference_croos_entropy_19} into \eqref{difference_worse_bound_2}, we obtain, for any $t \in [T]$,
\begin{align}
    c_{t} &\leq\frac{(C_1+\rho C_2)}{T}\label{bounding_constant_mc_diarmid}
\end{align}
We can now apply Corollary~\ref{mcdiarmid_inequality_markov_chain}, which implies
\begin{align*}
&\sum_{n=1}^{N}\mathbb{E}_{W}\left[\log\mathbb{E}_{\mathbf{S}}\left[e^{\frac{\lambda}{N}\left(\mathbb{E}_{\mathbf{S}}\left[\ell(\mathbf{z}^{(n)},W)\right]-\left[\ell(\mathbf{s}^{(n)},W)\right]\right)}\right]\right]\nonumber\\
&\quad\leq\frac{\lambda^2\|c\|^2\tau_{\min}}{8N},
\end{align*}

Rewriting \eqref{multi_round_generalization_bound_10}, we obtain
\begin{align}
    &\mathbb{E}_{\mathbf{S},{W}}\left[\gen(\mathbf{S},W)\right]\nonumber\\
    &\qquad\leq\frac{1}{\lambda}\left[I(\mathbf{S};W)+\mathbb{E}_{W}\left[\log\mathbb{E}_{\mathbf{S}}\left[e^{\phi(\mathbf{S},\mathbf{S}',{W})}\right]\right]\right]\nonumber\\
    &\qquad\leq\frac{1}{\lambda}\left[I(\mathbf{S};W)+\frac{{\lambda}^{2}\|c\|^{2}\tau_{\min}}{8 N}\right]\nonumber\\
    &\qquad=\left(2B(1+2\rho)+\log d\right)\sqrt{\frac{\tau_{\min}I(\mathbf{S};W)}{2N T}},\label{substtuting_lambda_k_1}
\end{align}
where \eqref{substtuting_lambda_k_1} follows by choosing
\begin{align}
    \lambda=\sqrt{\frac{8N I(\mathbf{S};W)}{\|c\|^2\tau_{\min}}},
\end{align}
which completes the proof.

\qed








\subsection{Proof of Theorem \ref{Theorem_Multi_round_LNTP}}\label{Proof_Theorem_Multi_round_LNTP}

For ease of notation, let
\begin{align}
&m=\left\lceil 63\left(\frac{2C\rho}{\epsilon_1}\right)^2\log\left(dNT\right)\right\rceil,\:\epsilon_1=\frac{\theta}{2}\label{Paramaters_dimension_multi_round_1}\\
&c_1=c_2=\sqrt{\frac{\epsilon^2_1}{C^2}+1},\:\nu=\frac{1}{2c_1}.\label{Paramaters_dimension_multi_round_3}
\end{align}

We use Theorem~\ref{lossy_inexpectation_bound_federated}, which can be derived in the same manner. To prove this theorem, let $\epsilon$ satisfy \eqref{def:Rdist}. We upper-bound the rate-distortion term $R_{\mathcal{D}}(\epsilon)$, defined in \eqref{general_bound_lossy_mutual_1}.

We define an appropriate choice of $P_{\hat{W}|\mathbf{S}}$ that satisfies condition~\ref{def:Rdist}.

It is straightforward to verify that
\begin{align}
    &\mathbb{E}_{\mathbf{S},{W},\hat{W}}\left[\operatorname{gen}_{\theta}(\mathbf{S},{W}) - \operatorname{gen}_{0}(\mathbf{S},\hat{W})\right]\nonumber\\
    &= \mathbb{E}_{\mathbf{S},\hat{W},W}\left[\mathcal{L}_{-\theta/2}(W) - {\mathcal{L}}_{0}(\hat{W})\right] \nonumber \\
    &\quad- \mathbb{E}_{\mathbf{S},\hat{W},W}\left[\hat{\mathcal{L}}_{\theta/2}(\mathbf{S},W)-\hat{\mathcal{L}}_{0}(\mathbf{S},\hat{W})\right] \nonumber \\
    &=
    \mathbb{E}_{\mathbf{S},\hat{W},W}\left[\frac{1}{T}\sum_{t \in [T]}
        \mathbf{1}\left\{ \log\frac{C_1({W},t,\mathbf{z})}{C_2(W,t,\mathbf{z})} < {-\theta/2} \right\}\right]\\
    &\quad-
    \mathbb{E}_{\mathbf{S},\hat{W},W}\left[\frac{1}{T}\sum_{t \in [T]}
        \mathbf{1}\left\{\log\frac{C_1(\hat{W},t,\mathbf{z})}{C_2(\hat{W},t,\mathbf{z})} < {0} \right\}\right]\\
        &+\mathbb{E}_{\mathbf{S},\hat{W},W}\Bigg[\frac{1}{N}\!\!\!\sum_{n\in[N]}\Bigg[\frac{1}{T}\sum_{t \in [T]}
    \mathbf{1}\left\{ \log\frac{C_{1,n}(\hat{W},t,\mathbf{z}^{(n)})}{C_{2,n}(\hat{W},t,\mathbf{z}^{(n)})}<0 \right\}\\
    &\quad-\mathbb{E}_{\mathbf{S},{W},\hat{W}}\Bigg[\frac{1}{N}\!\!\!\sum_{n\in[N]}\!\!\Bigg[\frac{1}{T}\!\!\!\sum_{t \in [T]}
    \mathbf{1}\left\{ \log\frac{C_{1,n}({{W}},t,\mathbf{z}^{(n)})}{C_{2,n}({W},t,\mathbf{z}^{(n)})}\!<\!\theta/2 \right\}\!\!\Bigg]\Bigg]\\
    &\leq
    \mathbb{E}_{\mathbf{S},{W},\hat{W}}\Bigg[\frac{1}{T}\sum_{t \in [T]}
        \mathbf{1}\Bigg\{\Bigg|\log\frac{C_1({W},t,\mathbf{z})}{C_2({W},t,\mathbf{z})}-\log\frac{C_1(\hat{W},t,\mathbf{z})}{C_2(\hat{W},t,\mathbf{z})}\Bigg|>\frac{\theta}{2}\Bigg\}\nonumber\\
    &\quad+\mathbb{E}_{\mathbf{S},{W},\hat{W}}\Bigg[\frac{1}{N}\!\!\!\!\sum_{n\in[N]}\!\!\Bigg[\frac{1}{T}\sum_{t \in [T]}
    \mathbf{1}\Bigg\{\Bigg|\log\frac{C_{1,n}(\hat{W},t,\mathbf{z}^{(n)})}{C_{2,n}(\hat{W},t,\mathbf{z}^{(n)})}-\log\frac{C_{1,n}({W},t,\mathbf{z}^{(n)})}{C_{2,n}({{W}},t,\mathbf{z}^{(n)})}\Bigg|>\frac{\theta}{2}\Bigg\}\Bigg]\Bigg]\nonumber\\
    &\quad\coloneqq\epsilon_A\label{lossy_pop_emp_difference_1}
\end{align}

Let
\begin{align}
C_{1}(w,t,{\mathbf{z}}) &= p_{w}\left(z_{t}|z_{t-1}, \dots, z_{t-\rho}\right),\label{defition_maximal_ratio_1} \\
C_{2}(w,t,{\mathbf{z}}) &= \max_{z'_{t} \neq z_{t}}
p_{w}\left(z'_{t}|z_{t-1}, \dots, z_{t-\rho}\right).
\label{defition_maximal_ratio_2}
\end{align}

Given the model $w$, we denote the corresponding matrices by $M_j^{w}$, where $j\in[\rho]$. Moreover, let
\begin{align}
    u_{t} \coloneqq & \mathcal{E}(z_{t}), \nonumber \\
    \ell_{t} \coloneqq& \ell, \text{ if } u_{t} = e_{\ell}.
\end{align}
In other words, $\ell_{t}$ denotes the order of $z_{t}$ in the dictionary. We define $\ell'_{t}$ analogously. We also denote the element $(a_1,a_2)\in [d] \times [d]$ of $M_j^{w}$ by $M_j^{w}[a_1,a_2]$.

Then, we can write
\begin{align}
    C_{1}(w,t,{\mathbf{z}}) &= \frac{e^{\sum_{j\in[\rho]} M_j^{w}[\ell_{t-j},\ell_{t}]}}{\sum_{\ell\in[d]}e^{\sum_{j\in[\rho]} M_j^{w}[\ell_{t-j},\ell]}},\nonumber\\
    C_{2}(w,t,{\mathbf{z}}) &= \frac{\max_{\ell' \neq \ell }e^{\sum_{j\in[\rho]} M_j^{w}[\ell_{t-j},\ell']}}{\sum_{\ell\in[d]}e^{\sum_{j\in[\rho]} M_j^{w}[\ell_{t-j},\ell]}}. \nonumber
\end{align}
Hence,
\begin{align}
    \log \frac{ C_{1}(w,t,{\mathbf{z}})}{ C_{2}(w,t,{\mathbf{z}})} \coloneqq \sum\limits_{j\in[\rho]} M_j^{w}[\ell_{t-j},\ell]- \max_{\ell' \neq \ell}
\sum_{j\in[\rho]}M_j^{w}[\ell_{t-j},\ell'].
\end{align}

To express this term as an inner product, for a given model $w$, define $V(w,r)\in \mathbb{R}^{\rho d}$ and $V(z_{t-1:t-\rho})\in \mathbb{R}^{\rho d}$, respectively, as
\begin{align}
    &V(w,r) \triangleq  \Big(\! M_1^{w}[1,r],\cdots,M_1^{w}[d,r],\ldots,M_{\rho}^{w}[1,r],\ldots,M_{\rho}^{w}[d,r]\!\Big),\nonumber \\
    &V(z_{t-1:t-\rho}) \triangleq  \Big(u_{t-1},\ldots,u_{t-\rho}\Big).\label{Auxiliary_vector_definition_LNTP}
\end{align}
With these definitions, we have
\begin{align}
    \log \frac{ C_{1}(w,t,{\mathbf{z}})}{ C_{2}(w,t,{\mathbf{z}})} &= \left\langle   V(w,\ell) , V(z_{t-1:t-\rho}) \right\rangle\nonumber\\
    &-\max_{\ell'\neq\ell} \left\langle   V(w,\ell') , V(z_{t-1:t-\rho}) \right\rangle.\nonumber
\end{align}

Thus,
\begin{align}
D_{t}&\coloneqq\left|\log\frac{C_{1}({w},t,\mathbf{z})}{C_{2}({w},t,\mathbf{z})}-\log\frac{C_1(\hat{w},t,{\mathbf{z}})}{C_2(\hat{w},t,\mathbf{z})}\right|\\
&\leq \Bigg| \left\langle   V({w},\ell) - V(\hat{W},\ell)  , V(z_{t-1:t-\rho}) \right\rangle \Bigg|\nonumber\\
&\quad +\Bigg| \max_{\ell' \neq \ell} \left\langle   V({w},\ell') , V(z_{t-1:t-\rho}) \right\rangle- \max_{\ell' \neq \ell} \left\langle   V(\hat{W},\ell') , V(z_{t-1:t-\rho}) \right\rangle \Bigg| \nonumber\\
&\leq\left| \left\langle   V({w},\ell) - V(\hat{W},\ell)  , V(z_{t-1:t-\rho}) \right\rangle \right|\nonumber\\
&\quad +\max_{\ell' \neq \ell}\Bigg|\left\langle   V({w},\ell') , V(z_{t-1:t-\rho}) \right\rangle- \left\langle   V(\hat{W},\ell') , V(z_{t-1:t-\rho}) \right\rangle\Bigg| \nonumber\\
&=\left| \left\langle   V({w},\ell) - V(\hat{W},\ell)  , V(z_{t-1:t-\rho}) \right\rangle \right| +\max_{\ell' \neq \ell}\left|  \left\langle   V({w},\ell')-V(\hat{W},\ell') , V(z_{t-1:t-\rho}) \right\rangle \right| \nonumber\\
&=\left|\left\langle V({w}-\hat{W},\ell), V(z_{t-1:t-\rho}) \right\rangle \right| +\max_{\ell' \neq \ell}\left|  \left\langle   V({w}-\hat{W},\ell'), V(z_{t-1:t-\rho}) \right\rangle \right| \nonumber\\
&\leq 2\max_{\ell}\left|  \left\langle   V({w}-\hat{W},\ell_{i,t}), V(z_{t-1:t-\rho}) \right\rangle \right|.
\end{align}

Note that $V(\mathsf{A} w,\ell)=\mathsf{A}V(w,\ell)$ and $\sum_{\ell\in[d]}\left\|V\left( w,\ell\right)\right\|^2=\|w\|^2$, where $\mathsf{A}\in \mathbb{R}^{m \times \rho d}$ has elements distributed in an i.i.d.\ manner according to $\mathcal{N}(0,1/m)$. For any $\ell\in[d]$, denote by $\mathcal{J}_{\ell}$ the event that $\left\|  \mathsf{A} \, V\left( w,\ell\right)\right\| > c_1  \left\|V\left( w,\ell\right)\right\|$, where $c_1 \in \mathbb{R}_+$, i.e.,
\begin{align}
    \mathcal{J}_{\ell} \coloneqq \mathrm{1}\left\{\left\| \, \mathsf{A}V\left(w,\ell\right)\right\| > c_1 \left\|V\left(w,\ell\right)\right\|\right\}.
\end{align}

To define $\hat{W}$, we define $V(\hat{W},\ell)$ for every $\ell\in[d]$. Fix some $\epsilon_1 \geq 0$ and denote
\begin{align}
    [d]_{w} = \left\{\ell\in[d]\colon \|V(w,\ell)\|\geq\frac{\epsilon_1}{2\rho}\right\}.
\end{align}

First, let
\begin{align}
    U_{\ell} = \begin{cases}
        V\left(\mathsf{A}w,\ell\right), & \text{ if }  \left\|  \mathsf{A} \, V\left( w,\ell\right)\right\| \leq c_1 \left\|V\left(w,\ell\right)\right\|, \\
        \mathbf{0}_m, & \text{ otherwise.}
    \end{cases}
\end{align}
Then, define
\begin{align}
    \hat{W}_{\ell} = \begin{cases}
         \mathbf{0}_{m}&, \text{ if } \ell \notin [d]_{w}, \\
         U_{\ell} + N_{\ell},& \text{ otherwise.}
     \end{cases},
\end{align}
where $N_{\ell} \sim \text{Uniform}\left(\mathcal{B}_m(0,\nu)\right)$. Finally, let
\begin{align}
     V(\hat{W},\ell) &= \mathsf{A}^\top \hat{W}_{\ell},
\end{align}
Note that, for all $\ell$ such that $\ell \notin [d]_{w}$, we have
\begin{align}
    \mathbb{P}\left(\left|  \left\langle   V(w-\hat{W},\ell), V(z_{t-1:t-\rho}) \right\rangle  \right| \geq \epsilon_1/2 \right)&=\mathbb{P}\left(\left|  \left\langle   V({w},\ell), V(z_{t-1:t-\rho}) \right\rangle  \right| \geq \epsilon_1/2\right)\nonumber\\
    &\quad=0.
\end{align}

Hence,
\begin{align}
   \mathbb{P}\left(D_{t}\geq\epsilon_1\right)&\leq\mathbb{P}\left(2\max_{\ell}\left|  \left\langle   V({w}-\hat{W},\ell), V(z_{t-1:t-\rho}) \right\rangle \right|\geq\epsilon_1\right)\nonumber\\
   &\quad=\mathbb{P}\left(\max_{\ell}\left|\left\langle V({W}-\hat{W},\ell), V(z_{t-1:t-\rho}) \right\rangle \right|\geq\frac{\epsilon_1}{2}\right)\nonumber\\
   &\quad\leq\sum_{\ell\in[d]}\mathbb{P}\left(\left|  \left\langle   V({w}-\hat{{W}},\ell), V(z_{t-1:t-\rho}) \right\rangle \right|\geq\frac{\epsilon_1}{2}\right)\nonumber\\
   &\quad=\sum_{\ell\in[d]_{w}}\mathbb{P}\left(\left|  \left\langle   V({w}-\hat{W},\ell), V(z_{t-1:t-\rho}) \right\rangle \right|\geq\frac{\epsilon_1}{2}\right)\nonumber\\
   &\quad\leq\sum_{\ell\in[d]_{w}}\mathbb{P}\left(\left|\left\langle V\left({w}-\hat{W},\ell\right),V\left(z_{t-1:t-\rho}\right)\right\rangle\right|\geq\frac{\epsilon_1}{2}, \mathcal{J}^c_{\ell}\right)\nonumber\\
   &\qquad+\sum_{\ell\in[d]_{w}}\mathbb{P}\left(\left|\left\langle V\left({w}-\hat{W},\ell\right),V\left(z_{t-1:t-\rho}\right)\right\rangle\right|\geq\frac{\epsilon_1}{2}, \mathcal{J}_{\ell}\right)\nonumber\\
   &\quad\leq \sum_{\ell\in[d]_{w}}\mathbb{P}\left(\left|\left\langle V\left({w}-\hat{W},\ell\right),V\left(z_{t-1:t-\rho}\right)\right\rangle\right|\geq\frac{\epsilon_1}{2}, \mathcal{J}^c_{\ell}\right)+\sum_{\ell\in[d]_{w}}\mathbb{P}\left(\mathcal{J}_{\ell}\right)\nonumber\\
   &\quad\leq\sum_{\ell\in[d]_{w}}\mathbb{P}\left(\left|\left\langle V\left({w},\ell\right)-\mathsf{A}^\top \mathsf{A}V\left({w},\ell\right),V\left(z_{t-1:t-\rho}\right)\right\rangle\right|\geq\frac{\epsilon_1}{2}, \mathcal{J}^c_{\ell}\right)\nonumber\\
   &\qquad+\sum_{\ell\in[d]_{w}}\mathbb{P}\left(\left|\left\langle \mathsf{A}^\top N_{\ell} ,V\left(z_{t-1:t-\rho}\right)\right\rangle\right|\geq\frac{\epsilon_1}{2}\right)+\sum_{\ell\in[d]_{w}}\mathbb{P}\left(\mathcal{J}_{\ell}\right)\nonumber\\
   &\leq \sum_{\ell\in[d]_{w}}\mathbb{P}\left(\left|\left\langle V\left({w},\ell\right)-\mathsf{A}^\top \mathsf{A}V\left({w},\ell\right),V\left(z_{t-1:t-\rho}\right)\right\rangle\right|\geq\frac{\epsilon_1}{2}\right)\nonumber\\
   &\qquad+\sum_{\ell\in[d]_{w}}\mathbb{P}\left(\left|\left\langle \mathsf{A}^\top N_{\ell} ,V\left(z_{t-1:t-\rho}\right)\right\rangle\right|\geq\frac{\epsilon_1}{2}\right)+\sum_{\ell\in[d]_{w}}\mathbb{P}\left(\mathcal{J}_{\ell}\right)\nonumber\\
   &\quad\leq\!\!\!\!\sum_{\ell\in[d]_{w}}\mathbb{P}\left(\left|\left\langle V\left({w},\ell\right)-\mathsf{A}^\top \mathsf{A}V\left({w},\ell\right),V\left(z_{t-1:t-\rho}\right)\right\rangle\right|\geq\frac{\epsilon_1}{2}\right)\nonumber\\
   &\qquad+\!\!\!\!\sum_{\ell\in[d]_{w}}\!\!\mathbb{P}\left(\left|\left\langle  N_{\ell} ,\mathsf{A}V\left(z_{t-1:t-\rho}\right)\right\rangle\right|\geq\frac{\epsilon_1}{2}\!\Big|\textit{Cond}\right)\nonumber\\
   &\qquad +\sum_{\ell
   \in[d]_{w}}\mathbb{P}\left(\left\|  \mathsf{A}V\!\!\left(z_{t-1:t-\rho}\right)\right\|\!\geq c_2 \!\left\|V\!\!\left(z_{t-1:t-\rho}\right)\right\|\right) +\sum_{\ell\in[d]_{w}}\mathbb{P}\left(\mathcal{J}_{\ell}\right)\nonumber\\
&\quad=\circled{I}+\circled{II}+\circled{III}+\circled{IV}\nonumber
\end{align}
where $\textit{Cond}$ denotes $\!\left\|  \mathsf{A}V\left(z_{t-1:t-\rho}\right)\right\|\leq c_2 \left\|V\left(z_{t-1:t-\rho}\right)\right\|$.

\begin{itemize}

    \item \textbf{Bounding \circled{I}:}

\begin{align}
        \circled{I}&\leq 4\sum_{\ell\in[d]_w} e^{-\frac{m}{7}\bigg(\!\frac{\epsilon_1}{2\rho\left\|V\left(w,\ell\right)\right\|}\!\bigg)^2}\label{Bounding_term_I}\\
        &\leq 4d\times e^{-\frac{m}{7}\bigg(\!\frac{\epsilon_1}{2\rho\left\|V\left(w,\ell\right)\right\|}\!\bigg)^2}\leq\mathcal{O}\left(\frac{1}{NT}\right) \nonumber
\end{align}
where inequality~\eqref{Bounding_term_I} follows from \cite[Lemma 8, Part 2]{gronlund2020near}.

    \item \textbf{Bounding \circled{II}:}

\begin{align}
    \circled{II}&\leq\!\!\sum_{\ell\in[d]_w}\!\!\frac{2m\nu^m}{\sqrt{\pi}}e^{-\frac{\left(m+1\right)}{2}\left(\frac{\epsilon_1}{2c_1\nu C}\right)^2}\leq\mathcal{O}\left(\frac{1}{NT}\right)\nonumber
\end{align}

    \item \textbf{Bounding \circled{III}-\circled{IV}:}

\begin{align}
    &\circled{III}\leq \!\!\sum_{\ell\in[d]_w}\!\!e^{-m\left(c^2_2-1-\log(c_2)\right)}\leq\mathcal{O}\left(\frac{1}{NT}\right),\nonumber\\
    &\circled{IV}\leq \!\!\sum_{\ell\in[d]_w}\!\!e^{-m\left(c^2_1-1-\log(c_1)\right)}\leq\mathcal{O}\left(\frac{1}{NT}\right),\nonumber
\end{align}
where the preceding two relations follow from \cite[Lemma~8, Part~1]{gronlund2020near}.

\end{itemize}

In the final step, we compute an upper bound on the mutual information term. For this purpose, let
\begin{align}
    \mathbf{T}_{[d]}=\left({T}_{1},\cdots,{T}_{d}\right)\in\{0,1\}^{d},
\end{align}
where ${T}_{\ell}=1$ if $\ell\in[d]_{w}$, and ${T}_{\ell}=0$ otherwise.
For simplicity and with a slight abuse of notation, denote
\begin{align}
    [d]_{\mathbf{T}_{[d]}} =\left\{\ell\in[d]\colon T_{\ell} = 1\right\}.\nonumber
\end{align}
With this definition, we have
\begin{align}
&\mathbb{I}^A\left(\mathbf{S};\hat{W}\right)\nonumber\\
&\quad\leq \mathbb{I}^{\mathsf{A}}\left({W};\hat{W}\right)\label{Bound_MI_Binary_1}\\
    &\quad= \mathbb{I}^{\mathsf{A}}\left(\left\{V(W,\ell)\right\}_{\ell\in[d]};\left\{V(\hat{W},\ell)\right\}_{\ell\in[d]}\right)\nonumber\\
    &\quad\leq \mathbb{I}^{\mathsf{A}}\left(\left\{V(W,\ell)\right\}_{\ell\in[d]};\left\{\hat{W}_{\ell}\right\}_{\ell\in[d]}\right)\label{Bound_MI_Binary_2}\\
    &\quad\leq \mathbb{I}^{\mathsf{A}}\left(\left\{V(W,\ell)\right\}_{\ell\in[d]};\left\{\hat{W}_{\ell}\right\}_{\ell\in[d]},\mathbf{T}_{[d]}\right)\nonumber\\
    &\quad=\mathbb{I}^{\mathsf{A}}\left(\left\{V(W,\ell)\right\}_{\ell\in[d]};\left\{\hat{W}_{\ell}\right\}_{\ell\in[d]}\Big|\mathbf{T}_{[d]}\right)+\mathbb{I}^{\mathsf{A}}\left(\left\{V(W,\ell)\right\}_{\ell\in[d]};\mathbf{T}_{[d]}\right)\nonumber\\
    &\quad\leq \mathbb{I}^{\mathsf{A}}\left(\left\{V(W,\ell)\right\}_{\ell\in[d]};\left\{\hat{W}_{\ell}\right\}_{\ell\in[d]}\Big|\mathbf{T}_{[d]}\right)+H\left(\mathbf{T}_{[d]}\right)\nonumber\\
    &\quad=\mathbb{E}_{\mathbf{T}_{[d]}}\left[\mathbb{I}^{\mathsf{A},\mathbf{T}_{[d]}}\left(\!\!\left\{V(W,\ell)\right\}_{\ell\in[d]};\left\{\hat{W}_{\ell}\right\}_{\ell\in[d]}\right)\!\right]\!+\! H\left(\mathbf{T}_{[d]}\right)\nonumber\\
    &\quad=\mathbb{E}_{\mathbf{T}_{[d]}}\Bigg[h^{\mathsf{A},\mathbf{T}_{[d]}}\left(\left\{\hat{W}_{\ell}\right\}_{\ell\in[d]_{\mathbf{T}_{[d]}}}\right)-h^{\mathsf{A},\mathbf{T}_{[d]}}\left(\left\{\hat{W}_{\ell}\right\}_{\ell\in[d]_{\mathbf{T}_{[d]}}}\!\Big|\!\left\{V(W,\ell)\right\}_{\ell\in[d]}\right)\!\!\Bigg]\!\!+H\!\!\left(\mathbf{T}_{[d]}\right)\nonumber\\
    &\quad\leq\mathbb{E}_{\mathbf{T}_{[d]}}\Bigg[\sum_{\ell\in[d]_{\mathbf{T}_{[d]}}} h^{\mathsf{A},\mathbf{T}_{[d]}}\left(\left\{\hat{W}_{\ell}\right\}\right)-h^{\mathsf{A},\mathbf{T}_{[d]}}\left(\!\!\left\{\hat{W}_{\ell}\right\}_{\ell\in[d]_{\mathbf{T}_{[d]}}}\Big|\left\{V(W,\ell)\right\}_{\ell\in[d]}\right)\!\!\Bigg]\!\!+\!\!H\left(\mathbf{T}_{[d]}\right)\nonumber\\
    &\quad=\mathbb{E}_{\mathbf{T}_{[d]}}\!\!\left[\sum_{\ell\in[d]_{\mathbf{T}_{[d]}}}\!\!\!\!\!\! \left(h^{\mathsf{A},\mathbf{T}_{[d]}}\left(\hat{W}_{\ell}\right)-h^{\mathsf{A},\mathbf{T}_{[d]}}\left(\hat{W}_{\ell}\Big|V(W,\ell)\right)\right)\!\right]+H\left(\mathbf{T}_{[d]}\right)\nonumber\\
    &\quad\leq \mathbb{E}_{\mathbf{T}_{[d]}}\left[\sum_{\ell\in[d]} T_{\ell}\right]\log\left(\text{Volume}\left(\mathcal{B}_m(c_1B_1+\nu)\right)\right)-\mathbb{E}_{\mathbf{T}_{[d]}}\left[\sum_{\ell\in\mathbf{T}_{[d]}}h^{A,\mathbf{T}_{[d]}}\left(\hat{W}_{\ell}\Big|V(W,\ell)\right)\right]+H\left(\mathbf{T}_{[d]}\right)\nonumber\\
    &\quad=\mathbb{E}_{\mathbf{T}_{[d]}}\left[\sum_{\ell\in[d]} T_{\ell}\right]\log\left(\text{Volume}\left(\mathcal{B}_m(c_1B_1+\nu)\right)\right)-\mathbb{E}_{\mathbf{T}_{[d]}}\left[\sum_{\ell\in[d]} T_{\ell}\right]\log\left(\text{Volume}\left(\mathcal{B}_m(\nu)\right)\right)+H\left(\mathbf{T}_{[d]}\right)\nonumber\\
    &\quad=m\mathbb{E}_{\mathbf{T}_{[d]}}\left[\sum_{\ell\in[d]} T_{\ell}\right]\log\left(\frac{c_1B+\nu}{\nu}\right)+H\left(\mathbf{T}_{[d]}\right)\label{Bound_MI_Binary_3}
\end{align}
where equations~\eqref{Bound_MI_Binary_1} and \eqref{Bound_MI_Binary_2} use the data-processing inequality.

It remains to upper-bound $H\left(\mathbf{T}_{[d]}\right)$. Denote $N =\min\left(d, \left(\frac{4C\rho}{\theta}\right)^2\right)$. The vector $\mathbf{T}_{[d]}$ can have at most $N$ entries equal to one. Hence,

\begin{itemize}

\item If $N<d/2$, the number of possible choices for
$\mathbf{T}_{[d]}$ is at most
\begin{align}
\left|\mathbf{T}_{[d]}\right|
&\triangleq
\binom{d}{0}+\binom{d}{1}+\cdots+\binom{d}{N}
\nonumber\\
&\leq
\exp\left(d\,h_b(N/d)\right).
\nonumber
\end{align}

Since
\begin{align}
N
=
\min\left\{
d,
\left(\frac{4C\rho}{\theta}\right)^2
\right\},
\nonumber
\end{align}
and $N<d/2$, we have
\begin{align}
N
=
\left(\frac{4C\rho}{\theta}\right)^2.
\nonumber
\end{align}
Hence,
\begin{align}
H\left(\mathbf{T}_{[d]}\right)
&\leq
\log\left(\left|\mathbf{T}_{[d]}\right|\right)
\nonumber\\
&\leq
d\,
h_b\left(
\frac{1}{d}
\left(\frac{4C\rho}{\theta}\right)^2
\right)
\nonumber\\
&=
\left(\frac{4C\rho}{\theta}\right)^2
\left[
\frac{d}{
\left(\frac{4C\rho}{\theta}\right)^2}
h_b\left(
\frac{
\left(\frac{4C\rho}{\theta}\right)^2
}{d}
\right)
\right]
\nonumber\\
&\leq
\left(\frac{4C\rho}{\theta}\right)^2
\left[
1+
\log\left(
\frac{d}{
\left(\frac{4C\rho}{\theta}\right)^2}
\right)
\right].
\nonumber
\end{align}

where we used
\begin{align}
x h_b(1/x)
&=
\log(x)
-
x\left(1-\frac{1}{x}\right)
\log\left(1-\frac{1}{x}\right)
\nonumber\\
&\leq
\log(x)
-
x\left(1-\frac{1}{x}\right)
\left(
1-\frac{1}{1-\frac{1}{x}}
\right)
\nonumber\\
&\leq
\log(x)+1,
\nonumber
\end{align}
by exploiting the inequality
$\log(u)\geq 1-1/u$.

\item If $N\geq d/2$, then
\begin{align}
H\left(\mathbf{T}_{[d]}\right)
\leq d
\leq
2N
\leq
2\left(\frac{4C\rho}{\theta}\right)^2.
\nonumber
\end{align}

Therefore, by recalling equation~\eqref{Bound_MI_Binary_3}, we have
\begin{align}
\ref{Bound_MI_Binary_3}
&\leq
m
\min\left\{
d,
\left(\frac{4C\rho}{\theta}\right)^2
\right\}
\log\left(
\frac{c_1C+\nu}{\nu}
\right)
\nonumber\\
&\quad+
\mathcal{O}\left(
\left(\frac{4C\rho}{\theta}\right)^2
\left[
1+
\log\left(
\frac{d}{
\left(\frac{4C\rho}{\theta}\right)^2}
\right)
\right]
\right).
\nonumber
\end{align}

This completes the proof.

\end{itemize}








 \subsection{Proof of Theorem \ref{Theorem_Multi_round_SANTP}}\label{Proof_Theorem_Multi_round_SANTP}

The proof of Theorem \ref{Theorem_Multi_round_SANTP} follows the same lossy approach as the proof of Theorem \ref{Theorem_Multi_round_LNTP}.

We begin by defining the parameters
\begin{align}
m_3&=\left\lceil 448\left(\frac{C}{\theta}\right)^2\log\left(d NT\right)\right\rceil
\label{Paramaters_dimension_multi_SANTP_round_1}\\
m_4&=\left\lceil 448\left(\frac{C^2\sqrt{\rho}}{\theta}\right)^2
\log\left(\rho d NT\right)\right\rceil,\nonumber\\
c_{4,1}=c_{3,1}
&=\sqrt{\frac{\theta^2}{C^2}+1},
\quad
\nu_{4}=\nu_{3}=\frac{1}{2c_{3,1}},
\nonumber\\
c_{4,2}=c_{3,2}
&=\sqrt{\frac{\theta^2}{C^2}+1}.
\label{Paramaters_dimension_multi_SANTP_round_3}
\end{align}

These parameters will be used throughout the remainder of the proof.

We use the extension of Theorem~\ref{lossy_inexpectation_bound_federated} to derive an upper bound on the rate--distortion term $R_{\mathcal{D}}(\cdot)$. Since the attention mechanism depends on the product of the transpose of the Key matrix and the Query matrix, i.e.,
$M^{\top}_{2}\!\!\cdot M_{1}$, as well as the matrix $M_{3}$, we introduce the matrix $M_{4}=M^{\top}_{2} M_{1}$.
Accordingly, we consider the equivalent parameterization $W=\left(M_{3},M_{4}\right)$
for defining the attention problem. In this setting, the agent's model is defined as
$W=\left(M_{3}, M_{4}\right)$.

Therefore, we can rewrite $g_{M_{[2]}}(\mathbf{z})$ in Theorem~\ref{Theorem_Multi_round_SANTP} as follows:
\begin{align}
    g_{M_{4}}(\mathbf{z})\coloneqq g_{M_{[2]}}(\mathbf{z})
    &\coloneqq\sigma\left(\frac{1}{\sqrt{\rho}}\big\langle M_{1}
    \mathcal{E}(z_{t-1}),M_{2}\mathcal{E}(z_{t-j})\big\rangle\right)_{j=1}^{\rho}\nonumber\\
    &=\sigma\left(\frac{1}{\sqrt{\rho}}\big\langle M^{\top}_{2}M_{1}
    \mathcal{E}(z_{t-1}),\mathcal{E}(z_{t-j})\big\rangle\right)_{j=1}^{\rho}\nonumber\\
    &=\sigma\left(\frac{1}{\sqrt{\rho}}\big\langle M_{4}
    \mathcal{E}(z_{t-1}),\mathcal{E}(z_{t-j})\big\rangle\right)_{j=1}^{\rho}\nonumber
\end{align}

Let $C_1\left({w},t,{\mathbf{z}}\right)$ and $C_2\left({w},t,\mathbf{z}\right)$ be defined as in \eqref{defition_maximal_ratio_1} and \eqref{defition_maximal_ratio_2}, respectively. These definitions imply that, given the model $w$, we have
\begin{align}
    &p_{w}\left(z_{t}\Big|z_{t-1},\cdots,z_{t-\rho}\right)=\frac{\exp\left[\left\langle \mathcal{E}\left(z_{t}\right),M_{3}\sum_{j=1}^{\rho}\left(g_{M_4}\left(\mathbf{z}\right)\right)_j\cdot\mathcal{E}\left(z_{t-j}\right)\right\rangle\right]}{\sum_{\overline{z}_{t}}\exp\left[\left\langle \mathcal{E}\left(\overline{z}_{t}\right),M_{3}\sum_{j=1}^{\rho}\left(g_{M_{4}}\left(\mathbf{z}\right)\right)_j\cdot\mathcal{E}\left(z_{t-j}\right)\right\rangle\right]},\nonumber
\end{align}
and
\begin{align}
    C_{1}(w,t,{\mathbf{z}})
    &=\frac{\exp\left[\left\langle \mathcal{E}\left(z_{t}\right),M_{3}\sum_{j=1}^{\rho}\left(g_{M_{4}}\left(\mathbf{z}\right)\right)_j\cdot\mathcal{E}\left(z_{t-j}\right)\right\rangle\right]}{\sum_{\overline{z}_{t}}\exp\left[\left\langle \mathcal{E}\left(\overline{z}_{t}\right),M_{3}\sum_{j=1}^{\rho}\left(g_{M_{4}}\left(\mathbf{z}\right)\right)_j\cdot\mathcal{E}\left(z_{t-j}\right)\right\rangle\right]}\nonumber\\
    C_2(w,t,{\mathbf{z}})
    &=\frac{\max_{z'_{t}\neq z_{t}}\exp\left[\left\langle \mathcal{E}\left(z'_{t}\right),M_{3}\sum_{j=1}^{\rho}\left(g_{M_{4}}\left(\mathbf{z}\right)\right)_j\cdot\mathcal{E}\left(z_{t-j}\right)\right\rangle\right]}{\sum_{\overline{z}_{t}}\exp\left[\left\langle \mathcal{E}\left(\overline{z}_{t}\right),M_{3}\sum_{j=1}^{\rho}\left(g_{M_{4}}\left(\mathbf{z}\right)\right)_j\cdot\mathcal{E}\left(z_{t-j}\right)\right\rangle\right]}\nonumber
\end{align}
which implies that
\begin{align}
    \log\frac{C_1\left(w,t,{\mathbf{z}}\right)}{C_2\left(w,t,{\mathbf{z}}\right)}&=\left[\left\langle \mathcal{E}\left(z_{t}\right),M_{3}\sum_{j=1}^{\rho}\left(g_{M_{4}}\left(\mathbf{z}\right)\right)_j\cdot\mathcal{E}\left(z_{t-j}\right)\right\rangle\right]\nonumber\\
    &\quad-\max_{z'_{t}\neq z_{t}}\left[\left\langle \mathcal{E}\left(z'_{t}\right),M_{3}\sum_{j=1}^{\rho}\left(g_{M_{4}}\left(\mathbf{z}\right)\right)_j\cdot\mathcal{E}\left(z_{t-j}\right)\right\rangle\right]\nonumber
\end{align}

For each model $i\in\{3,4\}$, we additionally define the variables $\hat{M}_{i}$, with $\mathsf{A}_i$ independent for each $i\in\{3,4\}$.

\begin{align}
&\hat{W}_{1}=\left({M}_{3},\hat{M}_{4}\right),\nonumber\\
&\hat{{W}}=\left(\hat{M}_{3},\hat{M}_{4}\right).\nonumber
\end{align}

Using the same approach as in \eqref{lossy_pop_emp_difference_1}, we have
\begin{align}
&\mathbb{E}_{\mathbf{S},{W},\hat{W}}\left[\operatorname{gen}_{\theta}(\mathbf{S},{W}) - \operatorname{gen}_{0}(\mathbf{S},\hat{W})\right]\nonumber\\
&\leq
    \mathbb{E}_{\mathbf{S},{W},\hat{W}}\Bigg[\frac{1}{T}\sum_{t \in [T]}
        \mathbf{1}\Bigg\{\Bigg|\log\frac{C_1({W},t,\mathbf{z})}{C_2({W},t,\mathbf{z})}-\log\frac{C_1(\hat{W},t,\mathbf{z})}{C_2(\hat{W},t,\mathbf{z})}\Bigg|>\frac{\theta}{2}\Bigg\}\nonumber\\
    &\qquad+\mathbb{E}_{\mathbf{S},\hat{W},W}\Bigg[\frac{1}{N}\!\!\!\sum_{n\in[N]}\!\Bigg[\frac{1}{T}\!\!\!\sum_{t \in [T]}
    \mathbf{1}\Bigg\{\Bigg|\log\frac{C_{1,n}(\hat{W},t,\mathbf{z}^{(n)})}{C_{2,n}(\hat{W},t,\mathbf{z}^{(n)})}-\log\frac{C_{1,n}(W,t,\mathbf{z}^{(n)})}{C_{2,n}({W},t,\mathbf{z}^{(n)})}\Bigg|>\frac{\theta}{2}\Bigg\}\nonumber\\
    &\leq\epsilon_{\mathsf{A}_{[3]},1}+\epsilon_{\mathsf{A}_{[3]},2}\nonumber\\
    &=\epsilon_{\mathsf{A}_{[3]}}\nonumber,
\end{align}

Applying the triangle inequality implies that
\begin{align}
    &\epsilon_{\mathsf{A}_{[3]},1}\leq\mathbb{E}_{\mathbf{S},\hat{W},W}\Bigg[\frac{1}{T}\sum_{t \in [T]}
        \mathbf{1}\Bigg\{\Bigg|\log\frac{C_1(W,t,\mathbf{z})}{C_2({W},t,\mathbf{z})}-\log\frac{C_1(\hat{W}_{1},t,\mathbf{z})}{C_2(\hat{W}_{1},t,\mathbf{z})}\Bigg|>\frac{\theta}{4}\Bigg\}\nonumber\\
        &\quad\qquad+\mathbb{E}_{\mathbf{S},\hat{W},W}\Bigg[\frac{1}{T}\sum_{t \in [T]}
        \mathbf{1}\Bigg\{\Bigg|\log\frac{C_1(\hat{W}_{1},t,\mathbf{z})}{C_2(\hat{W}_{1},t,\mathbf{z})}-\log\frac{C_1(\hat{W},t,\mathbf{z})}{C_2(\hat{W},t,\mathbf{z})}\Bigg|>\frac{\theta}{4}\Bigg\}\nonumber\\
        &\qquad=\circled{I}^1+\circled{II}^1\label{Definition_summarizing_losses_1}
\end{align}
and
\begin{align}
\epsilon_{\mathsf{A}_{[3]},2}&\leq\mathbb{E}_{\mathbf{S},\hat{W},W}\Bigg[\frac{1}{N}\sum_{n\in[N]}\Bigg[\frac{1}{T}\sum_{t \in [T]}
    \mathbf{1}\Bigg\{\Bigg|\log\frac{C_{1,n}(W,t,\mathbf{z}^{(n)})}{C_{2,n}({W},t,\mathbf{z}^{(n)})}-\log\frac{C_{1,n}(\hat{W}_{1},t,\mathbf{z}^{(n)})}{C_{2,n}(\hat{W}_{1},t,\mathbf{z}^{(n)})}\Bigg|>\frac{\theta}{4}\Bigg\}\nonumber\\
    &\quad+\mathbb{E}_{\mathbf{S},\hat{W},W}\Bigg[\frac{1}{N}\sum_{n\in[N]}\Bigg[\frac{1}{T}\sum_{t \in [T]}
    \mathbf{1}\Bigg\{\Bigg|\log\frac{C_{1,n}(\hat{W}_{1},t,\mathbf{z}^{(n)})}{C_{2,n}(\hat{W}_{1},t,\mathbf{z}^{(n)})}-\log\frac{C_{1,n}(\hat{W},t,\mathbf{z}^{(n)})}{C_{2,n}(\hat{W},t,\mathbf{z}^{(n)})}\Bigg|>\frac{\theta}{4}\Bigg\}\nonumber\\
    &=\circled{I}^2+\circled{II}^2\label{Definition_summarizing_losses_2}
\end{align}

We then show that
\begin{align}
    \mathbb{E}_{\mathsf{A}_{[3]}}\left[\epsilon_{\mathsf{A}_{[3]}}\right]\leq\mathcal{O}\left(\frac{1}{NT}\right).\nonumber
\end{align}

Similarly to \eqref{Auxiliary_vector_definition_LNTP}, for any
$W=\left(M_{3},M_{4}\right)$ and $i\in\{3,4\}$, we define
\begin{align}
    V\left[\ell,M_{i}\right]&\coloneqq\left(M_{i}[\ell,1],\cdots,M_{i}[\ell,d]\right)\in\mathbb{R}^{d}\nonumber\\
    V\left[M_{i},\ell\right]&\coloneqq\left(M_{i}[1,\ell],\cdots,M_{i}[d,\ell]\right)\in\mathbb{R}^{d}\label{Denote_row_column_def}
\end{align}

\begin{itemize}

\item\textbf{Bounding $\circled{I}^1$ and $\circled{I}^2$:} We have
\begin{align}
    \circled{I}^1&=\mathbb{E}_{\mathbf{S},\hat{W},W}\Bigg[\frac{1}{T}\sum_{t \in [T]}
        \mathbf{1}\Bigg\{\Bigg|\log\frac{C_1({W},t,\mathbf{z})}{C_2({W},t,\mathbf{z})}-\log\frac{C_1(\hat{W}_{1},t,\mathbf{z})}{C_2(\hat{W}_{1},t,\mathbf{z})}\Bigg|>\frac{\theta}{4}\Bigg\}\nonumber
\end{align}
which implies that
\begin{align}
&D^{(1)}_{t-1:t-\rho}\coloneqq\left|\log\frac{C_1({W},t,\mathbf{z})}{C_2({W},t,\mathbf{z})}-\log\frac{C_1(\hat{W}_{1},t,\mathbf{z})}{C_2(\hat{W}_{1},t,\mathbf{z})}\right|\nonumber\\
        &\leq\Bigg|\left\langle \mathcal{E}\left(z_{t}\right),\sum_{j=1}^{\rho}\left(g_{M_{4}}\left(\mathbf{z}\right)\right)_j\cdot{M}_{3}\cdot\mathcal{E}\left(z_{t-j}\right)\right\rangle -\max_{z'_{t}\neq z_{t}}\left[\left\langle \mathcal{E}\left(z'_{t}\right),\sum_{j=1}^{\rho}\left(g_{M_{4}}\left(\mathbf{z}\right)\right)_j\cdot{M}_{3}\cdot\mathcal{E}\left(z_{t-j}\right)\right\rangle\right]\nonumber\\
        &\qquad-\left\langle \mathcal{E}\left(z_{t}\right),\sum_{j=1}^{\rho}\left(g_{\hat{M}_{4}}\left(\mathbf{z}\right)\right)_j\!\!\!\cdot{M}_{3}\cdot\mathcal{E}\left(z_{t-j}\right)\right\rangle +\max_{z'_{t}\neq z_{t}}\left[\left\langle \mathcal{E}\left(z'_{t}\right),\sum_{j=1}^{\rho}\left(g_{\hat{M}_{4}}\left(\mathbf{z}\right)\right)_j\!\!\!\cdot{M}_{3}\cdot\mathcal{E}\left(z_{t-j}\right)\right\rangle\right]\Bigg|\nonumber\\
        &\leq\Bigg|\left\langle \mathcal{E}\left(z_{t}\right),\sum_{j=1}^{\rho}\left(g_{{M}_{4}}\left(\mathbf{z}\right)-g_{\hat{M}_{4}}\left(\mathbf{z}\right)\right)_j\!\!\!\cdot{M}_{3}\cdot\mathcal{E}\left(z_{t-j}\right)\right\rangle\Bigg|\\
        &\qquad+\max_{z'_{t}\neq z_{t}}\Bigg|\left[\left\langle \mathcal{E}\left(z'_{t}\right),\sum_{j=1}^{\rho}\left(g_{M_{4}}\left(\mathbf{z}\right)-g_{\hat{M}_{4}}\left(\mathbf{z}\right)\right)_j\!\!\!\cdot{M}_{3}\cdot\mathcal{E}\left(z_{t-j}\right)\right\rangle\right]\Bigg|\nonumber\\
        &\leq 2\max_{\overline{z}_t}\left|\left\langle \mathcal{E}\left(\overline{z}_{t}\right),\sum_{j=1}^{\rho}\left(g_{M_{4}}\left(\mathbf{z}\right)-g_{\hat{M}_{4}}\left(\mathbf{z}\right)\right)_j\!\!\!\cdot{M}_{3}\cdot\mathcal{E}\left(z_{t-j}\right)\right\rangle\right|\label{Difference_softmaxes_loss}
\end{align}

Define
\begin{align}
\vec{\Delta} g(\mathbf{z})&=g_{M_{4}}\left(\mathbf{z}\right)-g_{\hat{M}_{4}}\left(\mathbf{z}\right),\nonumber
\end{align}
and hence
\begin{align}
\eqref{Difference_softmaxes_loss}&=2\max_{\overline{z}_{t}}\left|\left\langle \mathcal{E}\left(\overline{z}_{t}\right),\sum_{j=1}^{\rho}\left(\vec{\Delta} g(\mathbf{z})\right)_j\cdot{M}_{3}\cdot\mathcal{E}\left(z_{t-j}\right)\right\rangle\right|\nonumber
\end{align}
Therefore, recalling $\circled{I}^1$ and $\circled{I}^2$ in \eqref{Definition_summarizing_losses_1} and \eqref{Definition_summarizing_losses_2}, respectively, we obtain
\begin{align}
    \circled{I}^1&=\frac{1}{T}\sum_{t=1}^{T}\mathbb{P}\left(D^{(1)}_{t-1:t-\rho}>\frac{\theta}{4}\right)\nonumber\\
    &\leq \frac{1}{T}\sum_{t=1}^{T}\mathbb{P}\left(2\max_{\overline{z}_{t}}\left|\left\langle \mathcal{E}\left(\overline{z}_{t}\right),\sum_{j=1}^{\rho}\left(\vec{\Delta} g(\mathbf{z})\right)_j\!\!\!\cdot {M}_{3}\cdot\mathcal{E}\left(z_{t-j}\right)\right\rangle\right|\geq\frac{\theta}{4}\right)\label{boudning_probability_loss_1}\\
    &=\frac{1}{T}\sum_{t=1}^{T}\mathbb{P}\left(\max_{\mathcal{E}\left(\overline{z}_{t}\right)}\left|\left\langle \mathcal{E}\left(\overline{z}_{t}\right),\sum_{j=1}^{\rho}\left(\vec{\Delta} g(\mathbf{z})\right)_j\!\!\!\cdot{M}_{3}\cdot\mathcal{E}\left(z_{t-j}\right)\right\rangle\right|\geq\frac{\theta}{8}\right)\nonumber\\
    &\leq\frac{1}{T}\sum_{t=1}^{T}\sum_{\mathcal{E}\left(\overline{z}_{t}\right)}\mathbb{P}\left(\left|\left\langle \mathcal{E}\left(\overline{z}_{t}\right),\sum_{j=1}^{\rho}\left(\vec{\Delta} g(\mathbf{z})\right)_j\!\!\!\cdot{M}_{3}\cdot\mathcal{E}\left(z_{t-j}\right)\right\rangle\right|\geq\frac{\theta}{8}\right)\label{boudning_probability_loss_2}\\
    &=\frac{1}{T}\sum_{t=1}^{T}\sum_{\mathcal{E}\left(\overline{z}_{t}\right)}\mathbb{P}\left(\left|\left\langle{M_{3}}^{\top}\cdot \mathcal{E}\left(\overline{z}_{t}\right),\sum_{j=1}^{\rho}\left(\vec{\Delta} g(\mathbf{z})\right)_j\cdot\mathcal{E}\left(z_{t-j}\right)\right\rangle\right|\geq\frac{\theta}{8}\right)\\
    &=\frac{1}{T}\sum_{t=1}^{T}\sum_{\ell\in[d]}\mathbb{P}\left(\left|\left\langle V\left[{M_{3}}^{\top},\ell\right],\sum_{j=1}^{\rho}\left(\vec{\Delta} g(\mathbf{z})\right)_j\cdot\mathcal{E}\left(z_{t-j}\right)\right\rangle\right|\geq\frac{\theta}{8}\right)\label{boudning_probability_loss_3}
\end{align}

where
\begin{itemize}
    \item Equation~\eqref{boudning_probability_loss_1} follows from the upper bound on $D^{(1)}_{t-1:t-\rho}$ derived in \eqref{Difference_softmaxes_loss}.

    \item Equation~\eqref{boudning_probability_loss_2} follows from the union bound.
\end{itemize}

For $i=3$, we consider the columns of $M_3^{\top}$, whereas, for $i=4$,
we consider the columns of $M_4$. In particular,
\begin{align}
\sum_{\ell\in[d]}
\left\|V\left[M_3^{\top},\ell\right]\right\|^2
&=
\left\|M_3\right\|_F^2,
\nonumber\\
\sum_{\ell\in[d]}
\left\|V\left[M_4,\ell\right]\right\|^2
&=
\left\|M_4\right\|_F^2.
\nonumber
\end{align}

For any $\ell\in[d]$, define
\begin{align}
\mathcal{J}_{\ell,3}
&\coloneqq
\mathbf{1}\left\{
\left\|
\mathsf{A}_3V[M_3^{\top},\ell]
\right\|
\geq
c_{3,1}
\left\|
V[M_3^{\top},\ell]
\right\|
\right\},
\nonumber\\
\mathcal{J}_{\ell,4}
&\coloneqq
\mathbf{1}\left\{
\left\|
\mathsf{A}_4V[M_4,\ell]
\right\|
\geq
c_{4,1}
\left\|
V[M_4,\ell]
\right\|
\right\}.
\nonumber
\end{align}

For $i\in\{3,4\}$, define
\begin{align}
Q_3 &\coloneqq M_3^{\top},
\qquad
Q_4 \coloneqq M_4.
\nonumber
\end{align}
Since the Frobenius norm is invariant under transposition, we have
\begin{align}
\|Q_i\|_F\leq C,
\qquad i\in\{3,4\}.
\nonumber
\end{align}

To define the compressed model, for each $i\in\{3,4\}$, let
\begin{align}
[d]_{M_i}
\coloneqq
\left\{
\ell\in[d]:
\left\|V[Q_i,\ell]\right\|
\geq
\frac{\theta}{8C\sqrt{\rho}}
\right\}.
\nonumber
\end{align}

First, let
\begin{align}
U_{\ell,i}
=
\begin{cases}
\mathsf{A}_i V[Q_i,\ell],
&
\text{if }
\left\|
\mathsf{A}_i V[Q_i,\ell]
\right\|
\leq
c_{i,1}
\left\|
V[Q_i,\ell]
\right\|,
\\
\mathbf{0}_{m_i},
&
\text{otherwise}.
\end{cases}
\nonumber
\end{align}

Then, define
\begin{align}
\hat{M}_{\ell,i}
=
\begin{cases}
\mathbf{0}_{m_i},
&
\text{if } \ell\notin[d]_{M_i},
\\
U_{\ell,i}+N_{\ell,i},
&
\text{otherwise},
\end{cases}
\nonumber
\end{align}
where
\begin{align}
N_{\ell,i}
\sim
\operatorname{Uniform}
\left(
\mathcal{B}_{m_i}(0,\nu_i)
\right).
\nonumber
\end{align}

Define the reconstructed matrices $\hat{Q}_i$ by
\begin{align}
V[\hat{Q}_i,\ell]
=
\mathsf{A}_i^\top
\hat{M}_{\ell,i},
\qquad
i\in\{3,4\}.
\nonumber
\end{align}
Finally, set
\begin{align}
\hat{M}_3
&\coloneqq
\hat{Q}_3^{\top},
\qquad
\hat{M}_4
\coloneqq
\hat{Q}_4.
\nonumber
\end{align}

Note that, for all $\ell\in[d]$ and $t\in[T]$, we have
\begin{align}
    &\left|\left\langle V\left[{{M}_{3}}^{\top},\ell\right],\sum_{j=1}^{\rho}\left(\vec{\Delta} g(\mathbf{z})\right)_j\cdot\mathcal{E}\left(z_{t-j}\right)\right\rangle\right|\nonumber\\
    &\quad\leq  \left\|V\left[{{M}_{3}}^{\top},\ell\right]\right\|_{\infty} \times \sum_{j=1}^{\rho}\left|\left(\vec{\Delta} g(\mathbf{z})\right)_j\right|\label{Diffference_lipchittc_querry_loss_3}\\
    &\quad =\left\|V\left[{{M}_{3}}^{\top},\ell\right]\right\|_{\infty} \times \left\|\vec{\Delta} g(\mathbf{z})\right\|_{1}\nonumber\\
    &\quad\leq\left\|V\left[{{M}_{3}}^{\top},\ell\right]\right\|_{\infty}\times\sum_{j=1}^{\rho}\left| g_{{M}_{4}}\left(\mathbf{z}\right)-g_{\hat{M}_{4}}\left(\mathbf{z}\right)\right|\nonumber\\
    &\quad\leq\left\|V\left[{M}_{3}^{\top},\ell\right]\right\|_{\infty} \times \frac{1}{\sqrt{\rho}}\sum_{j=1}^{\rho}\left|\left\langle \left({M}_{4}-\hat{M}_{4}\right)\cdot\mathcal{E}(z_{t-1}),\mathcal{E}({z}_{t-j})\right\rangle\right|\label{Diffference_lipchittc_querry_loss_6}\\
    &\quad\leq C\sqrt{\rho} \times\max_{\substack{j\in[\rho]\\z_{t-1} \\ z_{t-j}}}\left|\left\langle \left({M}_{4}-\hat{M}_{4}\right)\cdot\mathcal{E}(z_{t-1}),\mathcal{E}({z}_{t-j})\right\rangle\right|\nonumber
\end{align}

where
\begin{itemize}
     \item Equation~\eqref{Diffference_lipchittc_querry_loss_3} follows from H\"older's inequality:
\begin{align}
&\left|\left\langle V[M_3^\top,\ell],\sum_{j=1}^{\rho}\left(\vec{\Delta} g(\mathbf{z})\right)_j\cdot\mathcal{E}\left(z_{t-j}\right)\right\rangle\right|\leq  \left\|V[M_3^\top,\ell]\right\|_{\infty} \times \sum_{j=1}^{\rho}\left|\left(\vec{\Delta} g(\mathbf{z})\right)_j\right|.\nonumber
\end{align}

 \item Equation~\eqref{Diffference_lipchittc_querry_loss_6} follows from the 1-Lipschitz continuity of the softmax function with respect to the $\ell_1$ norm, i.e.,
\begin{align}
    \|\sigma(\mathbf{x})-\sigma(\mathbf{y})\|_1
    \le
    \|\mathbf{x}-\mathbf{y}\|_1,
    \qquad
    \forall\,\mathbf{x},\mathbf{y}\in\mathbb{R}^\rho.\nonumber
\end{align}

\item The last inequality follows from the assumption that $\left\|{M}_3\right\|\leq C$.
\end{itemize}

This implies that
\begin{align}
    \eqref{boudning_probability_loss_3}
    &\leq\sum_{\ell\in[d]}\mathbb{P}\!\left(\max_{\substack{j\in[\rho]\\z_{t-1}\\z_{t-j}}}
    \left|\left\langle
    \left({M}_{4}-\hat{M}_{4}\right)\cdot\mathcal{E}(z_{t-1}), \mathcal{E}(z_{t-j})
    \right\rangle\right|\geq\frac{\theta}{8C\sqrt{\rho}}\right)\label{Union_bound_type_bound_1}\\
    &\leq d\times\sum_{j=1}^{\rho}\sum_{\mathcal{E}(z'),\mathcal{E}(z'')}\mathbb{P}\!\left(\left|\left\langle\left({M}_{4}-\hat{M}_{4}\right)\cdot\mathcal{E}(z'), \mathcal{E}(z'')
    \right\rangle\right|\geq\frac{\theta}{8C\sqrt{\rho}}\right)\nonumber\\
    &= {\rho}d\times\!\!\sum_{\ell''=1}^{d}\sum_{\ell'=1}^{d}\mathbb{P}\left(\left|\left\langle V\left[\left({M}_{4}-\hat{M}_{4}\right),\ell'\right],V\left[\mathbf{I},\ell''\right]\right\rangle\right|\geq\frac{\theta}{8C\sqrt{\rho}}\right)\label{Union_bound_type_bound_2}\\
    &={\rho}d\times\sum_{\ell''=1}^{d}\sum_{\ell'\in[d]_{{M}_{4}}}\mathbb{P}\left(\left|
\left\langle
V\left[\left({M}_{4}-\hat{M}_{4}\right),\ell'\right],
V\left[\mathbf{I},\ell''\right]
\right\rangle
\right|\geq\frac{\theta}{8C\sqrt{\rho}}\right)  \label{Union_bound_type_bound_3}
\end{align}

where
\begin{itemize}
    \item The last inequality in \eqref{Union_bound_type_bound_1} follows from the union bound.

 \item Equation~\eqref{Union_bound_type_bound_2} follows by rewriting the notation according to the definitions in~\eqref{Denote_row_column_def}.

\item Equation~\eqref{Union_bound_type_bound_3} holds because, for any $(\ell',\ell'')$ with $\ell''\in[d]$ and $\ell'\notin [d]_{{M}_{4}}$, we have
\begin{align}
    \mathbb{P}\left(\left|V\left[\left({M}_{4}-\hat{M}_{4}\right),\ell'\right],V\left[\mathbf{I},\ell''\right]\right|\geq\frac{\theta}{8C\sqrt{\rho}}\right)=0.\nonumber
\end{align}
\end{itemize}

Hence, by rewriting \eqref{Union_bound_type_bound_3}, for $(\ell',\ell'')$ with $\ell''\in[d]$ and $\ell'\in [d]_{{M}_{4}}$, we have
\begin{align}
&\mathbb{P}\left(\left|\left\langle V\left[\left({M}_{4}-\hat{M}_{4}\right),\ell'\right],V\left[\mathbf{I},\ell''\right]\right\rangle\right|\geq\frac{\theta}{8C\sqrt{\rho}}\right)\nonumber\\
&\leq\mathbb{P}\left(\left|\left\langle V\left[{M}_{4},\ell'\right]-\mathsf{A}_4^{\top}\mathsf{A}_4V[M_4,\ell'],V\left[\mathbf{I},\ell''\right]\right\rangle\right|\geq\frac{\theta}{8C\sqrt{\rho}}\right)\nonumber\\
&\quad+\mathbb{P}\left(\left|\langle N_{\ell',4},\mathsf{A}_4V\left[\mathbf{I},\ell''\right]\rangle\right|\geq\frac{\theta}{8C\sqrt{\rho}}\Big|\left\|\mathsf{A}_4V\left[\mathbf{I},\ell''\right]\right\|\leq c_{4,2}\left\|V\left[\mathbf{I},\ell''\right]\right\|\right)\nonumber\\
&\quad+\mathbb{P}\left(\left\|\mathsf{A}_4V\left[\mathbf{I},\ell''\right]\right\|\geq c_{4,2}\left\|V\left[\mathbf{I},\ell''\right]\right\|\right)+\mathbb{P}\left(\mathcal{J}_{\ell,4}\right)\nonumber\\
&=\circled{T}_{1,1}+\circled{T}_{1,2}+\circled{T}_{1,3}+\circled{T}_{1,4},\nonumber
\end{align}

\begin{itemize}

\item\textbf{Bounding $\circled{T}_{1,1}$:}
By applying \cite[Lemma~8, Part~2]{gronlund2020near}, for any pair $(\ell',\ell'')$, we have
\begin{align}
&\mathbb{P}\Bigg(
\left|
\left\langle V\left[ M_{4},\ell'\right],V\left[\mathbf{I},\ell''\right]\right\rangle-
\left\langle \mathsf{A}_4V\left[ M_{4},\ell'\right],\mathsf{A}_4V\left[\mathbf{I},\ell''\right]\right\rangle\right|
\geq \frac{\theta}{8C\sqrt{\rho}}
\Bigg)\nonumber \\
&\qquad\leq 4\exp\!\left(
-\frac{m_4}{7}
\left(\frac{\theta}{8 C^2\sqrt{\rho}\,}\right)^2
\right)
\nonumber \\
&\qquad\leq \mathcal{O}\!\left(\frac{1}{\rho NTd^3}\right).
\nonumber
\end{align}

which implies that
\begin{align}
\circled{T}_{1,1} \leq \mathcal{O}\!\left(\frac{1}{\rho NTd^3}\right).
\nonumber
\end{align}

\item\textbf{Bounding $\circled{T}_{1,2}$:}

We have
\begin{align}
   &\mathbb{P}\left(\Big|\langle N_{\ell',4},\mathsf{A}_4V\left[\mathbf{I},\ell''\right]\rangle\Big|\geq\frac{\theta}{8C\sqrt{\rho}}\Big|\left\|\mathsf{A}_4V\left[\mathbf{I},\ell''\right]\right\|\leq c_{4,2}\left\|V\left[\mathbf{I},\ell''\right]\right\|\right)\nonumber\\
   &\qquad\leq\frac{2m_4\nu_4^{m_4}}{\sqrt{\pi}}e^{-\frac{(m_4+1)}{2}\left(\frac{\theta}{8c_{4,2}\nu_4 C\sqrt{\rho}}\right)^2}\label{Inequality_sphere_unif}\\
   &\qquad\leq\mathcal{O}\left(\frac{1}{\rho NTd^3}\right)\nonumber
\end{align}

where inequality~\eqref{Inequality_sphere_unif} corresponds to Part~iii of \cite[page~26]{sefidgaran2022distributed}.

\item \textbf{Bounding $\circled{T}_{1,3}$:}
For any pair $(\ell',\ell'')$ such that $\ell'\in[d]_{{M}_{4}}$ and $\ell''\in[d]$, we have
\begin{align}
&\mathbb{P}\left(\left\|\mathsf{A}_4V\left[\mathbf{I},\ell''\right]\right\|\geq c_{4,2}\left\|V\left[\mathbf{I},\ell''\right]\right\|\right)\nonumber\\
&\qquad\leq 2e^{-0.21m_4\left(c^2_{4,2}-1-2\log (c_{4,2}) \right)}\label{Bounding_T_{1,3}}\\
&\qquad\leq\mathcal{O}\left(\frac{1}{\rho NTd^3}\right)\nonumber
\end{align}

where
\begin{itemize}
    \item Equation~\eqref{Bounding_T_{1,3}} follows from
    \cite[Lemma~8, Part~1]{gronlund2020near}.
\end{itemize}

\item\textbf{Bounding $\circled{T}_{1,4}$:}
Following the same argument as in \eqref{Bounding_T_{1,3}}, we have
\begin{align}
&\mathbb{P}\left(\mathcal{J}_{\ell,4}\right)=\mathbb{P}\left(\left\|\mathsf{A}_4V[ M_{4},\ell]\right\|\geq c_{4,1}\left\|V\left[M_{4},\ell\right]\right\|\right)\nonumber\\
     &\qquad\leq 2e^{-0.21m_4\left(c^2_{4,1}-1-2\log (c_{4,1}) \right)}\nonumber\\
     &\qquad\leq\mathcal{O}\left(\frac{1}{\rho NTd^3}\right)\nonumber
\end{align}

Therefore, this implies that
\begin{align}
    \circled{I}^1,\circled{I}^2
    &\leq \eqref{Union_bound_type_bound_3}\leq \mathcal{O}\!\left(\frac{1}{NT}\right)\nonumber
\end{align}

This completes the proof of the bounds on $\circled{I}^1$ and $\circled{I}^2$ in \eqref{Definition_summarizing_losses_1} and \eqref{Definition_summarizing_losses_2}.

\item\textbf{Bounding $\circled{II}^1$ and $\circled{II}^2$:}
To bound the terms $\circled{II}^1$ and $\circled{II}^2$, we have
\begin{align}
    \circled{II}^1=\mathbb{E}_{S,\hat{W},\hat{W}_1}\Bigg[\frac{1}{T}\sum_{t \in [T]}
        \mathbf{1}\left\{\left|\log\frac{C_1(\hat{W}_{1},t,\mathbf{z})}{C_2(\hat{W}_{1},t,\mathbf{z})}-\log\frac{C_1(\hat{W},t,\mathbf{z})}{C_2(\hat{W},t,\mathbf{z})}\right|>\frac{\theta}{4}\right\},\nonumber
\end{align}

By considering
\begin{align}
    &D^{(2)}_{t-1:t-\rho}
    \coloneqq
    \left|
    \log\frac{C_1(\hat{W}_{1},t,\mathbf{z})}
    {C_2(\hat{W}_{1},t,\mathbf{z})}
    -
    \log\frac{C_1(\hat{W},t,\mathbf{z})}
    {C_2(\hat{W},t,\mathbf{z})}
    \right|,
    \nonumber
\end{align}
and following the same argument as in \eqref{Difference_softmaxes_loss}, we have
\begin{align}
    &D^{(2)}_{t-1:t-\rho}
    \leq
    2\max_{\overline{z}_{t}}
    \left|
    \left\langle
    \mathcal{E}\left(\overline{z}_{t}\right),
    \sum_{j=1}^{\rho}
    \left(g_{\hat{M}_{4}}
    \left(\mathbf{z}\right)\right)_j
    \left({M}_{3}
    -
    \hat{M}_{3}\right)
    \mathcal{E}\left(z_{t-j}\right)
    \right\rangle
    \right|,
    \label{Difference_softmaxes_loss_3}
\end{align}
which implies that
\begin{align}
    \!\!\!\!\!\!\!\!\!\!\!\!\!\!\!\!\!\!\!\!\!\!\!\!\!\!\!\!\circled{II}^1
    &=
    \frac{1}{T}\sum_{t=1}^{T}
    \mathbb{P}\left(D^{(2)}_{t-1:t-\rho}>\frac{\theta}{4}\right)
    \nonumber\\
    &\!\!\!\!\!\!\!\!\!\!\!\!\!\!\!\!\!\!\!\!\!\!\!\!\!\!\!\!\leq
    \frac{1}{T}\sum_{t=1}^{T}\sum_{\ell\in[d]}
    \mathbb{P}\left(
    \left|
    \left\langle
    V\left[
    \left({M_{3}}-{\hat{M}}_{3}\right)^{\top},\ell
    \right],
    \sum_{j=1}^{\rho}
    \left(g_{\hat{M}_{4}}(\mathbf{z})\right)_j
    \mathcal{E}(z_{t-j})
    \right\rangle
    \right|
    \geq \frac{\theta}{8}
    \right)
    \nonumber\\
    &\!\!\!\!\!\!\!\!\!\!\!\!\!\!\!\!\!\!\!\!\!\!\!\!\!\!\!\!=
    \frac{1}{T}\sum_{t=1}^{T}
    \sum_{\ell\in[d]_{{M}_{3}}}
    \mathbb{P}\left(
    \left|
    \left\langle
    V\left[
    \left({{M}_{3}}-{\hat{M}}_{3}\right)^{\top},\ell
    \right],
    \sum_{j=1}^{\rho}
    \left(g_{\hat{M}_{4}}(\mathbf{z})\right)_j
    \mathcal{E}(z_{t-j})
    \right\rangle
    \right|
    \geq \frac{\theta}{8}
    \right).\nonumber
\end{align}

where the last equality follows from the fact that, for any
$\ell\notin[d]_{{M}_{3}}$,
\begin{align}
    \mathbb{P}\left(
    \left|
    \left\langle
    V\left[
    \left({{M}_{3}}-{\hat{M}}_{3}\right)^{\top},\ell
    \right],
    \sum_{j=1}^{\rho}
    \left(g_{\hat{M}_{4}}(\mathbf{z})\right)_j
    \mathcal{E}(z_{t-j})
    \right\rangle
    \right|
    \geq\frac{\theta}{8}
    \right)
    =0.\nonumber
\end{align}

This follows because
\begin{align}
&\left|
\left\langle
V\left[
\left({{M}_{3}}-{\hat{M}}_{3}\right)^{\top},\ell
\right],
\sum_{j=1}^{\rho}
\left(g_{\hat{M}_{4}}(\mathbf{z})\right)_j
\mathcal{E}(z_{t-j})
\right\rangle
\right|
\nonumber\\
&\leq
\left\|
V\left[
\left({{M}_{3}}-{\hat{M}}_{3}\right)^{\top},\ell
\right]
\right\|_{\infty}
\left\|
\sum_{j=1}^{\rho}
\left(g_{\hat{M}_{4}}(\mathbf{z})\right)_j
\mathcal{E}(z_{t-j})
\right\|_1
\label{Holder_upperbound_1}\\
&\leq
\left\|
V\left[
\left({{M}_{3}}-{\hat{M}}_{3}\right)^{\top},\ell
\right]
\right\|_{\infty}
\label{Holder_upperbound_2}\\
&\leq\left\|V\left[
\left({{M}_{3}}-{\hat{M}}_{3}\right)^{\top},\ell
\right]\right\|\nonumber\\
&\leq \frac{\theta}{8}\nonumber
\end{align}

where
\begin{itemize}
    \item The inequality in \eqref{Holder_upperbound_1} follows from H\"older's inequality.

    \item Equation~\eqref{Holder_upperbound_2} follows from the fact that
    $\left(g_{\hat{M}_{4}}(\mathbf{z})\right)_j$
    are elements of probability distributions, which implies
    \begin{align}
    \left\|
    \sum_{j=1}^{\rho}
    \left(g_{\hat{M}_{4}}(\mathbf{z})\right)_j
    \mathcal{E}(z_{t-j})
    \right\|_1
    \leq 1.\nonumber
    \end{align}
\end{itemize}

Therefore, it suffices to show that, uniformly over any $t\in[T]$ and
$\ell\in[d]_{{M}_{3}}$,
\begin{align}
&\mathbb{P}\left(\left|\left\langle V\left[
\left({{M}_{3}}-{\hat{{M}}}_{3}\right)^{\top},\ell
\right],
\sum_{j=1}^{\rho}\left(g_{\hat{{M}}_{4}}(\mathbf{z})\right)_j
\mathcal{E}(z_{t-j})\right\rangle\right|\geq\frac{\theta}{8}\right)\leq\mathcal{O}\left(\frac{1}{dNT}\right).\nonumber
\end{align}
Define
\begin{align}
    \vec{g}_{\hat{{M}}_{4}}
    \coloneqq
    \sum_{j=1}^{\rho}
    \left(g_{\hat{M}_{4}}(\mathbf{z})\right)_j
    \mathcal{E}(z_{t-j}),\label{Denoting_Vector_attention}
\end{align}
Then, for any $t\in[T]$ and $\ell\in[d]$, we have
\begin{align}
&\mathbb{P}\left(\left|\left\langle V\left[
\left({{M}_{3}}-{\hat{{M}}}_{3}\right)^{\top},\ell
\right],
\sum_{j=1}^{\rho}\left(g_{\hat{{M}}_{4}}(\mathbf{z})\right)_j
\mathcal{E}(z_{t-j})\right\rangle\right|\geq \frac{\theta}{8}\right),\nonumber\\
&\quad=\mathbb{P}\left(\left|\left\langle V\left[
\left({{M}_{3}}-{\hat{{M}}}_{3}\right)^{\top},\ell
\right],\vec{g}_{\hat{{M}}_{4}}\right\rangle\right|\geq \frac{\theta}{8}\right)\nonumber\\
&\quad\leq\mathbb{P}\left(\left|\langle V[M^{\top}_3,\ell]-\mathsf{A}^\top_3 \mathsf{A}_3V[M^\top_3,\ell],\vec{g}_{\hat{{M}}_{4}}\rangle\right|\geq\frac{\theta}{8}\right)\nonumber\\
&\qquad +\mathbb{P}\left(\left|\langle N_{\ell,3},\mathsf{A}_3\vec{g}_{\hat{{M}}_{4}}\rangle\right|\geq\frac{\theta}{8}\Big|\left\|\mathsf{A}_3\vec{g}_{\hat{{M}}_{4}}\right\|\leq c_{3,2}\right)\nonumber\\
&\qquad+\mathbb{P}_{\mathsf{A}_3}\left(\left\|\mathsf{A}_3\vec{g}_{\hat{{M}}_{4}}\right\|\geq c_{3,2}\right)+\mathbb{P}\left(\mathcal{J}_{\ell,3}\right)\nonumber\\
&\quad=\circled{T}_{2,1}+\circled{T}_{2,2}+\circled{T}_{2,3}+\circled{T}_{2,4},\nonumber
\end{align}

\begin{itemize}

\item \textbf{Bounding $\circled{T}_{2,1}$:}
Recalling the notation in \eqref{Denoting_Vector_attention}, we have
\begin{align}
  \circled{T}_{2,1}= &\mathbb{P}\left(
    \left|
    \left\langle
    V\!\left[
    \left( M_{3}
    -\mathsf{A}_3^\top\mathsf{A}_3 M_{3}
    \right)^{\!\top},
    \ell
    \right],
    \vec{g}_{\hat{M}_{4}}
    \right\rangle
    \right|
    \geq\frac{\theta}{8}
    \right)
    \nonumber\\
    &=
    \mathbb{P}\left(
    \left|
    \left\langle
    V\!\left[
     M_{3}^{\!\top},
    \ell
    \right],
    \vec{g}_{\hat{{M}}_{4}}
    \right\rangle
    -
    \left\langle
    \mathsf{A}_3
    V\!\left[
    M_{3}^{\!\top},
    \ell
    \right],
    \mathsf{A}_3
    \vec{g}_{\hat{{M}}_{4}}
    \right\rangle
    \right|
    \geq\frac{\theta}{8}
    \right)
    \nonumber\\
    &\leq
    4\exp\!\left(
    -\frac{m_3}{7}
    \left(
    \frac{\theta}
    {8 C }
    \right)^2
    \right)
    \label{multiplile_concentration_1}\\
    &\leq
    \mathcal{O}\!\left(\frac{1}{dNT}\right).
    \label{multiplile_concentration_2}
\end{align}

where
\begin{itemize}
    \item Equation~\eqref{multiplile_concentration_1} follows by applying \cite[Lemma~8, Part~2]{gronlund2020near}.

    \item Equation~\eqref{multiplile_concentration_2} follows from the choice of the parameter \(m_3\).
\end{itemize}

\item \textbf{Bounding $\circled{T}_{2,2}$:}

\begin{align}
    \circled{T}_{2,2}&=\mathbb{P}\left(\left|\langle N_{\ell,3},\mathsf{A}_3\vec{g}_{\hat{{M}}_{4}}\rangle\right|\geq\frac{\theta}{8}\Big|\left\|\mathsf{A}_3\vec{g}_{\hat{{M}}_{4}}\right\|\leq c_{3,2}\right)\nonumber\\
    &\leq\frac{2m_3\nu_3^{m_3}}{\sqrt{\pi}}e^{-\left(\frac{m_3+1}{2}\right)\left(\frac{\theta}{8 c_{3,2}\nu_3 C}\right)^2}\label{Bounding_unif_sphere_1}\\
    &\leq\mathcal{O}\left(\frac{1}{dNT}\right),\label{Bounding_unif_sphere_2}
\end{align}
where
\begin{itemize}
    \item \eqref{Bounding_unif_sphere_1} follows from the same argument as Part~III of \cite[page~26]{sefidgaran2022distributed}.
    \item \eqref{Bounding_unif_sphere_2} follows by choosing $c_{3,2}$, $\nu_3$, and $m_3$.
\end{itemize}

\begin{itemize}
    \item \textbf{Bounding $\circled{T}_{2,3}$:}
\begin{align}
    \circled{T}_{2,3}&=\mathbb{P}_{\mathsf{A}_3}\left(\left\|\mathsf{A}_3\vec{g}_{\hat{{M}}_{4}}\right\|\geq c_{3,2}\right)\nonumber\\
    &\leq2e^{-0.21m_3\left(c^2_{3,2}-1-2\log (c_{3,2}) \right)}\label{Simlification_third_thitrm_term_3}\\
    &\leq\mathcal{O}\left(\frac{1}{dNT}\right)\nonumber
\end{align}

\item Equation~\eqref{Simlification_third_thitrm_term_3} follows from
    \cite[Lemma~8, Part~1]{gronlund2020near}.

    \item \textbf{Bounding $\circled{T}_{2,4}$:}
\begin{align}
    \circled{T}_{2,4}&=\mathbb{P}\left(\mathcal{J}_{\ell,3}\right)\nonumber\\
    &\leq2e^{-0.21m_3\left(c^2_{3,1}-1-2\log (c_{3,1}) \right)}\label{Simlification_third_thitrm_term_4}\\
    &\leq\mathcal{O}\left(\frac{1}{dNT}\right)\nonumber
\end{align}

\item Equation~\eqref{Simlification_third_thitrm_term_4} follows from
    \cite[Lemma~8, Part~1]{gronlund2020near}.

\end{itemize}

Therefore,
\begin{align}
    \circled{II}^1,\ \circled{II}^2
    &\leq
    \mathcal{O}\!\left(\frac{1}{NT}\right).
    \nonumber
\end{align}
By combining the bounds for $\circled{I}^1,\circled{I}^2$ and $\circled{II}^1,\circled{II}^2$, we conclude that the distortion term is bounded by $\mathcal{O}\!\left(\frac{1}{NT}\right)$.

\end{itemize}

\end{itemize}

\end{itemize}

In the final step, it remains to derive an upper bound on the mutual information term. To this end, define
\[
\mathbf{T}_{\{3,4\},[d]}
\coloneqq
\bigl(
\mathbf{T}_{3,[d]},
\mathbf{T}_{4,[d]}
\bigr).
\]

For each $i\in\{3,4\}$, let
\begin{align}
\mathbf{T}_{i,[d]}
\coloneqq
\left(
T_{i,1},\ldots,T_{i,d}
\right)
\in\{0,1\}^{d},
\nonumber
\end{align}
where, for every $\ell\in[d]$,
\begin{align}
T_{i,\ell}
\coloneqq
\mathbf{1}\left\{\ell\in[d]_{M_i}\right\}.
\nonumber
\end{align}
Hence, if $T_{i,\ell}=0$, then
\begin{align}
V[\hat Q_i,\ell]=0,
\qquad
V[Q_i-\hat Q_i,\ell]=V[Q_i,\ell].\nonumber
\end{align}

Therefore, we have
\begin{align}
   &\mathbb{I}^{\mathsf{A}_{\{3,4\}}}\left(\mathbf{S};\hat{W}\right)\nonumber\\
   &\quad\leq \mathbb{I}^{\mathsf{A}_{\{3,4\}}}\left({W};\hat{W}\right)\label{Bound_MI_Binary_SANTP_1}\\
  &\quad=\mathbb{I}^{\mathsf{A}_{\{3,4\}}}\left(\left\{V(W,\ell)\right\}_{\ell\in[d]};\left\{V(\hat{W},\ell)\right\}_{\ell\in[d]}\right)\nonumber\\
  &\quad=\mathbb{I}^{\mathsf{A}_{\{3,4\}}}\left(\left\{V(Q_{i},\ell)\right\}_{\ell\in[d],i\in\{3,4\}};\left\{V(\hat{Q}_{i},\ell)\right\}_{\ell\in[d],i\in\{3,4\}}\right)\label{Bound_MI_Binary_SANTP_2}\\
  &\quad\leq \mathbb{I}^{\mathsf{A}_{\{3,4\}}}\left(\left\{V(Q_{i},\ell)\right\}_{\ell\in[d],i\in\{3,4\}};\left\{V(\hat{Q}_{i},\ell)\right\}_{\ell\in[d],i\in\{3,4\}},\mathbf{T}_{\{3,4\},[d]}\right)\nonumber\\
  &\quad = \mathbb{I}^{\mathsf{A}_{\{3,4\}}}\left(\left\{V(Q_{i},\ell)\right\}_{\ell\in[d],i\in\{3,4\}};\left\{V(\hat{Q}_{i},\ell)\right\}_{\ell\in[d],i\in\{3,4\}}\Big|\mathbf{T}_{\{3,4\},[d]}\right)\nonumber\\
  &\quad\qquad+\mathbb{I}^{\mathsf{A}_{\{3,4\}}}\left(\left\{V(Q_{i},\ell)\right\}_{\ell\in[d],i\in\{3,4\}};\mathbf{T}_{\{3,4\},[d]}\right)\nonumber\\
  &\quad\leq\mathbb{I}^{\mathsf{A}_{\{3,4\}}}\left(\left\{V(Q_{i},\ell)\right\}_{\ell\in[d],i\in\{3,4\}};\left\{\hat{M}_{\ell,i}\right\}_{\ell\in[d],i\in\{3,4\}}\Big|\mathbf{T}_{\{3,4\},[d]}\right)+\mathbb{H}\left(\mathbf{T}_{\{3,4\},[d]}\right)\nonumber\\
  &\quad=\mathbb{I}^{\mathsf{A}_{\{3,4\}}}\left(\left\{V\left(Q_{3},\ell\right)\right\}_{\ell\in[d]};\left\{\hat{M}_{\ell,i}\right\}_{\ell\in[d],i\in\{3,4\}}\Big|\mathbf{T}_{\{3,4\},[d]}\right)\nonumber\\
  &\quad\qquad +\mathbb{I}^{\mathsf{A}_{\{3,4\}}}\left(\left\{V\left(Q_{4},\ell\right)\right\}_{\ell\in[d]};\left\{\hat{M}_{\ell,i}\right\}_{\ell\in[d],i\in\{3,4\}}\Big|\left\{V\left(Q_{3},\ell\right)\right\}_{\ell\in[d]},\mathbf{T}_{\left\{3,4\right\},[d]}\right)+\mathbb{H}\left(\mathbf{T}_{\{3,4\},[d]}\right)\nonumber\\
  &\quad=\mathbb{I}^{\mathsf{A}_{\{3,4\}}}\left(\left\{V\left(Q_{3},\ell\right)\right\}_{\ell\in[d]};\left\{\hat{M}_{\ell,3}\right\}_{\ell\in[d]}\Big|\mathbf{T}_{\{3,4\},[d]}\right)\nonumber\\
  &\quad\qquad+\mathbb{I}^{\mathsf{A}_{\{3,4\}}}\left(\left\{V\left(Q_{3},\ell\right)\right\}_{\ell\in[d]};\left\{\hat{M}_{\ell,4}\right\}_{\ell\in[d]}\Big|\left\{\hat{M}_{\ell,3}\right\}_{\ell\in[d]},\mathbf{T}_{\{3,4\},[d]}\right)\nonumber\\
  &\qquad\quad+\mathbb{I}^{\mathsf{A}_{\{3,4\}}}\left(\left\{V\left(Q_{4},\ell\right)\right\}_{\ell\in[d]};\left\{\hat{M}_{\ell,3}\right\}_{\ell\in[d]}\Big|\left\{V\left(Q_{3},\ell\right)\right\}_{\ell\in[d]},\mathbf{T}_{\{3,4\},[d]}\right)\nonumber\\
  &\qquad\quad+\mathbb{I}^{\mathsf{A}_{\{3,4\}}}\left(\left\{V\left(Q_{4},\ell\right)\right\}_{\ell\in[d]};\left\{\hat{M}_{\ell,4}\right\}_{\ell\in[d]}\Big|\left\{V\left(Q_{3},\ell\right)\right\}_{\ell\in[d]},\left\{\hat{M}_{\ell,3}\right\}_{\ell\in[d]},\mathbf{T}_{\left\{3,4\right\},[d]}\right)+\mathbb{H}\left(\mathbf{T}_{\left\{3,4\right\},[d]}\right)\nonumber\\
  &\quad\leq 2\sum_{i\in\left\{3,4\right\}}\mathbb{I}^{\mathsf{A}_{[3]}}\left(\left\{V\left(Q_{i},\ell\right)\right\}_{\ell\in[d]};\left\{\hat{M}_{\ell,i}\right\}_{\ell\in[d]}\Big|\mathbf{T}_{\left\{3,4\right\},[d]}\right)+\mathbb{H}\left(\mathbf{T}_{\left\{3,4\right\},[d]}\right)\label{Bound_MI_Binary_SANTP_3}\\
  &\quad=2\!\!\!\sum_{i\in\left\{3,4\right\}}\!\!\!\!\mathbb{E}_{\mathbf{T}_{i,[d]}}\!\!\left[\mathbb{I}^{\mathsf{A}_{\left\{3,4\right\}},\mathbf{T}_{i,[d]}}\left(\!\!\left\{V(Q_{i},\ell)\right\}_{\ell\in[d]};\left\{\hat{M}_{\ell,i}\right\}_{\ell\in[d]}\right)\right]+\mathbb{H}\left(\mathbf{T}_{\left\{3,4\right\},[d]}\right),\label{Bound unofrom MI}
\end{align}

Rewriting the definition of mutual information in \eqref{Bound unofrom MI}, we have
\begin{align}
\eqref{Bound unofrom MI}&=2\sum_{i\in\left\{3,4\right\}}\mathbb{E}_{\mathbf{T}_{i,[d]}}\Bigg[h^{\mathsf{A}_{\{3,4\}},\mathbf{T}_{i,[d]}}\left(\left\{\hat{M}_{\ell,i}\right\}_{\ell\in[d]}\right)-h^{\mathsf{A}_{\left\{3,4\right\}},\mathbf{T}_{i,[d]}}\left(\left\{\hat{M}_{\ell,i}\right\}_{\ell\in[d]}\Big |\left\{V(Q_{i},\ell)\right\}_{\ell\in[d]}\right)\Bigg]\nonumber\\
&\quad\qquad\qquad+\mathbb{H}\left(\mathbf{T}_{\left\{3,4\right\},[d]}\right)\nonumber\\
&\quad\leq 2\sum_{i\in\{3,4\}}\mathbb{E}_{\mathbf{T}_{i,[d]}}\Bigg[\sum_{\ell\in[d]}h^{\mathsf{A}_{\left\{3,4\right\}},\mathbf{T}_{i,[d]}}\left(\hat{M}_{\ell,i}\right)-h^{\mathsf{A}_{[3]},\mathbf{T}_{i,[d]}}\left(\left\{\hat{M}_{\ell,i}\right\}_{\ell\in[d]}\Big|\left\{V(Q_{i},\ell)\right\}_{\ell\in[d]}\right)\Bigg]+\mathbb{H}\left(\mathbf{T}_{\left\{3,4\right\},[d]}\right)\nonumber\\
&\quad=2\sum_{i\in\{3,4\}}
\mathbb{E}_{\mathbf{T}_{i,[d]}}
\Bigg[
\sum_{\ell\in[d]}
h^{\mathsf{A}_{\{3,4\}},\mathbf{T}_{\{3,4\},[d]}}
\left(\hat{M}_{\ell,i}\right) -
h^{\mathsf{A}_{\{3,4\}},\mathbf{T}_{\{3,4\},[d]}}
\left(
\hat{M}_{\ell,i}
\Big|
V(Q_i,\ell)
\right)
\Bigg]
+
\mathbb{H}\left(\mathbf{T}_{\{3,4\},[d]}\right).\nonumber\\
&\quad\leq 2\sum_{i\in\left\{3,4\right\}}\Bigg[\left[\mathbb{E}_{\mathbf{T}_{i,[d]}}\left[\sum_{\ell\in[d]}T_{i,\ell}\right]\log\left(\text{Volume}\left(\mathcal{B}_{m_i}\left(c_{i,1}C+\nu_i\right)\right)\right)\right]\nonumber\\
&\qquad\qquad\qquad-\mathbb{E}_{\mathbf{T}_{i,[d]}}\left[\sum_{\ell\in\mathbf{T}_{i,[d]}}T_{i,\ell}h^{\mathsf{A}_{\left\{3,4\right\}},\mathbf{T}_{\left\{3,4\right\},[d]}}\left(\hat{M}_{\ell,i}\Big|V(Q_{i},\ell)\right)\right]\Bigg]+\mathbb{H}\left(\mathbf{T}_{\left\{3,4\right\},[d]}\right)\nonumber\\
&\quad=2\sum_{i\in\{3,4\}}\Bigg[\left[\mathbb{E}_{\mathbf{T}_{i,[d]}}\left[\sum_{\ell\in[d]}T_{i,\ell}\right]\log\left(\text{Volume}\left(\mathcal{B}_{m_i}\left(c_{i,1}C+\nu_i\right)\right)\right)\right]\nonumber\\
&\qquad\qquad\qquad-\left[\mathbb{E}_{\mathbf{T}_{i,[d]}}\left[\sum_{\ell\in[d]}T_{i,\ell}\right]\log\left(\text{Volume}\left(\mathcal{B}_{m_i}\left(\nu_i\right)\right)\right)\right]\Bigg]+\mathbb{H}\left(\mathbf{T}_{\left\{3,4\right\},[d]}\right)\nonumber\\
&\quad\leq 2\sum_{i\in\left\{3,4\right\}} \left[m_i\mathbb{E}_{\mathbf{T}_{i,[d]}}\left[\sum_{\ell\in[d]}T_{i,\ell}\right]\log\left(\frac{c_{i,1}C+\nu_i}{\nu_i}\right)+\mathbb{H}\left(\mathbf{T}_{i,[d]}\right)\right]\nonumber\\
&\quad\leq
2\sum_{i\in\{3,4\}}
\Bigg[
m_i K_i
\log\left(
\frac{c_{i,1}C+\nu_i}{\nu_i}
\right)
+
\mathcal{O}\left(
K_i
\left[
1+\log\left(
\frac{d}{K_i}
\right)
\right]
\right)
\Bigg].
\nonumber\nonumber
\end{align}

where
\begin{itemize}
    \item Equations~\eqref{Bound_MI_Binary_SANTP_1} and~\eqref{Bound_MI_Binary_SANTP_2} follow from the data-processing inequality.

\item The inequality in~\eqref{Bound_MI_Binary_SANTP_3} follows from the fact that, for any $i\in\{3,4\}$, conditioned on
$V\!\left(\{Q_{i,\ell}\}_{\ell\in[d]}\right)$, the random variable
$\hat{M}_{i,[d]}$
is independent of
$V\!\left(\{Q_{j,\ell}\}_{\ell\in[d]}\right)$ for all $j\neq i$.

\item $K_i =\min\left\{d,\,\left(\frac{8C^2\sqrt{\rho}}{\theta}\right)^2\right\}$
\end{itemize}

It remains to upper-bound $\mathbb{H}\left(\mathbf{T}_{i,[d]}\right)$. This follows the same argument used to bound $\mathbb{H}\left(\mathbf{T}_{i,[d]}\right)$ in \eqref{Bound_MI_Binary_3}. Since the argument is identical for each $i\in\{3,4\}$, we omit the details.

Therefore,
\begin{align}
I(\mathbf{S};\hat{W})
&=
\mathcal{O}\left(
\frac{C^8\rho^2}{\theta^4}
\log(\rho dNT)
\right).
\nonumber
\end{align}

Combining this rate bound with the distortion bound
$\epsilon=\mathcal{O}(1/(NT))$ and the lossy
margin-generalization bound yields
\begin{align}
\mathbb{E}\left[\gen_{\theta}(\mathbf{S},W)\right]
&\leq
\mathcal{O}\left(
C^4
\sqrt{
\frac{
\rho^2\widetilde{\tau}_{\min}\log(\rho dNT)
}{
\theta^4NT
}}
\right),
\nonumber
\end{align}
which completes the proof.



\subsection{Proof of Lemma \ref{Hoeffdinf_Lemma}}\label{Proof_Hoeffding_employing}

To prove Lemma~\ref{Hoeffdinf_Lemma}, we follow the same steps as
\eqref{multi_round_generalization_bound_1}--\eqref{multi_round_generalization_bound_10}.
In particular, analogously to \eqref{bounding_moment_10}, we obtain
\begin{align}
    \mathbb{E}\left[\gen(\mathbf{S},W)\right]
    \leq
    \frac{I(\mathbf{S};W)}{\lambda}
    +
    \frac{1}{\lambda}
    \sum_{n=1}^{N}
    \mathbb{E}_{W}\Bigg[
        \log
        \mathbb{E}_{\mathbf{S}}\Bigg[
            e^{
            \frac{\lambda}{N}
            \left(
                \mathbb{E}_{\mathbf{S}'}
                \left[
                    \ell(\mathbf{Z}'^{(n)},W)
                \right]
                -
                \ell(\mathbf{Z}^{(n)},W)
            \right)}
        \Bigg]
    \Bigg].
\label{Hoeffding_MGF_Gen_error}
\end{align}

Thus, it remains to bound the MGF term on the right-hand side of
\eqref{Hoeffding_MGF_Gen_error}. To this end, define
\begin{align}
    Y_t
    &\coloneqq\frac{1}{T}\log
    \frac{1}{
        p_W(Z_t|Z_{t-1},\ldots,Z_{t-\rho})
    },
    \nonumber\\
    Y_t^{(n)}
    &\coloneqq\frac{1}{T}\log
    \frac{1}{
        p_W(Z_t^{(n)}|Z_{t-1}^{(n)},\ldots,Z_{t-\rho}^{(n)})
    }.\nonumber
\end{align}
Then,
\begin{align}
    \ell(\mathbf{z},w)
    &=
    \frac{1}{T}\sum_{t=1}^{T}
    \log
    \frac{1}{
        p_w(z_t\mid z_{t-1},\ldots,z_{t-\rho})
    }
    =\sum_{t=1}^{T}Y_t.
    \nonumber
\end{align}
Consequently,
\begin{align}
    &\log
    \mathbb{E}_{\mathbf{S}}\Bigg[
        e^{{\frac{\lambda}{N}}
        \left(
            \mathbb{E}_{\mathbf{S}'}
            \left[
                \sum_{t=1}^{T}Y_t'
            \right]
            -
            \sum_{t=1}^{T}Y^{(n)}_t
        \right)}
    \Bigg]\leq
    \frac{
        \lambda^2
        (2B+\log d)^2
        \tilde{\tau}_{\min}
    }{
        8TN^2
    }
    \label{Hoeefding_inequality_MGF_Bounding}
\end{align}
where
\begin{itemize}

\item $\tau_{\min}$ is defined in \eqref{new_tau-min}.

    \item \eqref{Hoeefding_inequality_MGF_Bounding} follows by applying the
    McDiarmid-type inequality in Lemma~\ref{lem:window-loss-concentration}.
    Indeed, since $Y_t\geq0$ for every $t\in[T]$, the function
    \begin{align}
        g(Y_1,\ldots,Y_T)
        \coloneqq\frac{1}{N}
        \sum_{t=1}^{T}Y_t
        \nonumber
    \end{align}
    satisfies
    \begin{align}
        g(Y_1,\ldots,Y_T)
        -
        g(Y'_1,\ldots,Y'_T)
        &\leq\frac{1}{N}
        \sum_{t=1}^{T}
        Y_t\,
        \mathrm{1}\{Y_t\neq Y'_t\}\leq
        \left(\frac{2B+\log d}{NT}\right)
        \sum_{t=1}^{T}
        \mathrm{1}\{Y_t\neq Y'_t\}
        \label{max_bound_prob_y_t}
    \end{align}
    where \eqref{max_bound_prob_y_t} follows from the uniform upper bound on
    $Y_t$ established in Lemma~\ref{max_probability_conditional_markov}.
\end{itemize}
   


\subsection{Proof of Lemma \ref{lem:window-loss-concentration}}

\begin{proof}
We apply the McDiarmid-type inequality for dependent random
variables in Lemma \ref{Mdiarmid_type_lemma}. For a partition
$\widehat Y$ with mixing matrix $\Gamma$, this inequality gives
\[
\log\mathbb E\left[
e^{\lambda(f(Y)-\mathbb Ef(Y))}\right]
\le
\frac{\lambda^2}{8}\|\Gamma C(c)\|_2^2.
\]
It therefore remains to construct a Marton coupling for the
loss sequence and bound its mixing matrix.

Fix $\delta\in(0,1)$ such that $\tilde{\tau}(\delta)<\infty$,
set $L=\tilde{\tau}(\delta)$, and partition $Y$ into consecutive
blocks $\widehat Y_1,\ldots,\widehat Y_{\widehat T}$ of length $L$,
except possibly the last. Fix $i<\widehat T$ and two admissible
histories as in Definition \ref{Marton_coupling}, which agree up
to block $i-1$ and may differ at block $i$. Let $s=iL$ denote
the end of block $i$. Although the two histories may induce
different conditional distributions of $S_s$, the Markov
property ensures that, given $S_s$, the future is independent
of the past losses. Hence, we may couple the two conditional
distributions of $S_s$ and then their Markov continuations,
thereby preserving the conditional future laws required by
Definition \ref{Marton_coupling}.

We construct the coupling at the successive boundaries
$\{s+kL\}_{k\ge1}$. At each boundary, conditional on the current
pair of states, we maximally couple their $L$-step transition
distributions. By the definition of $\tilde{\tau}(\delta)$,
the resulting endpoints differ with probability at most
$\delta$, uniformly over the current states. We then sample
the intervening paths from the corresponding conditional Markov
bridge distributions, preserving the trajectory law of each
copy. Once the states coincide, we use a common continuation
so that they remain equal thereafter. Thus, the conditional
probability of disagreement at the next boundary is at most
$\delta$ whenever the current states differ. When they coincide,
their transition laws are identical, and we couple the subsequent
transitions identically so that the two copies remain equal.
Consequently,
\[
\mathbb P\!\left(
S_{s+(r+1)L}\neq S'_{s+(r+1)L}
\mid S_{s+rL},S'_{s+rL}
\right)
\le
\delta\,
\mathbf 1_{\{S_{s+rL}\neq S'_{s+rL}\}}.
\]
Taking expectations and using the tower property yields
\[
\mathbb P(S_{s+(r+1)L}\neq S'_{s+(r+1)L})
\le
\delta\,
\mathbb P(S_{s+rL}\neq S'_{s+rL}).
\]
Since the initial disagreement probability is at most one,
induction gives
\begin{equation}
\mathbb P(S_{s+rL}\neq S'_{s+rL})\le\delta^r.
\label{eq:window-state-coupling}
\end{equation}

Now consider block $i+k$, where $k\ge2$. It begins at time
$a=s+(k-1)L+1$, and the state immediately preceding it is
\[
S_{a-1}
=(Z_{a-1},\ldots,Z_{a-\rho}).
\]
This state contains the complete context required to generate
the first loss in the block. Therefore, if the two states
$S_{a-1}$ and $S'_{a-1}$ coincide, their common continuation
makes every loss in block $i+k$ identical. It follows from
\eqref{eq:window-state-coupling} that
\[
\mathbb{P}(\widehat Y_{i+k}\neq\widehat Y'_{i+k})
\le
\mathbb{P}(S_{s+(k-1)L}\neq S'_{s+(k-1)L})
\le\delta^{k-1}.
\]
Since this bound is uniform over the admissible conditional
histories, Definition \ref{Marton_coupling} yields
\[
\Gamma_{ii}=1,\qquad
\Gamma_{i,i+1}\le1,\qquad
\Gamma_{i,i+k}\le\delta^{k-1},\quad k\ge2.
\]
In particular, the row and column sums are bounded by
\[
1+1+\sum_{r=1}^{\infty}\delta^r
=\frac{2-\delta}{1-\delta},
\]
and thus
\[
\|\Gamma\|_2
\le\sqrt{\|\Gamma\|_1\|\Gamma\|_\infty}
\le\frac{2-\delta}{1-\delta}.
\]

Finally, by the definition of $C(c)$ in Lemma
\ref{Mdiarmid_type_lemma} and the fact that each block has
length at most $L$, the Cauchy--Schwarz inequality yields
\[
\|C(c)\|_2^2
=
\sum_{j=1}^{\widehat T}
\left(\sum_{t\in I(\widehat Y_j)}c_t\right)^2
\le L\|c\|_2^2.
\]
Substituting the preceding two bounds into Lemma
\ref{Mdiarmid_type_lemma} gives
\[
\log\mathbb E\left[
e^{\lambda(f(Y)-\mathbb Ef(Y))}\right]
\le
\frac{\lambda^2\tilde{\tau}(\delta)}{8}
\left(\frac{2-\delta}{1-\delta}\right)^2
\|c\|_2^2.
\]
Taking the infimum over $\delta\in(0,1)$ yields
\[
\log\mathbb E\left[
e^{\lambda(f(Y)-\mathbb Ef(Y))}\right]
\le
\frac{\lambda^2\tilde{\tau}_{\min}\|c\|_2^2}{8},
\]
which proves \eqref{Concentration_under_condition_A}.
\end{proof}



\subsection{Proof of Lemma \ref{max_probability_conditional_markov}}\label{Proof_max_probability_conditional_markov}

\begin{align}
    \log\frac{1}{p_{w}(z_{t}|z_{t-1},\cdots,z_{t-\rho})}&=\log\frac{1}{\big<\sigma\left(g_w(\mathcal{E}^{\odot\rho}(\mathbf{z}))\right),\mathcal{E}(z_{t})\big>}\nonumber\\
    &\leq\max_{j\in[d]}\log\left(\frac{\sum_{j\in[d]}\exp({g_{w,j}})}{\exp(g^{(j)}_w)\cdot\mathcal{E}_j(z_{t})}\right)\label{worst_case_prob_bound_3}\\
    &\leq \log\left(\frac{de^{B}}{e^{-B}}\right)\label{worst_case_prob_bound_2}\\
    &=2B+\log d\nonumber
\end{align}
where
\begin{itemize}
    \item Equation~\eqref{worst_case_prob_bound_3} follows from the definitions in \eqref{conditional_distribution_def} and \eqref{soft_max_def}.
    
    \item Equation~\eqref{worst_case_prob_bound_2} follows from Assumption~\eqref{subsubsec:bounded_output}, which implies that $|{g_{w,j}}|\leq\|g_w\|\leq B$ for any $j\in[d]$.
\end{itemize}



\subsection{TinyStories: Experimental and Implementation Details}
\label{app:tinystories_details}

This section provides additional implementation details for the
TinyStories experiment reported in
Section~\ref{subsec:tinystories}.

\paragraph{Dataset and tokenization.}
We use the training split of the TinyStories corpus and retain
$8\%$ of the stories for our experiments. Before tokenization,
newline characters are removed from each story. We train a byte-level
BPE tokenizer from scratch on the selected corpus with vocabulary size $d=8000$.
The tokenizer contains the special tokens
\texttt{<pad>}, \texttt{<bos>}, \texttt{<eos>}, and
\texttt{<unk>}. A beginning- and end-of-sequence template,
\texttt{<bos>} $\cdots$ \texttt{<eos>}, is applied to each
story at encoding time.

The selected stories are tokenized once and concatenated into a single
token stream that is reused for all context lengths. The resulting
stream is split contiguously, with the first $10\%$ of the tokens
reserved as a held-out evaluation set and the remaining $90\%$ used
for training. In the source code, this held-out subset is denoted as
the ``validation'' set; throughout the paper and in the reported
figures, we refer to the loss evaluated on exactly this same subset
as the \emph{test loss}. Thus, no additional data partition is
introduced by this difference in terminology.

For a context length $\rho$, training examples are constructed from
sliding windows of $\rho+1$ consecutive tokens. Specifically, if
$(x_i,\ldots,x_{i+\rho})$ denotes such a window, the input sequence
and the corresponding next-token targets are $(x_i,\ldots,x_{i+\rho-1})$
and $(x_{i+1},\ldots,x_{i+\rho})$, respectively. Hence, the same
tokenized corpus is used for every value of $\rho$, and only the window
length changes.

\paragraph{Model architecture.}
We employ a GPT-style decoder-only Transformer with causal
self-attention. The model has $N_{\mathrm{layer}}=6$, $d_{\mathrm{model}}=512$,
and $H=8$, where $N_{\mathrm{layer}}$ is the number of Transformer blocks and
$H$ is the number of attention heads. The feed-forward module uses an
expansion ratio of $4$, and dropout is set to $0.1$.

Each block uses a pre-normalized residual architecture with RMSNorm,
causal scaled dot-product self-attention, and a SwiGLU feed-forward
module. Causality is enforced directly in the scaled dot-product
attention operation, preventing each token position from attending
to future tokens. Learned token and positional embeddings are used,
and the input token-embedding matrix is tied to the output language
model head. Residual projection weights are initialized with a
depth-dependent scaling factor. The architecture contains
approximately $25$ million trainable parameters; the only
context-dependent contribution to the parameter count comes from the
learned positional embeddings.

\paragraph{Centralized context-length sweep.}
The experiment is fully centralized; we train an
independent model from scratch for every context length
$\rho\in\{2,4,8,16,32,64,128\}$.

The tokenized data and all optimization hyperparameters are shared
across the different values of $\rho$. In particular, the total
number of gradient updates is held fixed, so that changes in the
reported losses cannot be attributed simply to a larger optimization
budget for longer contexts.

\paragraph{Optimization.}
Every model is trained for exactly $10000$
gradient steps using AdamW with learning rate $\eta=3\times10^{-4}$,
weight decay $0.01$, and mini-batch size $64$. The global gradient
norm is clipped at $1.0$. The optimization objective is the standard
token-level cross-entropy next-token prediction loss. When a GPU is
available, mixed-precision computation uses \texttt{bfloat16}.

\paragraph{Evaluation.}
During training, the training and held-out losses are evaluated
periodically for diagnostic purposes. At each evaluation point, each
loss is estimated using at most $100$ mini-batches. The values
reported in Fig.~\ref{fig:tinystories_context} correspond to the
losses measured at the final optimization step.

For each realization and context length, we define
\begin{equation}
    \widehat{\operatorname{gen}}(\rho)
    =
    \widehat{\mathcal L}_{\mathrm{test}}(\rho)
    -
    \widehat{\mathcal L}_{\mathrm{train}}(\rho).\nonumber
\end{equation}
Here, $\widehat{\mathcal L}_{\mathrm{test}}$ is precisely the loss
called \texttt{val\_loss} in the implementation; the two quantities
therefore refer to the same held-out data and differ only in
terminology. The figures report the mean across independent
realizations, while the shaded bands correspond to one standard
deviation.



\subsection{ETTh2: Experimental and Implementation Details}
\label{app:etth2_details}

This section provides additional details for the ETTh2 experiments reported in
Section~\ref{subsec:etth2_main}. We first describe the discrete
next-token prediction experiment used to study context length and
generalization, followed by an auxiliary analysis of the characteristic
predictive-memory scale of ETTh2.

\subsubsection{Context-Length Generalization Experiment}
\label{app:etth2_generalization_details}

\paragraph{Data representation and prediction task}
We consider the univariate \texttt{OT} series of ETTh2 and uniformly
quantize its range into $d=20$ bins, each treated as a discrete token.
For context length $\rho$, prediction examples have the form
\begin{equation}
    \bigl(X_{t-\rho},\ldots,X_{t-1}\bigr)
    \longmapsto X_t.
    \nonumber
\end{equation}
We consider
$\rho\in\{2,4,8,16,32,64,128\}$.
The resulting next-token windows are divided into contiguous blocks that
are randomly assigned to the training or test partition. Whenever a
training block is adjacent to a test block, a buffer of width $\rho$ is
removed from the training side to prevent overlap between their receptive
fields.

\paragraph{Model and optimization}
For each $\rho$, we train from scratch a two-layer causal self-attention
predictor with embedding dimension $d_{\mathrm{model}}=64$, $4$ attention
heads, feed-forward dimension $4d_{\mathrm{model}}$, GELU activation,
pre-normalization, and dropout $0.1$. The model uses learned token and
positional embeddings together with a causal attention mask. After the Transformer
layers and a final LayerNorm, the hidden representation at the last
context position is mapped to logits over the $20$ output symbols.

Training uses cross-entropy loss and Adam with learning rate
$\eta=10^{-2}$ and mini-batch size $64$. Every model is trained for
exactly $15{,}000$ gradient updates, independently of $\rho$, with no
early stopping. For each seed, an inner validation subset of the training
data is used exclusively for checkpoint selection: the validation loss is
evaluated every $100$ updates, and the checkpoint attaining the minimum
validation loss over the full training trajectory is retained. The
selected model is then evaluated on the complete training set and on the
separate test set.

The experiment is repeated with $5$ independent random seeds. For each
context length, we report the mean training loss, test loss, and empirical
generalization gap
\begin{equation}
    \widehat{\operatorname{gen}}(\rho)
    =
    \widehat{\mathcal L}_{\mathrm{test}}(\rho)
    -
    \widehat{\mathcal L}_{\mathrm{train}}(\rho).
    \nonumber
\end{equation}
The shaded regions in Fig.~\ref{fig:etth2_context_main} represent
$95\%$ confidence intervals for the corresponding means.

As discussed in Section~\ref{subsec:etth2_main}, the training loss
decreases substantially with context length. The test loss initially
improves and reaches its lowest observed value around $\rho=32$, but does
not improve for larger contexts, while the training loss remains
substantially smaller. This growing train--test discrepancy motivates the
auxiliary predictive-memory analysis below.

\subsubsection{Empirical Estimation of the Effective Predictive Memory}
\label{app:etth2_memory}

The purpose of this auxiliary experiment is to provide an empirical
interpretation of the change in test-loss behavior observed in
Section~\ref{subsec:etth2_main}. One possible explanation is that most of
the predictively useful temporal dependence accessible to the model under
consideration is contained within a characteristic range of history, so that
making increasingly older observations available eventually provides
little measurable out-of-sample benefit.

We therefore study how predictive performance changes as progressively
longer histories are made available. This analysis does not attempt to
identify an exact finite-order Markov representation of ETTh2, nor does
it test whether individual lags beyond a particular point have zero
conditional information. Rather, it estimates an
\emph{effective predictive-memory scale} under the considered model and
evaluation protocol and is used only as a plausible empirical
interpretation of the behavior observed in the main experiment.

Let $\mathbf{X}_t\in\mathbb{R}^{d_x}$, with $d_x=7$, denote the complete
multivariate ETTh2 observation at time $t$. The analysis is motivated by
the finite-history predictive approximation
\begin{equation}
P\!\left(
    \mathbf{X}_t
    \mid
    \mathbf{X}_{t-1},\mathbf{X}_{t-2},\ldots
\right)
\simeq
P\!\left(
    \mathbf{X}_t
    \mid
    \mathbf{X}_{t-1},\ldots,\mathbf{X}_{t-\rho}
\right).
\label{eq:etth2_effective_memory}
\nonumber
\end{equation}
We use this relation only as an operational motivation, rather than as an
exact conditional-independence statement, and ask whether increasing
$\rho$ yields statistically supported gains in out-of-sample prediction.

\paragraph{Predictive estimator and evaluation protocol}
For each candidate $\rho$, a nonlinear Transformer predicts the current
multivariate observation from the preceding $\rho$ observations,
\begin{equation}
    \mathbf{X}_{t-\rho:t-1}
    =
    \left(
        \mathbf{X}_{t-\rho},\ldots,\mathbf{X}_{t-1}
    \right)
    \longmapsto
    \widehat{\mathbf{X}}_t .
    \nonumber
\end{equation}
It parameterizes the diagonal Gaussian conditional distribution
\begin{equation}
q_{\rho}\!\left(
    \mathbf{x}_t
    \mid
    \mathbf{x}_{t-\rho:t-1}
\right)
=
\prod_{j=1}^{d_x}
\mathcal{N}\!\left(
    x_{t,j};
    \mu_{\rho,j}\!\left(\mathbf{x}_{t-\rho:t-1}\right),
    \sigma_{\rho,j}^{2}\!\left(\mathbf{x}_{t-\rho:t-1}\right)
\right),
\nonumber
\end{equation}
where both conditional means and variances are learned from the data.
The primary criterion is the out-of-sample negative log-likelihood
\begin{equation}
\widehat{\mathcal L}_{\mathrm{NLL}}(\rho)
=
-\frac{1}{|\mathcal T_{\mathrm{test}}|}
\sum_{t\in\mathcal T_{\mathrm{test}}}
\log q_{\rho}\!\left(
    \mathbf X_t
    \mid
    \mathbf X_{t-\rho:t-1}
\right),
\nonumber
\end{equation}
where smaller values indicate better predictive performance. Normalized
multivariate MSE is also monitored as a secondary criterion.

We use expanding-window temporal cross-validation with four folds and
three random seeds, without random temporal shuffling. Normalization
parameters are estimated exclusively from the training portion of each
fold, and all context lengths are evaluated on identical target
timestamps, thereby enabling paired comparison of their out-of-sample losses.

\paragraph{Coarse and refined context searches}
We first perform the broad search
\begin{equation}
\mathcal R_{\mathrm{coarse}}
=
\{1,2,3,6,12,18,24,30,36,48,72,96,120,144,168\}
\quad\text{hours}.
\nonumber
\end{equation}
The minimum mean NLL is attained at
$\rho_{\mathrm{coarse}}^\star=24$ hours, with
$\widehat{\mathcal L}_{\mathrm{NLL}}(24)\simeq-0.2345$.
Increasing the history to substantially larger values, including
$48$, $72$, $96$, $120$, $144$, and $168$ hours, does not yield a lower
mean NLL, thereby localizing the dominant predictive-memory scale around one day.

To verify that this optimum is not an artifact of the coarse
discretization, we repeat the experiment with two-hour resolution:
\begin{equation}
\mathcal R_{\mathrm{fine}}
=
\{20,22,24,26,28,30,32,34,36\}
\quad\text{hours}.
\nonumber
\end{equation}
All other components of the protocol are unchanged.

\begin{figure*}[t]
     \centering
    \begin{minipage}[t]{0.35\textwidth}
        \centering
        \includegraphics[width=\linewidth]{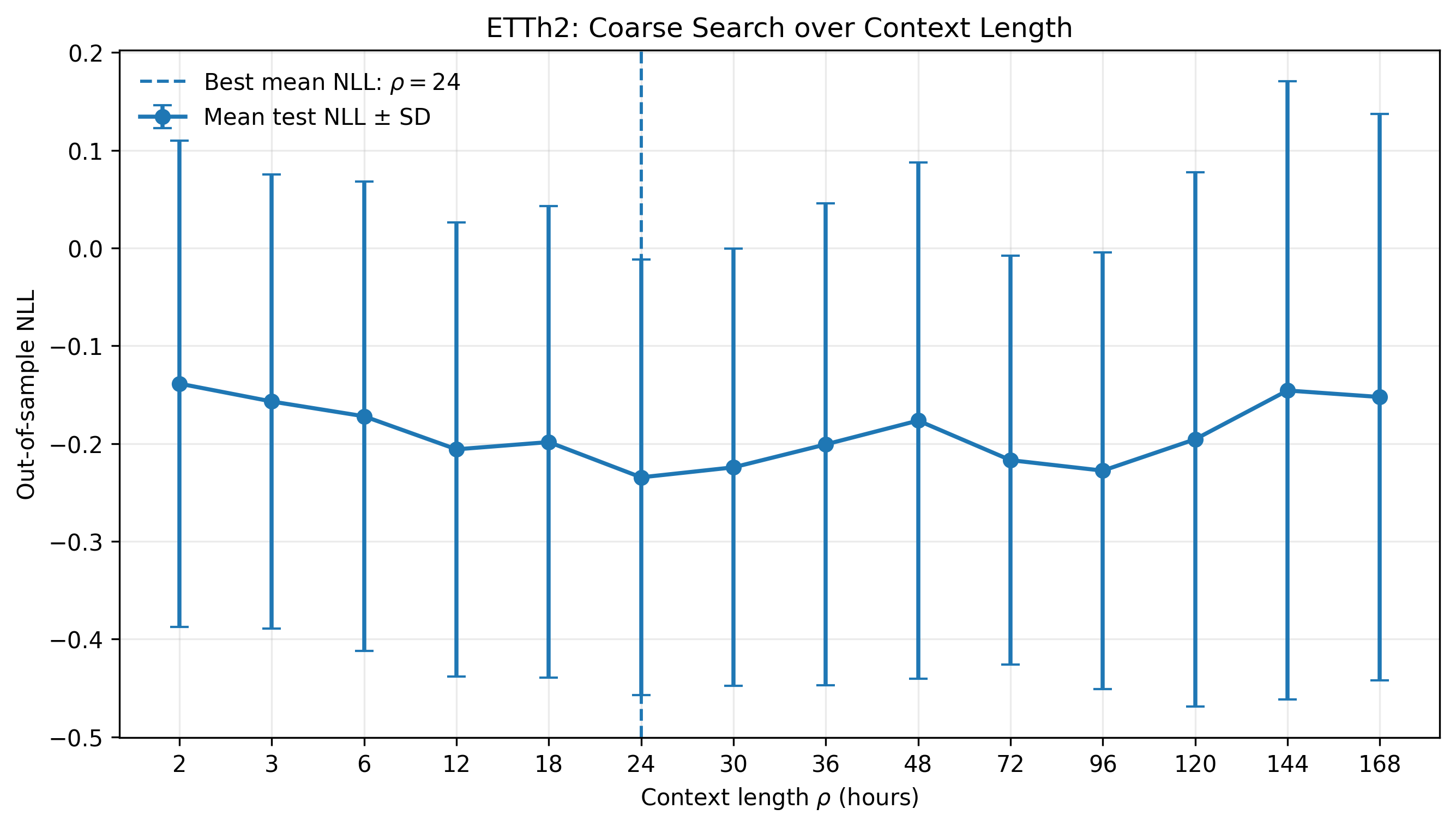}
        \small (a) Coarse search over context length
    \end{minipage}
    \hspace{0.1\textwidth}
    \begin{minipage}[t]{0.35\textwidth}
        \centering
        \includegraphics[width=\linewidth]{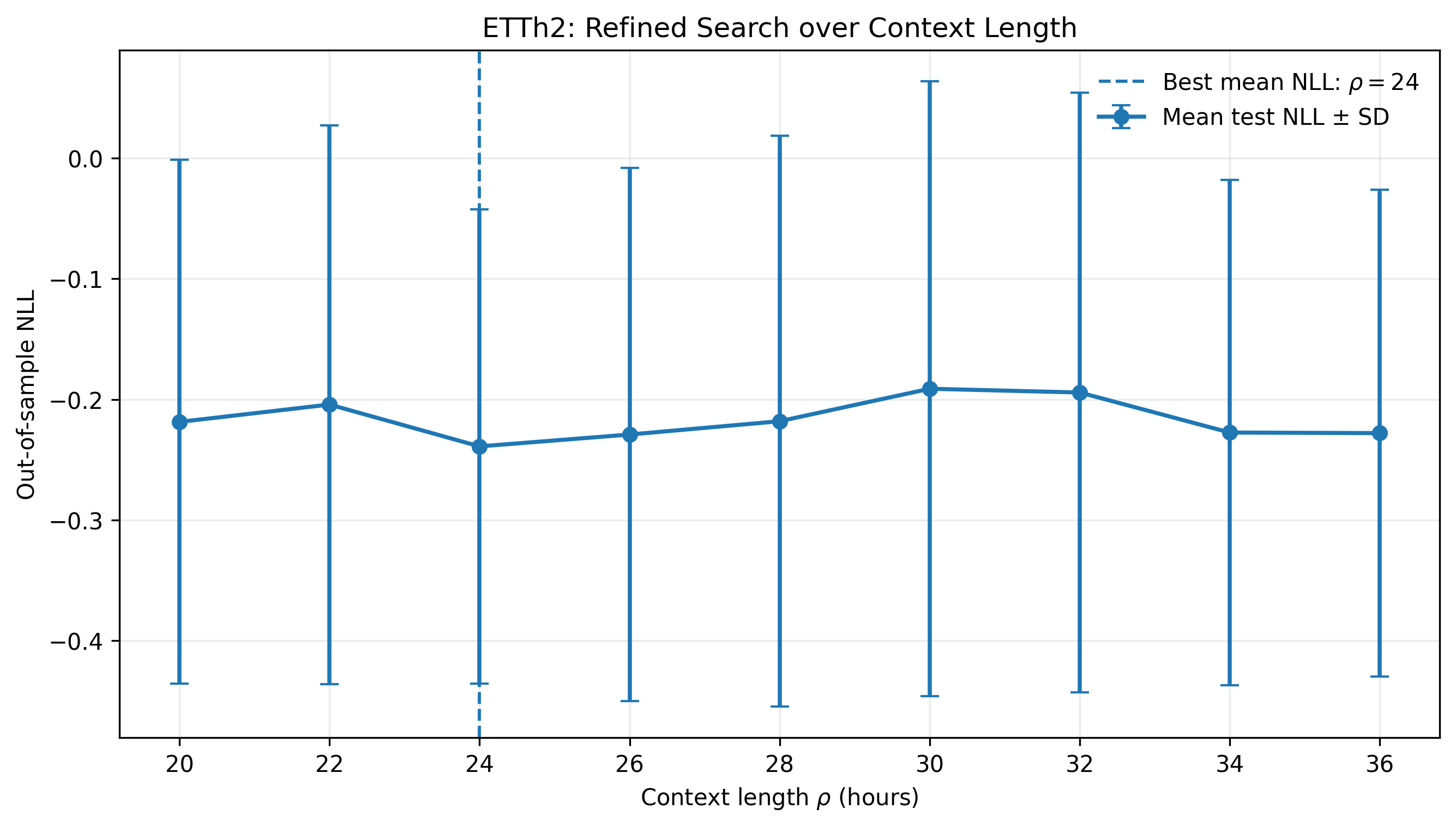}
        \small (b) Refined search around the coarse optimum
    \end{minipage}
    \caption{\textbf{Empirical estimation of the effective
    predictive-memory scale of ETTh2.}
    Panel~(a) reports the coarse search over context lengths ranging
    from $1$ to $168$ hours. Panel~(b) reports the refined search over
    $20$--$36$ hours with a two-hour resolution. Both experiments use
    four expanding-window temporal cross-validation folds and three
    random seeds, with error bars representing one standard deviation.
    The minimum mean out-of-sample NLL is attained at $\rho=24$ hours in
    both searches. The refined comparison finds no statistically supported
    predictive improvement from increasing the context beyond this value,
    while several moderately larger contexts, including $\rho=34$, remain
    close to the optimum. The results therefore indicate a dominant
    predictive-memory scale around one day together with a broader
    near-optimal predictive plateau, rather than a sharp empirical memory
    cutoff.}
    \label{fig:etth2_markov_order}
\end{figure*}

\paragraph{Statistical comparison and interpretation}
For $\rho_1<\rho_2$, let $\ell_t^{(\rho)}$ denote the out-of-sample NLL at
a common target timestamp $t$, and define
\begin{equation}
D_t(\rho_1,\rho_2)
=
\ell_t^{(\rho_1)}
-
\ell_t^{(\rho_2)}.
\nonumber
\end{equation}
We use a fold-aware moving-block bootstrap with $24$-hour blocks to
account for temporal dependence. Resampling is performed separately
within each temporal fold so that blocks never cross fold boundaries,
and pairwise comparisons are corrected using Holm's family-wise
error-rate procedure at $\alpha=0.05$.

Refining the grid does not change the location of the minimum:
\[
\rho_{\mathrm{fine}}^\star=24~\text{hours},
\qquad
\widehat{\mathcal L}_{\mathrm{NLL}}(24)\simeq-0.2390.
\]
Increasing the context beyond $24$ hours produces no statistically
supported predictive improvement under the paired moving-block bootstrap
with Holm correction. Nevertheless, the predictive loss remains
relatively flat for several nearby and moderately larger contexts. In
particular,
$\widehat{\mathcal L}_{\mathrm{NLL}}(34)\simeq-0.2275$ remains close to
the minimum, and the normalized multivariate MSE exhibits a similar
plateau.

These results suggest a broad predictive plateau rather than an abrupt
memory cutoff. A $34$-hour context can therefore be viewed as a viable
conservative history length under this protocol, although its
near-minimum NLL should not be interpreted as evidence that the
additional lags from $25$ to $34$ hours individually carry significant
predictive information.

The agreement between the coarse and refined searches motivates the empirical
estimate $\rho_{\mathrm{eff}}\approx24~\text{hours},$
which we interpret as a \emph{dominant predictive-memory scale}, not an
exact finite Markov order. The relatively flat performance at moderately
larger contexts further indicates that neither $24$ nor $34$ constitutes
a sharp boundary beyond which the underlying process contains no
temporal dependence; rather, a broader range around the dominant scale
retains near-optimal predictive performance.

\paragraph{Relation to the context-length generalization experiment}
The auxiliary analysis provides a possible explanation for the transition
observed in Section~\ref{subsec:etth2_main}. There, the discrete
next-token test cross-entropy reaches its lowest observed value around
$\rho=32$ and does not improve for larger contexts, while the training
loss remains substantially smaller.

Independently, the auxiliary analysis places the minimum mean predictive
NLL at $\rho=24$ hours, finds no statistically supported improvement from
extending the history beyond this value, and observes near-optimal
performance for some moderately larger contexts, including $\rho=34$.
The result is therefore more naturally interpreted as a predictive
plateau than as a sharp cutoff in temporal dependence.

Consequently, once the context enters this regime, making additional past
observations available may increase the statistical and representational
burden of the model without producing a corresponding out-of-sample gain.
This provides a plausible empirical interpretation of the transition
observed around $\rho\simeq32$ in the discrete ETTh2 experiment, which
falls within the broader predictive-memory regime identified by the
auxiliary analysis.



\bibliographystyle{IEEEtran}
\bibliography{bib}

@article{sefidgaran2024minimum,
  title={Minimum description length and generalization guarantees for representation learning},
  author={Sefidgaran, Milad and Zaidi, Abdellatif and Krasnowski, Piotr},
  journal={Advances in Neural Information Processing Systems},
  volume={36},
  year={2024}
}

@article{sefidgaran2022rate,
  title={Rate-distortion theoretic bounds on generalization error for distributed learning},
  author={Sefidgaran, Milad and Chor, Romain and Zaidi, Abdellatif},
  journal={Advances in Neural Information Processing Systems},
  volume={35},
  pages={19687--19702},
  year={2022}
}

@inproceedings{steinke2020reasoning,
  title={Reasoning about generalization via conditional mutual information},
  author={Steinke, Thomas and Zakynthinou, Lydia},
  booktitle={Conference on Learning Theory},
  pages={3437--3452},
  year={2020},
  organization={PMLR}
}

@article{xu2017information,
  title={Information-theoretic analysis of generalization capability of learning algorithms},
  author={Xu, Aolin and Raginsky, Maxim},
  journal={Advances in neural information processing systems},
  volume={30},
  year={2017}
}

@inproceedings{sefidgaran2024lessons,
  title={Lessons from Generalization Error Analysis of Federated Learning: You May Communicate Less Often!},
  author={Sefidgaran, Milad and Chor, Romain and Zaidi, Abdellatif and Wan, Yijun},
  booktitle={Forty-first International Conference on Machine Learning},
  year={2024}
}

@inproceedings{Sefidgaran2022,
  title={Rate-distortion theoretic generalization bounds for stochastic learning algorithms},
  author={Sefidgaran, Milad and Gohari, Amin and Richard, Gael and Simsekli, Umut},
  booktitle={Conference on Learning Theory},
  pages={4416--4463},
  year={2022},
  organization={PMLR}
}

@InProceedings{gronlund2020,
	title = 	 {Near-Tight Margin-Based Generalization Bounds for Support Vector Machines},
	author =       {Gr{\o}nlund, Allan and Kamma, Lior and Larsen, Kasper Green},
	booktitle = 	 {Proceedings of the 37th International Conference on Machine Learning},
	pages = 	 {3779--3788},
	year = 	 {2020},
	editor = 	 {III, Hal Daumé and Singh, Aarti},
	volume = 	 {119},
	series = 	 {Proceedings of Machine Learning Research},
	month = 	 {13--18 Jul},
	publisher =    {PMLR},
}

@inproceedings{gronlund2020near,
  title={Near-tight margin-based generalization bounds for support vector machines},
  author={Gr{\o}nlund, Allan and Kamma, Lior and Larsen, Kasper Green},
  booktitle={International Conference on Machine Learning},
  pages={3779--3788},
  year={2020},
  organization={PMLR}
}

@article{paulin2012concentration,
  title={Concentration inequalities for Markov chains by Marton couplings and spectral methods},
  author={Paulin, Daniel},
  journal={arXiv preprint arXiv:1212.2015},
  year={2012}
}

@inproceedings{yuksel2025sample,
  title={On the sample complexity of next-token prediction},
  author={Y{\"u}ksel, O{\u{g}}uz Kaan and Flammarion, Nicolas},
  booktitle={The 28th International Conference on Artificial Intelligence and Statistics},
  year={2025}
}

@inproceedings{arora2018stronger,
  title={Stronger generalization bounds for deep nets via a compression approach},
  author={Arora, Sanjeev and Ge, Rong and Neyshabur, Behnam and Zhang, Yi},
  booktitle={International conference on machine learning},
  pages={254--263},
  year={2018},
  organization={PMLR}
}

@book{levin2017markov,
  author    = {David A. Levin and Yuval Peres},
  title     = {Markov Chains and Mixing Times},
  publisher = {American Mathematical Society},
  year      = {2017},
  edition   = {2nd}
}

@inproceedings{hao2018learning,
  author    = {Yi Hao and Alon Orlitsky and Venkatadheeraj Pichapati},
  title     = {On Learning Markov Chains},
  booktitle = {Advances in Neural Information Processing Systems (NeurIPS)},
  volume    = {31},
  year      = {2018}
}

@inproceedings{wolfer2019minimax,
  author    = {Geoffrey Wolfer and Aryeh Kontorovich},
  title     = {Minimax Learning of Ergodic Markov Chains},
  booktitle = {Algorithmic Learning Theory (ALT)},
  year      = {2019},
  pages     = {904--930}
}

@inproceedings{bengio2000,
  author    = {Yoshua Bengio and R{\'e}jean Ducharme and Pascal Vincent},
  title     = {A Neural Probabilistic Language Model},
  booktitle = {Advances in Neural Information Processing Systems (NeurIPS)},
  volume    = {13},
  year      = {2000}
}

@inproceedings{vaswani2017,
  author    = {Ashish Vaswani and Noam Shazeer and Niki Parmar and Jakob Uszkoreit and Llion Jones and Aidan N. Gomez and {\L}ukasz Kaiser and Illia Polosukhin},
  title     = {Attention Is All You Need},
  booktitle = {Advances in Neural Information Processing Systems (NeurIPS)},
  volume    = {30},
  year      = {2017}
}

@book{berger1971rate,
  author    = {Thomas M. Berger},
  title     = {Rate Distortion Theory: A Mathematical Basis for Data Compression},
  publisher = {Prentice-Hall},
  year      = {1971}
}

@inproceedings{lotfi2024unlocking,
  author    = {Sanae Lotfi and Yilun Kuang and Marc Finzi and Brandon Amos and Micah Goldblum and Andrew G. Wilson},
  title     = {Unlocking Tokens as Data Points for Generalization Bounds on Large Language Models},
  booktitle = {Advances in Neural Information Processing Systems (NeurIPS)},
  volume    = {37},
  year      = {2024}
}

@article{raginsky2017,
  author  = {Maxim Raginsky and Rachit Talwar and Aolin Xu},
  title   = {Information-Theoretic Lower Bounds for Generalization Error},
  journal = {IEEE Transactions on Information Theory},
  volume  = {63},
  number  = {12},
  pages   = {7676--7688},
  year    = {2017}
}

@article{bu2020tightening,
  author  = {Zhenqing Bu and Shaofeng Zou and Venugopal V. Veeravalli},
  title   = {Tightening Mutual Information-Based Bounds on Generalization Error},
  journal = {IEEE Journal on Selected Areas in Information Theory},
  volume  = {1},
  number  = {1},
  pages   = {121--130},
  year    = {2020}
}

@inproceedings{sefidgaran2022distributed,
  author    = {Milad Sefidgaran and Ronen Chor and Abdellatif Zaidi},
  title     = {Rate-Distortion Theoretic Bounds on Generalization Error for Distributed Learning},
  booktitle = {Advances in Neural Information Processing Systems (NeurIPS)},
  volume    = {35},
  year      = {2022}
}

@article{sefidgaran2024variable,
  author  = {Milad Sefidgaran and Abdellatif Zaidi},
  title   = {Data-Dependent Generalization Bounds via Variable-Size Compressibility},
  journal = {IEEE Transactions on Information Theory},
  volume  = {70},
  number  = {2},
  pages   = {1065--1085},
  year    = {2024}
}

@inproceedings{russo2016controlling,
  title={Controlling bias in adaptive data analysis using information theory},
  author={Russo, Daniel and Zou, James},
  booktitle={Artificial Intelligence and Statistics},
  pages={1232--1240},
  year={2016},
  organization={PMLR}
}

@article{malach2023auto,
  title={Auto-regressive next-token predictors are universal learners},
  author={Malach, Eran},
  journal={arXiv preprint arXiv:2309.06979},
  year={2023},
  url={https://arxiv.org/abs/2309.06979}
}

@inproceedings{NIPS2017_b22b257a,
 author = {Bartlett, Peter and Foster, Dylan J and Telgarsky, Matus J},
 booktitle = {Advances in Neural Information Processing Systems},
 editor = {I. Guyon and U. Von Luxburg and S. Bengio and H. Wallach and R. Fergus and S. Vishwanathan and R. Garnett},
 pages = {},
 publisher = {Curran Associates, Inc.},
 title = {Spectrally-normalized margin bounds for neural networks},
 url = {https://proceedings.neurips.cc/paper_files/paper/2017/file/b22b257ad0519d4500539da3c8bcf4dd-Paper.pdf},
 volume = {30},
 year = {2017}
}

@article{JMLR:v22:19-479,
  author  = {Jianqing Fan and Bai Jiang and Qiang Sun},
  title   = {Hoeffding's Inequality for General Markov Chains and Its Applications to Statistical Learning},
  journal = {Journal of Machine Learning Research},
  year    = {2021},
  volume  = {22},
  number  = {139},
  pages   = {1--35},
  url     = {http://jmlr.org/papers/v22/19-479.html}
}

@article{jerison2013general,
  title={General mixing time bounds for finite Markov chains via the absolute spectral gap},
  author={Jerison, Daniel},
  journal={arXiv preprint arXiv:1310.8021},
  year={2013}
}

@INPROCEEDINGS{11653692,
  author={Kavian, Masoud and Zaidi, Abdellatif and Sefidgaran, Milad},
  booktitle={2026 IEEE International Symposium on Information Theory (ISIT)}, 
  title={Generalization Analysis of Next-Token Prediction Learning Algorithms}, 
  year={2026},
  volume={},
  number={},
  pages={1-6},
  doi={10.1109/ISIT62367.2026.11653692}}

@inproceedings{negrea2019information,
  title={Information-Theoretic Generalization Bounds for SGLD via Data-Dependent Estimates},
  author={Negrea, Jeffrey and Haghifam, Mahdi and Dziugaite, Gintare Karolina and Khisti, Ashish and Roy, Daniel M.},
  booktitle={Advances in Neural Information Processing Systems},
  volume={32},
  year={2019}
}

@article{hellstrom2020generalization,
  title={Generalization Bounds via Information Density and Conditional Information Density},
  author={Hellstr{\"o}m, Fredrik and Durisi, Giuseppe},
  journal={IEEE Journal on Selected Areas in Information Theory},
  volume={1},
  number={3},
  pages={824--839},
  year={2020},
  doi={10.1109/JSAIT.2020.3040992}
}

@inproceedings{harutyunyan2021blackbox,
  title={Information-Theoretic Generalization Bounds for Black-Box Learning Algorithms},
  author={Harutyunyan, Hrayr and Raginsky, Maxim and Ver Steeg, Greg and Galstyan, Aram},
  booktitle={Advances in Neural Information Processing Systems},
  volume={34},
  pages={24670--24682},
  year={2021}
}

@inproceedings{barsbey2021heavy,
  title={Heavy Tails in SGD and Compressibility of Overparametrized Neural Networks},
  author={Barsbey, Melih and Sefidgaran, Milad and Erdogdu, Murat A. and Richard, Ga{\"e}l and Simsekli, Umut},
  booktitle={Advances in Neural Information Processing Systems},
  volume={34},
  year={2021}
}

@article{marton1996bounding,
  title={Bounding $\bar d$-Distance by Informational Divergence: A Method to Prove Measure Concentration},
  author={Marton, Katalin},
  journal={The Annals of Probability},
  volume={24},
  number={2},
  pages={857--866},
  year={1996}
}

@inproceedings{mohri2008rademacher,
  title={Rademacher Complexity Bounds for Non-I.I.D. Processes},
  author={Mohri, Mehryar and Rostamizadeh, Afshin},
  booktitle={Advances in Neural Information Processing Systems},
  volume={21},
  pages={1097--1104},
  year={2008}
}

@article{kuznetsov2017generalization,
  title={Generalization Bounds for Non-Stationary Mixing Processes},
  author={Kuznetsov, Vitaly and Mohri, Mehryar},
  journal={Machine Learning},
  volume={106},
  number={1},
  pages={93--117},
  year={2017},
  doi={10.1007/s10994-016-5588-2}
}

@article{yu1994rates,
  title={Rates of Convergence for Empirical Processes of Stationary Mixing Sequences},
  author={Yu, Bin},
  journal={The Annals of Probability},
  volume={22},
  number={1},
  pages={94--116},
  year={1994},
  doi={10.1214/aop/1176988849}
}

@article{agarwal2013generalization,
  title={The Generalization Ability of Online Algorithms for Dependent Data},
  author={Agarwal, Alekh and Duchi, John C.},
  journal={IEEE Transactions on Information Theory},
  volume={59},
  number={1},
  pages={573--587},
  year={2013},
  doi={10.1109/TIT.2012.2212414}
}

@inproceedings{ralaivola2009chromatic,
  title={Chromatic PAC-Bayes Bounds for Non-IID Data},
  author={Ralaivola, Liva and Szafranski, Marie and Stempfel, Guillaume},
  booktitle={Proceedings of the Twelfth International Conference on Artificial Intelligence and Statistics},
  series={Proceedings of Machine Learning Research},
  volume={5},
  pages={416--423},
  year={2009},
  publisher={PMLR}
}

@article{kontorovich2008concentration,
  title={Concentration Inequalities for Dependent Random Variables via the Martingale Method},
  author={Kontorovich, Leonid and Ramanan, Kavita},
  journal={The Annals of Probability},
  volume={36},
  number={6},
  pages={2126--2158},
  year={2008},
  doi={10.1214/07-AOP384}
}

@article{samson2000concentration,
  title={Concentration of Measure Inequalities for Markov Chains and $\Phi$-Mixing Processes},
  author={Samson, Paul-Marie},
  journal={The Annals of Probability},
  volume={28},
  number={1},
  pages={416--461},
  year={2000},
  doi={10.1214/aop/1019160125}
}

@article{glynn2002hoeffding,
  title={Hoeffding's Inequality for Uniformly Ergodic Markov Chains},
  author={Glynn, Peter W. and Ormoneit, Dirk},
  journal={Statistics \& Probability Letters},
  volume={56},
  number={2},
  pages={143--146},
  year={2002},
  doi={10.1016/S0167-7152(01)00158-4}
}

@article{lezaud1998chernoff,
  title={Chernoff-Type Bound for Finite Markov Chains},
  author={Lezaud, Pascal},
  journal={The Annals of Applied Probability},
  volume={8},
  number={3},
  pages={849--867},
  year={1998}
}

@inproceedings{li2025generalizationntp,
  title={On the Generalization Ability of Next-Token-Prediction Pretraining},
  author={Li, Zhihao and Jiang, Xue and Liu, Liyuan and Zhang, Xuelin and Chen, Hong and Zheng, Feng},
  booktitle={Proceedings of the 42nd International Conference on Machine Learning},
  series={Proceedings of Machine Learning Research},
  volume={267},
  pages={34943--34975},
  year={2025},
  publisher={PMLR}
}

@inproceedings{bachmann2024pitfalls,
  title={The Pitfalls of Next-Token Prediction},
  author={Bachmann, Gregor and Nagarajan, Vaishnavh},
  booktitle={Proceedings of the 41st International Conference on Machine Learning},
  series={Proceedings of Machine Learning Research},
  volume={235},
  pages={2296--2318},
  year={2024},
  publisher={PMLR}
}

@article{sander2024universality,
  title={Towards Understanding the Universality of Transformers for Next-Token Prediction},
  author={Sander, Michael E. and Peyr{\'e}, Gabriel},
  journal={arXiv preprint arXiv:2410.03011},
  year={2024}
}

@inproceedings{yun2020transformers,
  title={Are Transformers Universal Approximators of Sequence-to-Sequence Functions?},
  author={Yun, Chulhee and Bhojanapalli, Srinadh and Rawat, Ankit Singh and Reddi, Sashank J. and Kumar, Sanjiv},
  booktitle={International Conference on Learning Representations},
  year={2020}
}

@article{likhosherstov2021expressive,
  title={On the Expressive Power of Self-Attention Matrices},
  author={Likhosherstov, Valerii and Choromanski, Krzysztof and Weller, Adrian},
  journal={arXiv preprint arXiv:2106.03764},
  year={2021}
}

@inproceedings{edelman2022inductive,
  title={Inductive Biases and Variable Creation in Self-Attention Mechanisms},
  author={Edelman, Benjamin L. and Goel, Surbhi and Kakade, Sham M. and Zhang, Cyril},
  booktitle={Proceedings of the 39th International Conference on Machine Learning},
  series={Proceedings of Machine Learning Research},
  volume={162},
  year={2022},
  publisher={PMLR}
}

@inproceedings{li2024mechanics,
  title={Mechanics of Next Token Prediction with Self-Attention},
  author={Li, Yingcong and Huang, Yixiao and Ildiz, Muhammed E. and Singh Rawat, Ankit and Oymak, Samet},
  booktitle={Proceedings of The 27th International Conference on Artificial Intelligence and Statistics},
  series={Proceedings of Machine Learning Research},
  volume={238},
  pages={685--693},
  year={2024},
  publisher={PMLR}
}

@article{madden2024capacity,
  title={Next-Token Prediction Capacity: General Upper Bounds and a Lower Bound for Transformers},
  author={Madden, Liam and Fox, Curtis and Thrampoulidis, Christos},
  journal={arXiv preprint arXiv:2405.13718},
  year={2024}
}

@inproceedings{golowich2025sparsity,
  title={The Role of Sparsity for Length Generalization in {LLM}s},
  author={Golowich, Noah and Jelassi, Samy and Brandfonbrener, David and Kakade, Sham M. and Malach, Eran},
  booktitle={Proceedings of the 42nd International Conference on Machine Learning},
  series={Proceedings of Machine Learning Research},
  volume={267},
  pages={19809--19840},
  year={2025},
  publisher={PMLR}
}

@article{eldan2023tinystories,
  title={TinyStories: How Small Can Language Models Be and Still Speak Coherent English?},
  author={Eldan, Ronen and Li, Yuanzhi},
  journal={arXiv preprint arXiv:2305.07759},
  year={2023}
}

@inproceedings{zhou2021informer,
  title={Informer: Beyond Efficient Transformer for Long Sequence Time-Series Forecasting},
  author={Zhou, Haoyi and Zhang, Shanghang and Peng, Jieqi and Zhang, Shuai and Li, Jianxin and Xiong, Hui and Zhang, Wancai},
  booktitle={Proceedings of the AAAI Conference on Artificial Intelligence},
  volume={35},
  number={12},
  pages={11106--11115},
  year={2021},
  doi={10.1609/aaai.v35i12.17325}
}

@inproceedings{wu2021autoformer,
  title={Autoformer: Decomposition Transformers with Auto-Correlation for Long-Term Series Forecasting},
  author={Wu, Haixu and Xu, Jiehui and Wang, Jianmin and Long, Mingsheng},
  booktitle={Advances in Neural Information Processing Systems},
  volume={34},
  year={2021}
}

@inproceedings{zhou2022fedformer,
  title={{FEDformer}: Frequency Enhanced Decomposed Transformer for Long-Term Series Forecasting},
  author={Zhou, Tian and Ma, Ziqing and Wen, Qingsong and Wang, Xue and Sun, Liang and Jin, Rong},
  booktitle={Proceedings of the 39th International Conference on Machine Learning},
  series={Proceedings of Machine Learning Research},
  volume={162},
  pages={27268--27286},
  year={2022},
  publisher={PMLR}
}

@inproceedings{nie2023patchtst,
  title={A Time Series Is Worth 64 Words: Long-Term Forecasting with Transformers},
  author={Nie, Yuqi and Nguyen, Nam H. and Sinthong, Phanwadee and Kalagnanam, Jayant},
  booktitle={International Conference on Learning Representations},
  year={2023}
}

@inproceedings{wu2023timesnet,
  title={TimesNet: Temporal 2D-Variation Modeling for General Time Series Analysis},
  author={Wu, Haixu and Hu, Tengge and Liu, Yong and Zhou, Hang and Wang, Jianmin and Long, Mingsheng},
  booktitle={International Conference on Learning Representations},
  year={2023}
}

@inproceedings{zhang2023crossformer,
  title={Crossformer: Transformer Utilizing Cross-Dimension Dependency for Multivariate Time Series Forecasting},
  author={Zhang, Yunhao and Yan, Junchi},
  booktitle={International Conference on Learning Representations},
  year={2023}
}

@inproceedings{liu2024itransformer,
  title={{iTransformer}: Inverted Transformers Are Effective for Time Series Forecasting},
  author={Liu, Yong and Hu, Tengge and Zhang, Haoran and Wu, Haixu and Wang, Shiyu and Ma, Lintao and Long, Mingsheng},
  booktitle={International Conference on Learning Representations},
  year={2024}
}

\end{document}